\documentclass[format=acmsmall, review=false]{acmart}
\usepackage{acm-ec24-arxiv}
\makeatletter
\def\@ACM@checkaffil{
    \if@ACM@instpresent\else
    \ClassWarningNoLine{\@classname}{No institution present for an affiliation}%
    \fi
    \if@ACM@citypresent\else
    \ClassWarningNoLine{\@classname}{No city present for an affiliation}%
    \fi
    \if@ACM@countrypresent\else
        \ClassWarningNoLine{\@classname}{No country present for an affiliation}%
    \fi
}
\makeatother
\usepackage{booktabs} 
\usepackage[ruled]{algorithm2e} 
\usepackage{natbib} 

\SetAlFnt{\small}
\SetAlCapFnt{\small}
\SetAlCapNameFnt{\small}
\SetAlCapHSkip{0pt}
\IncMargin{-\parindent}

\setcitestyle{authoryear}

\usepackage{mathtools}

\usepackage{nicefrac}
\usepackage{cleveref}
\usepackage{enumitem}
\usepackage{wrapfig}

\newtheorem{theorem}{Theorem}
\newtheorem{corollary}{Corollary}
\newtheorem{lemma}{Lemma}
\newtheorem{claim}{Claim}

\newtheorem{observation}{Observation}

\newtheorem{proposition}{Proposition}
\newtheorem{assumption}{Assumption}

\newtheorem{axiom}{Axiom}

\newcommand{\pivec}{\boldsymbol{\pi}}

\newcommand{\inst}{\mathcal{I}}
\newcommand{\gold}{\textsc{Goldilocks} }
\newcommand{\linear}{\textsc{Linear} }
\newcommand{\maximin}{\textsc{Maximin}}
\newcommand{\leximin}{\textsc{Leximin} }
\newcommand{\minimax}{\textsc{Minimax} }
\newcommand{\nash}{\textsc{Nash} }
\newcommand{\dbelow}{\delta_{below}}
\newcommand{\dabove}{\delta_{above}}
\newcommand{\emdash}{\,---\,}
\newcommand{\mint}{\textsf{manip}_{int}}

\newcommand{\mcomp}
{\textsf{manip}_{comp}}
\newcommand{\manipinst}{\inst^{\to_C \boldsymbol{\tilde{w}}}}

\newcommand{\vecprobs}{\textsf{p}}
\newcommand{\panelcomp}{\textsf{K}}
\newcommand{\poolcomp}{\textsf{N}}
\newcommand{\algo}{\textsc{A}}

\newcommand{\pdist}{\mathbf{d}}

\author{Carmel Baharav}
\email{cbaharav@student.ethz.ch}
\orcid{0009-0004-3634-6721}
\affiliation{%
  \institution{ETH Zurich}
  \city{Zurich}
  \country{Switzerland, cbaharav@student.ethz.ch}
}
\author{Bailey Flanigan}
\email{bflaniga@andrew.cmu.edu}
\orcid{0000-0002-7805-1989}
\affiliation{%
  \institution{Carnegie Mellon University}
  \city{Pittsburgh}
  \state{PA}
  \country{USA, bflaniga@andrew.cmu.edu}
}

\title{Fair, Manipulation-Robust, and Transparent Sortition}

\begin{abstract}
\textit{Sortition}, the random selection of political representatives, is increasingly being used around the world to choose participants of deliberative processes like \textit{Citizens' Assemblies}. Motivated by the practical importance of sortition, there has been a recent flurry of computer science research on sortition algorithms, whose task it is to randomly select a panel that satisfies several quotas enforcing representation of key population subgroups. This existing work has contributed an algorithmic approach for sampling a quota-satisfying set of willing participants while ensuring their chances of selection are \textit{maximally equal}, as measured by any convex equality objective. The question, then, is \textit{which equality objective is the right one?} Past work has mainly studied the objectives \textit{Minimax} and \textit{Leximin}, which respectively minimize the maximum and maximize the minimum chance of selection given to any willing participant. Recent work showed that both of these objectives have key weaknesses: \textit{Minimax} is highly \textit{robust to manipulation}, but it is arbitrarily \textit{unfair}; and oppositely, \textit{Leximin} is highly fair but arbitrarily manipulable. 

In light of this gap, we propose a new equality objective, \textit{Goldilocks}, that aims to achieve these ideals simultaneously by ensuring that no potential participant receives \textit{too little or too much} chance of selection. We give tight theoretical bounds on the extent to which \textit{Goldilocks} achieves this, finding that in a very important sense, \textit{Goldilocks} recovers among the best available solutions in a given instance. We then extend these theoretical bounds to the case where the output of \textit{Goldilocks} is transformed to achieve a third goal, \textit{Transparency}. Our empirical analysis of \textit{Goldilocks} in real data is even more promising: we find that this objective achieves nearly instance-optimal minimum and maximum selection probabilities \textit{simultaneously} in most real instances\emdash an outcome not even guaranteed to be possible for any algorithm. 
In many respects, \textit{Goldilocks} closes the question of whether we can simultaneously achieve three key ideals of sortition\emdash \textit{Fairness}, \textit{Manipulation Robustness}, and \textit{Transparency}\emdash and contributes a practicable algorithm for doing so. 
\end{abstract}

\begin{document}

\maketitle

 \section{Introduction}
In a \textit{citizens' assembly}, a panel of \textit{randomly-selected} everyday people is convened to discuss and collectively weigh in on a policy issue. 
Each year, more and more cities, regions, countries, and even supranational bodies are turning to citizens' assemblies\footnote{\textit{Citizens' assemblies} belong to a broader category of closely-related methods called \textit{deliberative minipublics}, which consist also of citizens' juries, citizens' panels, deliberative polls, and other processes of similar form. We will discuss citizens' assemblies as the primary application domain of this paper.} to involve the public in policymaking; prominent recent examples include multiple national citizens' assemblies in France \cite{giraudet2022co,france2}, Scotland's national climate assembly \cite{scotland}, and a permanent assembly instated in the Ostbelgien government \cite{ostbelgian}.

The subject of this paper is the process used to randomly select the panel members, called \textit{sortition}. 
Broadly defined, sortition just means ``random selection'', and it is often thought of as a simple uniform lottery over the population. In practice, however, the task of sortition is more complicated: practitioners require the panel to satisfy custom \textit{quotas}, which enforce near-proportional representation of key population sub-groups. These groups are usually defined by individual features  (e.g., \textit{women} or \textit{right-leaning voters}), but can be defined by intersections of features as well. While representation of groups would in theory be achieved by a uniform lottery, there is \textit{selection bias}: different subgroups tend to agree to participate at very different rates, meaning that a simple lottery would produce a panel that is far from representative. To ensure representation despite selection bias, in practice panels are selected via the following two-stage process:

\begin{enumerate}
    \item[(1)] First, a uniform sample of the population is invited to participate. Those who respond affirmatively form the \textit{pool of volunteers}. Due to selection bias, this pool is typically very skewed compared to the population.
    \item [(2)] All pool members are asked to report their values of the features on which quotas will be imposed. Then, a \textit{selection algorithm} is used to find a panel within the pool, which must satisfy the practitioner-defined quotas and be of predetermined size $k$.
\end{enumerate}
Our focus is the design of the selection algorithm in Stage (2), whose task it is to sample a representative panel from the skewed pool. The skew of the pool relative to the panel prevents any selection algorithm from randomizing over the pool members perfectly uniformly, as would a simple lottery.\footnote{If a group is disproportionately overrepresented in the pool compared to their quota-allotted fraction of the panel, satisfying the quotas requires giving at least one member of this overrepresented group below-average chance of selection.} However, recent work by \citet{flanigan2021fair} has made it possible to randomize \textit{as equally as possible} over pool members: they introduce an algorithmic framework that can make volunteers' probabilities maximally equal subject to the quotas, as measured by any convex \textit{equality objective} $\mathcal{E}$ (i.e., any mapping from a vector of pool members' selection probabilities to a real number measuring how equal they are). 
The ability to make selection probabilities maximally equal is desirable because it offers hope of retaining\emdash at least to a maximum degree possible\emdash the normative ideals granted by a simple lottery, such as \textbf{\textit{Fairness}, \textit{Manipulation Robustness},} and \textbf{\textit{Transparency}} (to be defined shortly). The question is then: \textit{what equality objective $\mathcal{E}$ should we optimize, in order to maximally achieve these ideals}? Subsequent work, which we overview now, has revealed how the choice of $\mathcal{E}$ can have important consequences for these ideals.

The originally proposed equality objective was \textit{Leximin} (a refinement of \textit{Maximin}), which measures equality according to the minimum selection probability and thereby aims to ensure that no selection probability is too low. This choice of objective was motivated by the ideal of \textbf{\textit{Fairness}}: that every willing participant is entitled to their fair share of the chance to participate. 

\textit{Leximin} made substantial fairness gains over existing state-of-the-art algorithms, but subsequent work identified a major weakness of this objective: in theory and in practice, it allows pool members to ensure they are deterministically selected for the panel by misreporting their features at the beginning of Stage (2) \citep{flanigan2024manipulation}. In light of this finding, Flanigan \emph{et al.} defined a new ideal: a selection algorithm's \textit{\textbf{Manipulation Robustness}} is the extent to which it limits how much any agent can increase their chance of selection by misreporting. 

\citet{flanigan2024manipulation} then diagnose the reason \textit{Leximin} is so vulnerable to manipulation: it raises low probabilities without regard for high probabilities. Then, if a prospective participant can guess which identities are essential to raising the lowest probabilities, they can misreport these identities and the algorithm will push their selection probability all the way up to 1. They further show that the well-known equality objective \textit{Nash Welfare} (the geometric mean of selection probabilities) suffers the same problem for the same reason. Motivated by these negative results, they propose a new equality objective, \textit{Minimax}, which minimizes the maximum selection probability.\footnote{Technically, \citet{flanigan2024manipulation} study the $\ell_{\infty}$-norm rather than \textit{Minimax}; in our context, these objectives are essentially the same, so for simplicity we consider \textit{Minimax}.} They show that \textit{Minimax} minimizes agents' incentives for manipulating, thereby achieving optimal \textit{Manipulation Robustness}. Unfortunately, they find that \textit{Minimax} has essentially the opposite problem as \textit{Leximin}: because \textit{Minimax} does not control \textit{low} probabilities, it often gives many people zero selection probability, thereby performing unacceptably poorly with respect to \textit{Fairness}. 

From this related work, we distill three observations: low probabilities are a problem for \textit{Fairness}; high probabilities are a problem for \textit{Manipulation Robustness}; and no known objective controls both simultaneously. These observations motivate the first two questions we will tackle in this paper:\\[-0.75em]

\noindent \textit{\textbf{Question 1:}} Can we design an equality objective $\mathcal{E}$ that ensures optimal simultaneous lower and upper bounds on selection probabilities? and consequently,\\[-0.75em]

\noindent \textit{\textbf{Question 2:}} To what extent do these bounds on the minimum and maximum selection probability permit simultaneous guarantees on $\mathcal{E}$'s \textit{Fairness} and \textit{Manipulation Robustness}?\\[-0.75em]

Much of the paper will be dedicated to simultaneously achieving \textit{Fairness} and \textit{Manipulation Robustness}, because these ideals trade off with one another and are similarly determined by $\mathcal{E}$. Then, we will investigate whether we can achieve these two ideals alongside a third ideal of sortition algorithms considered in previous work: \textit{Transparency}. 

Colloqually, transparency means that the public should be able to confirm that the selection process is actually \textit{random}, and the organizers are not stacking the panel behind the scenes. \citet{flanigan2021transparent} proposed a more precise version of this definition: without reasoning in-depth about probability, the public should be able to observe all volunteers' chances of selection. \citet{flanigan2021fair} proposed a method for achieving this ideal: to round the output of a maximally equal algorithm into a \textit{uniform lottery} over $m$ (potentially duplicated) panels \cite{flanigan2021fair}, and then select the panel by performing this uniform lottery live. This approach permits transparency as follows: prior to the lottery, organizers release an (anonymized) list of panels containing each pool member, allowing any pool member's selection probability to be tabulated by simply counting how many panels they are on and dividing by the total number of panels (usually 1000 in practice). 

The catch, of course, is that transforming the output of a maximally equal algorithm into a uniform lottery over panels could require significant loss in equality. Fortunately, \citet{flanigan2021transparent} proved that there exist rounding procedures guaranteed to produce uniform lotteries that are near-optimal---at least with respect to \textit{Leximin} and \textit{Nash Welfare}, objectives which both target fairness alone.
Given that this uniform lottery approach is used in practice, for our new equality objective to be viable, we must ensure that rounding approximately preserves its \textit{Fairness} and \textit{Manipulation Robustness}. This motivates our third question:\\[-0.75em]

\noindent \textit{\textbf{Question 3:}} If we achieve \textit{Transparency} by rounding the output of our $\mathcal{E}$-optimal algorithm to a uniform lottery, to what extent does $\mathcal{E}$ still achieve \textit{Fairness} and \textit{Manipulation Robustness}?
\subsection{Approach and Contributions}
\textbf{Unification of existing models, and a new equality objective (\Cref{sec:model}).} First, we undertake the considerable task of unifying existing models of fairness, manipulation robustness, and transparency. 
As in past work \cite{flanigan2024manipulation}, we study three manipulation incentives: increasing one's own selection probability, decreasing someone else's, or misappropriating panel seats from other groups. We permit manipulating coalitions of up to linear size (in the pool size $n$), and we permit agents to misreport any features costlessly with full knowledge of the selection algorithm and the pool's composition.\footnote{While these assumptions may seem extremely pessimistic, we adopt them because our algorithm is to be implemented in practice, and we want guarantees that do not depend on any specific assumed behavioral model.} Given the insufficiency of known equality objectives, we propose a new one, \textit{Goldilocks}$_\gamma$, defined below. Here, (slightly informally for now) $\pivec$ is an assignment of selection probabilities to pool members, and $\max(\pivec), \min(\pivec),$ and $\text{avg}(\pivec)$ describe the maximum, minimum, and average selection probability, respectively.
\begin{center}
    $\textit{Goldilocks}_\gamma(\pivec):= \ \frac{\max(\pivec)}{\text{avg}(\pivec)} + \gamma \cdot \frac{\text{avg}(\pivec)}{\min(\pivec)}.$
\end{center}
In the style of multi-objective optimization, the first and second term respectively aim to control high and low selection probabilities, and $\gamma \in [0, \infty)$ is a scalar determining the extent to which the objective prioritizes controlling the minimum versus the maximum probability.
For intuition about this objective's behavior, consider $\gamma = 1$: minimizing $Goldilocks_1$ (essentially) minimizes the maximum multiplicative deviation of any probability from the average. Although we will occasionally consider other values of $\gamma$, our analysis will mainly focus on \textit{Goldilocks}$_1$.\\[-0.9em]


\noindent \textbf{Establishing the \textit{existence} of ``good'' solutions (\Cref{sec:existence}).} Guaranteeing \textit{Fairness} and \textit{Manipulation Robustness} depends on the ability to simultaneously ensure that no selection probability is too low (to ensure fairness) nor too high (to ensure manipulation robustness). In turn, the extent to which \textit{any} algorithm can do so depends on the quality of feasible solutions, \textit{which we observe can be diminished by manipulation.} In particular, we show that a coalition of agents can misreport their features in a way that creates or worsens fundamental gaps between agents' probabilities, thereby eliminating candidate-optimal solutions.\footnote{This issue does not arise when one aims only to control \textit{Manipulation Robustness}. To see this, see \citet{flanigan2024manipulation}, which shows that \textit{Minimax} optimally achieves manipulation robustness. This is possible, however, because when manipulating coalitions induce gaps between two groups' probabilities, \textit{Minimax} can simply give these groups zero selection probability in other words, when fairness is not a concern, manipulating coalitions do not affect the set of potentially \textit{optimal} solutions. In contrast, any objective controlling both high and low probabilities \textit{must respond} to such fundamental gaps in selection probabilities, so manipulating coalitions can affect the set of potentially optimal solutions.} This brings us to the fundamental challenge associated with ensuring both \textit{Fairness} and \textit{Manipulation Robustness}: we must first establish the quality of existing solutions, which may depend on both the original instance and the number of manipulators.
We address this challenge in \Cref{sec:existence}, giving matching upper and lower bounds on the extent to which manipulation can eliminate solutions with high minimum probabilities and/or low maximum probabilities. Among these results, we demonstrate a fundamental trade-off between controlling high and low probabilities.\\[-0.9em]

\noindent \textbf{\textit{Fairness}, \textit{Manipulation Robustness} and \textit{Transparency} of \textit{Goldilocks}$_1$ (Sections \ref{sec:gold-analysis} and  \ref{sec:transparency}).} 
To bound \textit{Goldilocks}$_1$'s fairness and manipulation robustness, we first show that \textit{Goldilocks}$_1$ guarantees lower and upper bounds on selection probabilities that scale naturally\emdash and in many relevant cases, tightly\emdash with the quality of available solutions. These bounds translate almost directly to bounds on fairness and manipulation robustness. For reasons described in \Cref{sec:existence}, our bounds will depend on two quantities: $c$, the number of agents who are willing to misreport, and $n_{min}$, the minimum number of agents with given vector of features in the pool. Algorithms can be best distinguished when $n_{min}$ is large (otherwise, all algorithms are subject to the same impossibility), so we recap our results here assuming that $n_{min}$ is large (our results handle general $n_{min}$, though). We are particularly interested in how our guarantees depend on $n$, as increasing $n$ is a practically-implementable way to potentially diminish manipulation incentives \cite{flanigan2024manipulation}.

Regarding \textit{Goldilocks}$_1$'s manipulation robustness, we find that no manipulating coalition of size $c$ can increase a single member's selection probability by more than order $\sqrt{c}/n$\emdash a quantity that diminishes quickly in $n$, even if $c$ grows linearly with $n$. We give similar bounds for the other two manipulation incentives. Regarding fairness, we find that \textit{Goldilocks}$_1$ guarantees \textit{Maximin} fairness (i.e., minimum probability) of at least order $1/(\sqrt{c} n)$. We then give a lower bound on the extent to which any algorithm can simultaneously guarantee fairness and manipulation robustness, revealing that our bounds are tight in a natural subset of regimes and nearly tight in the rest. Finally, in \Cref{sec:transparency} we bound the extent to which these results (approximately) hold after the output of \textit{Goldilocks}$_1$ is rounded to a uniform lottery, thereby establishing the extent to which fairness and manipulation robustness can be achieved alongside transparency. \\[-0.9em]

\noindent \textbf{Empirical study of \textit{Goldilocks}$_1$ (\Cref{sec:empirics}).} Finally, we analyze \textit{Goldilocks}$_1$ in real citizens' assembly datasets, and we find that it performs even better than our bounds guarantee. Our first key finding is that \textit{Goldilocks}$_1$ achieves near \textit{Leximin}-optimal minimum probabilities \textit{and} \textit{Minimax}-optimal maximum probabilities\emdash an outcome whose possibility by \textit{any} algorithm was not guaranteed. On our ideals, we compare \textit{Goldilocks}$_1$'s performance to the other previously-studied equality objectives, \emdash \textit{Leximin}, \textit{Nash Welfare}, and \textit{Minimax}. We additionally compare these algorithms to \textit{Legacy}, a heuristic standing in for the wide variety of heuristic selection algorithms still used in practice today. Our main finding is that \textit{Goldilocks}$_1$ performs nearly as well as \textit{Leximin} on fairness and \textit{Minimax} on manipulation robustness, and it far outperforms all other algorithms in its ability to achieve both these goals at once. In our evaluation, we find that our theoretical results translate to the more realistic case where manipulators are not worst-case, but rather use a natural heuristic to decide how to manipulate. Finally, we find that \textit{Goldilocks}$_1$ can be made transparent with little-to-no cost to the maximum and minimum selection probabilities.

\subsection{Related Work}
In addition to the existing work on fairness \cite{flanigan2021fair}, manipulation robustness \cite{flanigan2024manipulation}, and transparency \cite{flanigan2021transparent} on which we directly build, there is a growing body of work pursuing selection algorithms achieving similar ideals. There is especially a wealth of literature considering the interplay of two ideals: \textit{fairness} (as we define it), and proportional representation of the underlying population, which we enforce with quotas.
However, much of this work is done in the distinct model of sortition where it is possible to sample the population directly, and all chosen will participate (i.e., there is no selection bias). For example, \citet{ebadianboosting} study how to achieve exact fairness and deterministic proportional representation simultaneously; closely related is work by \citet{benade2019no}, which focuses on uniform-like stratified sampling while preserving subgroup-level representation. \citet{EKMP+22} ask richer questions about the nature of representation that can be achieved when individual people can serve as representatives for others to varying extents. Outside the uniform selection model, \citet{gkasiorowska2023sortition} does a qualitative survey across many selection process case studies, evaluating them on the basis of \textit{randomness} (closely related to our ideal of \textit{fairness}) as well as representation. Beyond related work on sortition, the existing theoretical results we build on in this paper use tools from across several fields, including randomized rounding \cite{gandhi2006dependent}, discrepancy theory \cite{beck1981integer}, and optimization of large linear programs \cite{bradley1977applied}.

    \section{Model} \label{sec:model}
    We use $\Delta(S)$ to represent the set of all distributions over the elements of set $S$. Let $[n]$ be the \textit{pool}, where $i \in [n]$ is an individual agent. The pool is formed by inviting a uniform sample of the population to participate; the agents in $[n]$ are those who responded affirmatively to this invitation.\\[-0.9em] 
    
    \noindent \textbf{Features, feature-values, and feature-vectors.} Let $F$ be a predefined set of \textit{features}, where each feature $f \in F$ can take on some predefined set of values $V_f$. For example, $F$ could be $\{$\textit{age}, \textit{gender}$\}$, and $V_{age}$ might be $\{18\,\text{-}\, 40,\, 41\,\text{-}\,60,\, 61\,+\}$. We call each $v \in V_f$ a \textit{feature-value} and each $f,v$ a \textit{feature-value pair}. $FV:=\{(f,v) | f \in F, v \in V_f\}$ is the set of all feature-value pairs. 
    
    We assume that for each feature $f$, its possible values $V_f$ are exhaustive and mutually exclusive, so every agent has exactly one value $v \in V_f$ for every feature $f$. We denote $i$'s value for feature $f$ as $f(i)$, thereby using each $f$ as a function $f: [n] \to V_f$. We let $i$'s \textit{feature vector} $w_i := (f(i) | f \in F)$ summarize their feature-values, and $\boldsymbol{w}:= (w_i | i \in [n])$ contains all agents' feature vectors. We let $\mathcal{W} := \prod_{f \in F} V_f$ be the set of all possible feature vectors; then, $w_i \in \mathcal{W}$ and $\boldsymbol{w} \in \mathcal{W}^n$.
    \\[-0.9em] 
    

    \noindent \textbf{The panel selection task.} Our task is to choose a \textit{panel} $K \subseteq [n]$ of some pre-chosen size $k \in \mathbb{N}$. The main constraint on $K$ is that it must satisfy \textit{upper and lower quotas} on all feature-values. Formally, for each $f,v \in FV$, we define lower and upper quotas $\ell_{f,v} \in \mathbb{N}^+$ and $u_{f,v} \in \mathbb{N}^+$ 
    We summarize these quotas in $\boldsymbol{\ell} = \{\ell_{f,v} | f,v \in FV\}$ and $\boldsymbol{u} = \{u_{f,v} | f,v \in FV\}$. The set of all \textit{valid panels}\emdash i.e., those satisfying all requirements\emdash is then
     \[\mathcal{K}:=\left\{K : K \subseteq [n] \ \  \land \ \  |K| = k \ \  \land \ \  \ell_{f,v} \leq |\{i \in K : f(i) = v\}| \leq u_{f,v} \ \forall f,v \in FV\right\}.\]
      %
    An \textit{instance} of the panel selection task is defined as $\mathcal{I}:=([n],\boldsymbol{w},k,\boldsymbol{\ell},\boldsymbol{u}).$ Given an instance, the panel selection task is to output a valid panel $K \in \mathcal{K}$.\\[-0.9em] 

     In a given instance $\inst$, we will sometimes refer to $\boldsymbol{w}$ as the pool's \textit{vector composition}, as it encodes how many times each feature vector $w \in \mathcal{W}$ appears in the pool, a number we denote as $n_w(\inst):= |\{i \in [n] | w_i = w\}|$. When $\inst$ is clear, we will simply write $n_w$. Because the pool will generally not contain all possible feature vectors, we let $\mathcal{W}_{\inst} \subseteq \mathcal{W}$ denote the set of unique feature vectors present in the pool in $\inst$ (so $n_w(\inst) > 0$ iff $w \in \mathcal{W}_{\inst}$). We assume that the instance is not degenerate in that $|\mathcal{W}_{\inst}| > 1$ (the pool contains more than one type of person). We let $n_{min}(\inst):=\min_{w \in \mathcal{W}_{\inst}}n_w(\inst)$ be the size of the smallest vector group present in the pool. Note that $|\mathcal{W}_{\inst}| > 1$ implies that $n_{min}(\inst) \leq n/2$.\\[-0.9em]
    
    \noindent \textbf{Panel distributions and selection probabilities.} In instance $\mathcal{I}$ with valid panels $\mathcal{K}$, $\Delta(\mathcal{K})$ is the set of all possible randomizations over valid panels. We call each $\mathbf{d} \in \Delta(\mathcal{K})$ a \textit{panel distribution}, where $d_K$ then denotes the probability of drawing $K$ from $\mathbf{d}$. Any given $\mathbf{d}$ must imply some \textit{selection probability} for each agent $i \in N$, defined as
    \begin{center}
       $ \pi_i(\mathbf{d}) := \sum_{K \in \mathcal{K} : i \in K} d_K \qquad \text{for all }i \in [n]. $
    \end{center}
     In words, $\pi_i(\mathbf{d})$ is the probability that $i$ is included on the panel when the panel is drawn from $\mathbf{d}$. We refer to $\boldsymbol{\pi}(\mathbf{d}):=(\pi_i(\mathbf{d})| i \in [n])$ as an \textit{assignment} of selection probabilities to all agents in the pool. A generic selection probability assignment will be $\pivec$. We use the shorthand $\max(\pivec):=\max_{i \in [n]} \pi_i$ and $\min(\pivec):=\min_{i \in [n]} \pi_i$ to respectively represent the maximum and minimum selection probability assigned by $\pivec$ to any agent.
     
     In any instance $\inst$, the space of all \textit{realizable} selection probability assignments is $\Pi(\inst):= \{\boldsymbol{\pi}(\mathbf{d}) : \mathbf{d} \in \Delta(\mathcal{K})\}.$
     In words, $\Pi(\inst)$ is the set of all selection probability assignments that are implied by some randomization over exclusively valid panels. Observe that for any $\pivec \in \Pi(\inst)$, $\sum_{i \in [n]}\pi_i=k$. Therefore, in any given instance, the selection probability assignment that gives all agents equal selection probability must be $\pivec = k/n\mathbf{1}^n$, the $n$-length vector in which every entry is $k/n$ (note that in most instances, this selection probability assignment will not be in $\Pi(\inst)$).

     Finally, we say that $\pivec$ is \textit{anonymous} iff it gives all agents with the same feature vector the same selection probability\emdash that is, for all $w \in \mathcal{W}$, there exists a constant $z_w$ such that $\pi_i = z_w$ for all $i$ : $w_i = w$. As we will typically work with anonymous selection probability assignments, we define \textit{vector-indexed selection probabilities} $\vecprobs_w(\pivec) = z_w$. Let $\vecprobs(\pivec) = (\vecprobs_w(\pivec)| w \in \mathcal{W})$. When $\pivec$ is clear from context or when we work with arbitrary vector-indexed probabilities, we simply write $\vecprobs$.\\[-0.9em]

    \noindent \textbf{Equality objectives.} Let an \textit{equality objective} $\mathcal{E} : [0,1]^n \to \mathbb{R}$ be a function that intakes a selection probability assignment and outputs a scalar measure of how \textit{equal} the selection probabilities within it are.
    All equality objectives we will consider are convex. They will also all have the property that $\pivec$ is ``more equal'' than $\pivec'$ according to $\mathcal{E}$ if $\mathcal{E}(\pivec) \leq \mathcal{E}(\pivec')$. Then, a selection probability assignment $\pivec$ is \textit{maximally equal} in $\inst$ iff
    $\pivec \in \text{arg}\inf_{\pivec \in \Pi(\mathcal{I})} \mathcal{E}(\pivec)$. The set of all maximally equal selection probability assignments in $\inst$, as measured by 
    $\mathcal{E}$, is 
    \begin{center}
        $\Pi^{\mathcal{E}}(\inst):= \text{arg}\inf_{\pivec \in \Pi(\mathcal{I})} \mathcal{E}(\pivec) \subseteq \Pi(\inst)$.
    \end{center}
    We will study the three equality objectives considered in past work on sortition \cite{flanigan2021fair,flanigan2021transparent,flanigan2024manipulation}, defined below.
    Here, \textit{Nash} is the \textit{Nash Welfare}.
    \begin{center}
        $Maximin(\pivec) := -\min(\pivec), \qquad Minimax(\pivec):= \max(\pivec), \qquad Nash(\pivec):= -\left(\prod_{i \in [n]}\pi_i\right)^{1/n}$.
    \end{center} 
    We also study \textit{Leximin}, which is not strictly an equality objective, but is a previously-studied refinement of \textit{Maximin}. \textit{Leximin} first maximizes the minimum selection probability (i.e., finds the \textit{Maximin}-optimal solution), then maximizes the \textit{second-lowest} selection probability, then the third-lowest, and so on. 
    Finally, we newly introduce the objective \textit{Goldilocks}$_\gamma$, which penalizes multiplicative deviations both above and below $k/n$ with $\gamma \in \mathbb{R}_{\geq 0}$ controlling the relative priority placed on either type of deviation. 
    \begin{center}
        $\textit{Goldilocks}_\gamma(\pivec):= \frac{\max(\pivec)}{k/n} + \gamma \,\frac{k/n}{\min(\pivec)}.$
    \end{center}

In addition to being convex (\Cref{prop:convexity}), all objectives we consider\footnote{Because \textit{Leximin} is not an equality objective, it cannot formally satisfy these properties. However, as will be clear throughout the paper, \textit{Leximin} \textit{effectively} satisfies these properties to the extent we need it to.} satisfy two other natural axioms\emdash \textit{conditional equitability} (\Cref{prop:ce}) and \textit{anonymity} (\Cref{prop:anon}). These axioms are both weak requirements reflecting that equality objectives truly measure the level of \textit{equality} of selection probabilities. In words, \textit{conditional equitability} requires $\mathcal{E}$ to consider $\pivec = k/n\mathbf{1}^n$ the most equal possible probability assignment, and \textit{anonymity} requires that $\mathcal{E}$ does not penalize giving identical agents identical selection probabilities.
    \begin{axiom}
        $\mathcal{E}$ is \textbf{conditionally equitable} iff for all $\inst$, $k/n\mathbf{1}^n \in \Pi(\inst) \implies k/n\mathbf{1}^n \in \Pi^\mathcal{E}(\inst)$.
    \end{axiom}
    \vspace{-1em}
    \begin{axiom}
        $\mathcal{E}$ is \textbf{anonymous} iff for all $\mathcal{I}$, there exists an anonymous $\pivec \in \Pi^\mathcal{E}(\inst)$.
    \end{axiom}
    Because all objectives $\mathcal{E}$ we consider satisfy anonymity, we will without loss of generality redefine $\Pi(\inst)$ and $\Pi^{\mathcal{E}}(\inst)$ to contain \textit{only anonymous selection probability assignments}.\\[-0.75em]

    \noindent{\textbf{Selection algorithms.}} A \textit{selection algorithm} $\textsc{A} : \mathcal{I} \to \mathcal{K}$ is any (potentially randomized) mapping from an instance to a valid panel $K \in \mathcal{K}$. Note that in a given instance, any selection algorithm must induce a panel distribution; we denote the panel distribution implied by $\textsc{A}$ in $\mathcal{I}$ as $\mathbf{d}^{\textsc{A}}(\inst) \in \Delta(\mathcal{K})$. Its implied selection probability assignment is then $\boldsymbol{\pi}(\mathbf{d}^{\textsc{A}}(\inst))$; for simplicity of notation, when the panel distribution is not directly relevant, we will shorten this to $\pivec^{\algo}(\inst)$.

     A selection algorithm $\textsc{A}$ is \textit{maximally equal} with respect to $\mathcal{E}$ iff $\pivec^{\algo}(\inst) \in \Pi^{\mathcal{E}}(\inst)$ for all $\mathcal{I}.$
     
    Fortunately, the optimization framework proposed by \citet{flanigan2021fair} gives an algorithmic implementation for any maximally equal selection algorithm whose corresponding equality objectives $\mathcal{E}$ is convex, which we will use to optimize the equality objectives defined above. 
    At a high level, their algorithmic approach works in two steps: first, it explicitly computes a panel distribution implying \textit{maximally equal} selection probabilities per $\mathcal{E}$; then, it draws the final panel from this panel distribution, thereby realizing those maximally equal selection probabilities. As shorthand, we will refer to the algorithm from this framework optimizing $\mathcal{E}$ as \textsc{E} (e.g., the algorithm optimizing \textit{Maximin} is called \maximin).

\subsection{Ideals: \textit{Manipulation Robustness}, \textit{Fairness}, and \textit{Transparency}}
    
    \textit{Manipulation Robustness}.
    To capture the fact that agents may misreport their feature-values to the algorithm, we denote $i$'s \textit{reported} feature vector as $\tilde{w}_{i} \in \mathcal{W}$, to distinguish it from $w_i$. We refer to the set of agents who may misreport their vectors as a \textit{coalition}, denoted $C \subseteq [n]$. We denote the vector of agents' \textit{reported} feature vectors as $\boldsymbol{\tilde{w}} := (\tilde{w}_i | i \in [n])$. We define the instance $\inst^{\to_C \tilde{\boldsymbol{w}}}:= ([n],\tilde{\boldsymbol{w}},k,\boldsymbol{\ell},\boldsymbol{u})$ as the instance created when a coalition $C$ misreports such that the vector composition of the pool changes from $\boldsymbol{w}$ to $\tilde{\boldsymbol{w}}$. Note that not all $\tilde{\boldsymbol{w}}$ can result from a given $\inst,C$ pair; we let $\boldsymbol{\mathcal{W}}_{\inst,C} \subseteq \mathcal{W}^{n}$ be the set of all possible $\boldsymbol{\tilde{w}}$ that result from any misreports of $C$ starting from $\inst$. Formally, fixing $\inst$ and $C \subseteq [n]$, $\tilde{\boldsymbol{w}} \in \boldsymbol{\mathcal{W}}_{\inst,C}$ iff $\tilde{w}_i = w_i$ for all $i \notin C$ (for all $i \in C$, it can be that $\tilde{w}_i \neq w_i \neq $ or $\tilde{w}_i = w_i$).

    As in \citet{flanigan2024manipulation}, we assume that agents or coalitions can costlessly misreport any feature vector in $\mathcal{W}$, and they do so with full information about the selection algorithm and pool. We consider three incentives for doing so: $\textsf{manip}_\text{int}$ captures how much a coalition can \textit{increase the selection probability of someone internal to the coalition}; $\textsf{manip}_\text{ext}$ measures how much a coalition can \textit{decrease the selection probability of someone external to the coalition}; and $\textsf{manip}_{\text{comp}}$ measures how many seats a coalition can, in expectation, misappropriate from another group. 
\begin{align*}
    \textsf{manip}_\text{int}  (\inst,\algo,c) &:= \max_{C \subseteq[n], |C| = c} \ \max_{\tilde{\boldsymbol{w}} \in \boldsymbol{\mathcal{W}}_{\inst,C}} \ \ \max_{i \in C} \ \ \pi^{\algo}_{i}(\inst^{\to_C \tilde{\boldsymbol{w}}}) - \pi^{\algo}_{i}(\inst),\\
    \textsf{manip}_{\text{ext}}(\inst,\algo,c) &:=\max_{C \subseteq[n], |C| = c} \ \max_{\tilde{\boldsymbol{w}} \in \boldsymbol{\mathcal{W}}_{\inst,C}} \  \ \max_{i \notin C} \ \ \pi^{\algo}_{i}(\inst) - \pi^{\algo}_{i}(\inst^{\to_C \tilde{\boldsymbol{w}}}),\\
    \textsf{manip}_{\text{comp}}(\inst,\algo,c)&:=  \max_{C \subseteq[n], |C| = c} \ \max_{\tilde{\boldsymbol{w}} \in \boldsymbol{\mathcal{W}}_{\inst,C}} \ \ \max_{(f,v) \in FV} \sum_{i : f(i) = v} \left(\pi^{\algo}_{i}(\inst^{\to_C \tilde{\boldsymbol{w}}}) - \pi^{\algo}_{i}(\inst) \right).
\end{align*}
These definitions can be interpreted as Nash equilibrium-style measures, capturing how much a coalition can gain if everyone else is truthful. From a formal game theoretic perspective, the argument of each maximum above can be thought of as a utility function, which an agent or coalition may aim to maximize. These definitions are worst-case (over coalitions and strategies) to avoid assuming a behavioral model; importantly, they encompass the entire range of cases where agents do not collude, but multiple agents pursue these motives individually via any strategy.\\[-0.9em] 

\noindent \textit{Fairness}. In accordance with past work on fairness in sortition \cite{flanigan2021fair,flanigan2021transparent}, we evaluate the fairness of a selection algorithm $\textsc{A}$ such that a \textit{fairer} algorithm makes the minimum selection probability higher. The \textit{fairness} of algorithm $\textsc{A}$ in instance $\inst$ is defined formally as below; note that any algorithm optimizing \textit{Maximin} is by definition optimally fair in any given instance.
\[\textsf{fairness}(\inst,\textsc{A}):= \min_{i \in [n]}\pi_i^{\textsc{A}}(\inst).\]
Because we want to guarantee fairness and manipulation robustness simultaneously, we are actually interested in studying fairness in the presence of manipulation. As such, the fairness of \textsc{A} in $\inst$ in the presence of the worst-case manipulating coalition of size $c$ is defined as 
\begin{align*}
    &\textsf{manip-fairness}(\inst,\textsc{A},c):= \min_{C \subseteq[n], |C| = c} \ \min_{\tilde{\boldsymbol{w}} \in \boldsymbol{\mathcal{W}}_{\inst,C}} \ \ \textsf{fairness}( \pi_i^{\textsc{A}}(\inst^{\to_C \tilde{\boldsymbol{w}}}),\textsc{A}).
\end{align*}

\noindent \textit{Transparency.} We now formally define the components of \citet{flanigan2021transparent}'s algorithmic approach to transparency. For a given set of valid panels $\mathcal{K}$, let $m \in \mathbb{Z}^+$ and define the set of all $m$\textit{-uniform lotteries} $\overline{\Delta}_m(\mathcal{K}) :=(\mathbb{Z}^+/m)^{|\mathcal{K}|} \cap \Delta(\mathcal{K})$ as the set of all panel distributions in which all probabilities are multiples of $1/m$. 
$\overline{\pdist} \in \overline{\Delta}_m$ is called an $m$-uniform lottery due to the following key observation: $\overline{\pdist}$ contains exactly $m$ discrete blocs of $1/m$ probability mass, so we can sample a panel from $\overline{\pdist}$ via a \textit{uniform lottery over $m$ panels (with duplicates)} by numbering these probability blocs $1\dots m$, and then uniformly drawing a number from $[m]$. For example, if $m = 1000$, we can execute this uniform lottery physically, by drawing balls from bins corresponding to drawing 3 digits between 0 and 9, as in Figure 3 of \citet{flanigan2021fair}. Finally, we define $\overline\Pi_m(\inst):= \{\pivec(\mathbf{\bar{d}}) | \mathbf{\bar{d}} \in \overline{\Delta}_m(\mathcal{K})\}$ as the set of all selection probability assignments realizable by $m$-uniform lotteries in $\inst$. 

An $m$-uniform lottery is created by a  \textit{rounding algorithm} $\mathcal{R}_m : \Delta(\mathcal{K}) \to \overline{\Delta}_m(\mathcal{K})$, which is any (possibly randomized) mapping from a panel distribution into an $m$-uniform lottery. We apply a rounding algorithm $\mathcal{R}_m$ \textit{in conjunction} with maximally fair algorithm \textsc{E} as follows: first, run \textsc{E} to compute panel distribution $\pdist^{\textsc{E}}(\inst)$; then, use $\mathcal{R}_m$ to round $\pdist^{\textsc{E}}(\inst)$ to an $m$-uniform lottery $\mathcal{R}_m(\pdist^{\textsc{E}}(\inst))$. Note that $\mathcal{R}_m \circ \textsc{E}$ is itself a selection algorithm, mapping $\inst$ to an $m$-uniform lottery $\mathbf{d}^{\mathcal{R}_m \circ \textsc{E}}(\inst)$. We will define specific rounding algorithms as needed.

\subsection{Handling \textit{Structural Exclusion}} \label{sec:exclusion}
We say that an instance $\inst$ is \textit{structurally exclusive} if there exist one or more agents who are not included on \textit{any valid panel}. We denote the set of agents who are structurally excluded in $\inst$ as EX$(\inst):= \{i \in [n] | i \notin K \text{ for all }K \in \mathcal{K}\}$. In any instance where EX$(\inst) \neq \emptyset$, any algorithm must have \textsf{manip-fairness} of 0 by construction. We now introduce restrictions to avoid this trivial impossibility, which we prove are necessary in \Cref{prop:assumptions-necessary} (\Cref{app:assumptions-necessary}).

First, Assumption 1 ensures that the \textit{truthful} instance $\inst$ does not exhibit structural exclusion. This is a very weak restriction, and it is satisfied by all real-world datasets we study in \Cref{sec:empirics}.
\begin{assumption}\label{ass:inclusion}
    Instance $\inst$ is such that EX$(\inst) = \emptyset$.
\end{assumption}
Second, we must handle structural exclusion that is \textit{caused by manipulation} (i.e., EX($\inst^{\to_C \tilde{\boldsymbol{w}}}) \neq \emptyset$, even if EX$(\inst) = \emptyset$). We avoid the structural exclusion of \textit{non}-coalition members (i.e., EX$(\inst^{\to_C \tilde{\boldsymbol{w}}}) \cap ([n] \setminus C) \neq \emptyset$) by considering only coalitions of limited size: in $\inst$, we consider $C \subseteq [n]$ such that
\begin{equation} \label{eq:restriction1}
   |C| \leq \max\{0, n_{min}(\inst)- k\}. \tag{Restriction 2.1} 
\end{equation}
This restriction ensures that no vector group in the original pool is depleted to a size below $k$, meaning that all valid panels in $\mathcal{K}$ remain valid in $\inst^{\to_C \tilde{\boldsymbol{w}}}$ for all $C \subseteq [n], \tilde{\boldsymbol{w}} \in \boldsymbol{\mathcal{W}}_{\inst,C}$. 

Finally, we must handle coalition members who exclude \textit{themselves} (i.e., EX$(\inst^{\to_C \tilde{\boldsymbol{w}}}) \cap C \neq \emptyset$). This is an unnatural corner case (such a misreport would be costly to the agent while barely affecting any other selection probability), so we do not risk eliminating interesting cases with a new restriction. Instead, we simply do not ``count'' such agents in our guarantees, and handle the resulting decrease in pool size in our bounds. Formally, we henceforth implicitly redefine the manipulated instance as $\inst^{\to_C \boldsymbol{\tilde{w}}}$ as $([n'], \tilde{\boldsymbol{w}}',k,\boldsymbol{\ell},\boldsymbol{u})$,
where (re-indexing agents) $[n']:= [n] \setminus (C \cap \textsc{EX}(\inst))$, and $\tilde{\boldsymbol{w}}':=(\tilde{w}_i | i \in [n'])$, so any manipulators who have excluded themselves are removed. In practice, this approach corresponds to a natural and implementable algorithmic behavior: detect and remove structurally excluding agents before running the algorithm.

\section{\textit{Existence} of ``Good'' Solutions} \label{sec:existence}
\subsection{Intuition: Problems 1 and 2 Prevent Good Solutions} \label{sec:intuition-problems}
Before analyzing any algorithms, we examine the extent to which it is \textit{possible}, in any given instance, to simultaneously control maximum and minimum selection probabilities. We begin by identifying two potential barriers to this goal. \\[-0.9em]

\noindent \textbf{Problem 1: Small groups.} Sometimes, high probabilities can be required simply due to the structure of the quotas and the pool. Consider Example \ref{ex:small-group} below: all valid panels contain some minimum number of agents with a certain vector, and there is a small number of such agents in the pool. As a result, any valid panel distribution --- and therefore \textit{any algorithm} --- must give each agent in this group high selection probability. 

\begin{example}[small groups] \label{ex:small-group}
 Let there be one feature $f$ with binary values $V_f = \{0,1\}$. Let $\inst$ such that the quotas require the panel to contain $k-1$ agents with $f = 0$ and 1 agent with $f = 1$ ($\ell_{f,0} = u_{f,0} = 1$). For any $n_{min} \in \{1,\dots,n-k+1\}$, let the pool be such that $n_{0} = n - n_{min}$ and $n_{1}  = n_{min}$. Then, any valid panel distribution must give each agent with vector $1$ at least $1/n_{min}$ selection probability. It follows that for any $\pivec \in \Pi(\inst)$, it must be that $\max(\pivec) \geq 1/n_{min}.$

\end{example}

\noindent \textbf{Problem 2: Fundamental trade-offs between maximum, minimum probabilities.} 
Even in the absence of Problem 1, we can face another: the inclusion of one vector group on a panel can \textit{necessitate} the inclusion of another vector group. As illustrated by Example 2 (based on an example from \citet{flanigan2024manipulation}), if these two ``linked-fate'' vector groups differ in size, then agents in the smaller group must receive higher selection probabilities than those in the larger group. This creates an inescapable trade-off between the maximum and minimum selection probabilities.

\begin{example}[fundamental trade-off]\label{ex:trade-off}
    Let there be two binary features $f_1,f_2$ with $V_{f_1} = V_{f_2}=\{0,1\}$. Let $\inst$ be such that the quotas require the panel be evenly split between 0/1 values of both features ($\ell_{f_1, 0} = u_{f_1, 0} = \ell_{f_2, 0} = u_{f_2, 0} = k/2$), and let the pool be such that $n_{00} = n_{11} = n/4$, $n_{10} = n/2-1$, and $n_{01} = 1$. 
    The key observation is that, to avoid upsetting the equal balance of 0 and 1 values of either feature, any valid panel must contain an equal number of agents with 10 and 01, making them ``linked-fate'' groups. There are $n/2-1$ times as many agents in the former group over which to spread these panel seats, so for any panel distribution, it must be that $\textsf{p}_{01} = \textsf{p}_{10} \, (n/2-1) \in \Theta(n \textsf{p}_{10})$ (see \Cref{app:trade-off} for full proof). Then, for all $z \in (0, 1/n]$ and all $\pivec \in \Pi(\inst)$,
    \begin{align*}
        \min(\pivec)\in \Omega(z) \implies \max(\pivec) \in \Omega(\max\{1/n,z \, n\}).
    \end{align*}
\end{example}

We remark that of these two problems, Problem 2 is far more interesting from an algorithmic design perspective. This is because worst case, \textit{no algorithm can do anything about Problem 1}: this is illustrated by Example 1, whose lower bound arises purely due to the quotas and pool structure, and thus must be suffered by any algorithm. In contrast, when facing Problem 2, different algorithms can make different trade-offs between the maximum and minimum probability (effectively implementing different values of $z$). It is for this reason that, as we will show, \textsc{Goldilocks} is tailored to the goal of addressing Problem 2.








\subsection{Formal Bounds on the Existence of Good Solutions}
With some intuition about factors affecting the extent to which it is possible to simultaneously control high and low probabilities in an instance, we now pursue formal bounds. To do so, we must contend not only with the possibility that Problems 1 and 2 exist in the original instance, but that they can be \textit{created by manipulation where they did not previously exist}. The following lower bound captures both of these possibilities using an instance class that almost directly combines Examples \ref{ex:small-group} and \ref{ex:trade-off}; we defer the proof to \Cref{app:LB-tradeoff}. 

\begin{theorem}[\textbf{Lower Bound}]\label{thm:LB-tradeoff}
    Fix any even $k \in \{6,\dots,n_{min}-3\}$, any $n_{min} \in  \{k+3, \dots, \lfloor n/2\rfloor\}$, and any $c \in \{3,\dots,n_{min}-k\}$. Then there exists $\inst$ with $k$, $n_{min}(\inst) = n_{min}$ satisfying Assumption \ref{ass:inclusion}, $C \subseteq [n]$ with $|C|=c$, and $\boldsymbol{\tilde{w}} \in \boldsymbol{\mathcal{W}}_{\inst,C}$ such that, for all $z \in (0,1/n]$ and all $\tilde{\pivec} \in \Pi(\inst^{\to_C \boldsymbol{\tilde{w}}})$,
    \begin{center}
        $\min(\tilde{\pivec})\geq z \implies \max(\tilde{\pivec}) \in \Omega(\max\{cz,1/(n_{min}-c)\}).$
    \end{center}
\end{theorem}

Theorem 3 shows that manipulators can worsen Problems 1 or 2, or even create them when they did not occur in the original instance (consider, e.g., the lower bound construction with $n_{min} = n/4$). Then, the question is: to \textit{what extent} can manipulators eliminate good solutions via problems 1, 2, or other kinds of problems?
We now answer this question with \Cref{thm:UB-tradeoff}, which gives an upper bound matching \Cref{thm:LB-tradeoff}. Conceptually, this tightness means that in the worst case, we must handle Problems 1 and 2 and \textit{only} these problems. Our handling of these two problems will be reflected in our bounds: as in \Cref{thm:LB-tradeoff}, \Cref{thm:UB-tradeoff} and many of our subsequent bounds contain a $\max\{\cdot,\cdot\}$ term whose first and second entries handle Problem 2 and Problem 1, respectively.
\begin{theorem}[\textbf{Upper Bound}] \label{thm:UB-tradeoff}
Fix any $\inst$ satisfying \Cref{ass:inclusion}, $C \subseteq [n]$ with $|C|=c$ respecting \ref{eq:restriction1}, and any \, $\boldsymbol{\tilde{w}} \in \boldsymbol{\mathcal{W}}_{\inst,C}$. For all $z \in (0,1/n]$, 
there exists $\pivec \in \Pi(\manipinst)$ such that 
   \[\min(\pivec) \in \Omega(z) \qquad  \text{and} \qquad \max(\pivec) \in O(\max\{cz,1/(n_{min}(\inst)-c)\}).\] 
\end{theorem}
\begin{proof}[Proof Sketch] The full proof is intricate and requires significant notation, so we defer it to \Cref{app:impact-ub} and sketch it here. Fix any pair of instances $\inst, \manipinst$ with $C, \tilde{\boldsymbol{w}}$ as specified in the statement. At a high level, we will construct a panel distribution $\tilde{\mathbf{d}}$ in instance $\manipinst$ starting from a ``good'' panel distribution $\mathbf{d}$ in instance $\inst$. In particular, we will choose $\mathbf{d}$ with associated $\pivec$ such that $\min(\pivec) \in \Omega(1/n)$ and $\max(\pivec) \in O(1/n_{min})$. We know such a $\mathbf{d}$ to exist by the following lemma:
\begin{lemma} \label{lem:feasibility-1/n}
    If $\inst$ satisfies Assumption \ref{ass:inclusion}, then $\exists$ $\pivec \in \Pi(\inst)$ such that $\pivec \in [\Omega(1/n), \ O(1/n_{min}(\inst))]^n.$
\end{lemma}
\begin{proof}[Proof Sketch]We prove this lemma in \Cref{app:feasibility-1/n}, but the argument is simple: all agents $i$ must exist on some panel $K_i \in \mathcal{K}$, per \Cref{ass:inclusion}. By defining $\mathbf{d}$ to uniformly randomize over $K_1 \dots K_n$, we guarantee all agents at least probability $\pi_i \in \Omega(1/n)$. The maximum total probability given to any vector group by $\mathbf{d}$ is trivially at most $k$; spread over at least $n_{min}(\inst)$ agents, it follows that for all agents, $\pi_i \leq k/n_{min}(\inst) \in O(1/n_{min}(\inst))$.
\end{proof}

Now, we construct $\tilde{\mathbf{d}}$ from $\mathbf{d}$ as pictured in \Cref{fig:ub-construction} (with some nuances omitted). Let $\overline{C} \subseteq C$ with $|\overline{C}| = \overline{c}$ be the set of agents who misreport a feature vector not in the original pool (i.e., $\tilde{w}_i \notin \mathcal{W}_{\inst}$). For every $i \in \overline{C}$, identify a panel $\tilde{K}_i \in \tilde{\mathcal{K}}$ such that $i \in \tilde{K}_i$ (these need not be unique). Beginning from $\mathbf{d}$, we transfer $\overline{c}z$ total probability mass from $\mathbf{d}$ (in blue) and spread it evenly over the new panels $\tilde{K}_i | i \in \overline{C}$ (in red).
\begin{figure}
    \centering
     \hspace*{1cm} \includegraphics[width=0.9\textwidth]{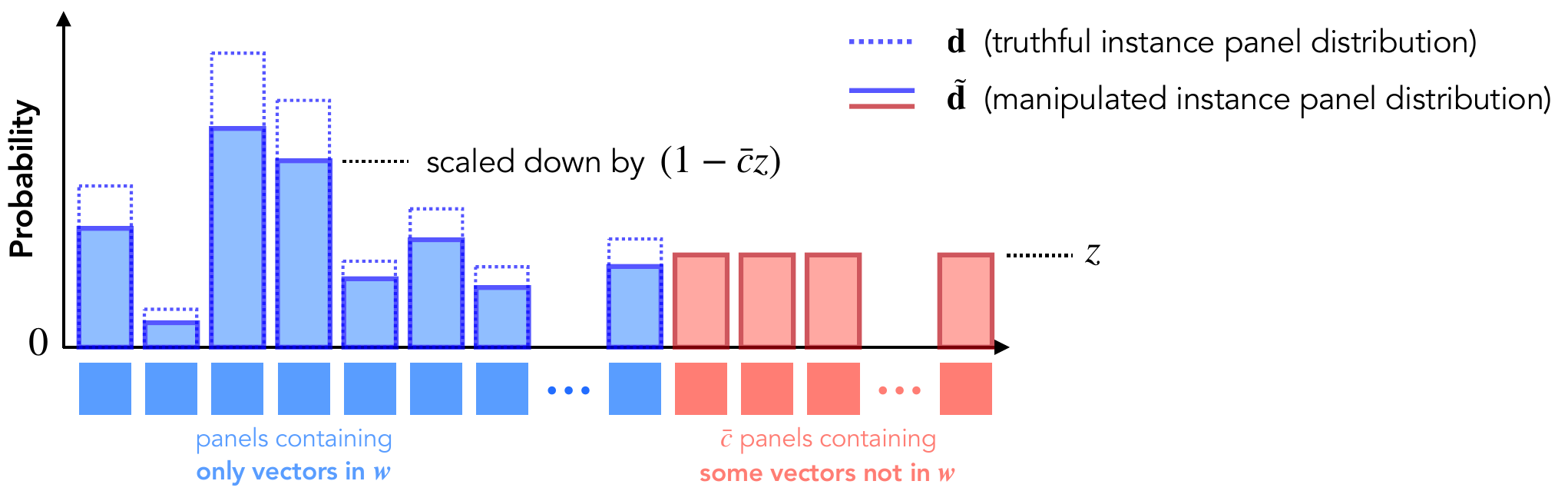}
    \caption{ The solid bars depict $\tilde{\mathbf{d}}$, the distribution we will construct. Blue bars represent the probability mass placed on panels that were \textit{already present} in the support of $\mathbf{d}$, and red bars represent the probability placed on panels newly constructed to contain agents reporting vectors not present in the truthful pool.}
    \label{fig:ub-construction}
\end{figure}
From this construction results the following panel distribution:
\begin{align*}
    \tilde{\pdist}_K = \Big\{
        \pdist_K (1-\overline{c}z) \ \text{ for all } K \in \mathcal{K}, \qquad  z \ \text{ for all } \tilde{K}_i | i \in [n], \qquad 0 \ \text{ else}\Big\}.
\end{align*}
Let $\tilde{\pivec}$ be the probability allocation implied by $\mathbf{d}$. First, those who reported a vector \textit{outside} the original pool exist only on the $\overline{c} \leq c$ red panels, so $\tilde{\pi}_i \in [z,cz]$ for all $i \in \overline{C}$. Next, considering those who reported (either truthfully or not) a vector in the original pool: since any member of this group may be on all $\overline{c}$ red panels, we apply \Cref{lem:feasibility-1/n} to get that $\pi_i \leq \tilde{\pi}_i \leq \pi_i \, \frac{n_{w_i}}{n_{w_i}-c} + \overline{c}z  \leq \pi_i \, \frac{n_{min}}{n_{min}-c} \in O(1/(n_{min}(\inst)-c)+cz)$ for all $i \in [n] \setminus \overline{C}$ (here, the intermediate steps handle the case that the coalition leaves $i$'s vector group, making it smaller and thereby raising its members' probabilities by a factor of at most $n_{min}/(n_{min}-c)$). On the lower end, we have that $\tilde{\pi}_i \geq (1-cz)\pi_i \in \Omega(1/n)$, using that $\pi_i \in \Omega(1/n)$ and $1-cz \geq 1/2 \in \Theta(1)$ (since $z \leq 1/n$ and $c \leq n_{min} \leq n/2$ per \Cref{sec:model}). We conclude that $\tilde{\pi}_i \in [\Omega(1/n),O(1/(n_{min}(\inst)-c)+cz)]$ for all $i \in [n] \setminus \overline{C}$. Taking the union of the two ranges we have deduced (the first for $i \in \overline{C}$, the second for $i \in [n] \setminus \overline{C}$), we conclude that $\tilde{\pivec} \in [\Omega(z), O(\max\{cz,1/(n_{min}(\inst)-c)\})]^n$, as needed.
\end{proof}

\section{\textsc{Goldilocks}$_1$: \textit{Manipulation Robustness} and \textit{Fairness}}
\label{sec:gold-analysis}

Having characterized the quality of solutions that must exist, we now prove bounds on the extent to which \textsc{Goldilocks}$_1$ recovers the best available solutions --- and what this means for our ideals. The bulk of this section is dedicated to proving \Cref{thm:manip-gold}, our main result, which positively bounds \textsc{Goldilocks}$_1$'s manipulation robustness and fairness. We will then discuss the optimality of \textsc{Goldilocks}$_1$, and how \textsc{Goldilocks}$_1$ theoretically compares to (1) \textsc{Goldilocks}$_\gamma$ with a finer-tuned choice of $\gamma$, and (2) the previously-studied algorithms \textsc{Leximin}, \textsc{Nash}, and \textsc{Minimax}.
   \begin{theorem}[\textbf{Upper Bound}]\label{thm:manip-gold} Fix any $\inst$ satisfying \Cref{ass:inclusion} with $n_{min}(\inst) > k$. Then, for all $c \leq n_{min}(\inst) - k$,  
   \begin{align*}
    \text{\upshape{\textsf{manip}}}_{\text{int}}(\mathcal{I},\gold_1,c)  &\in  O(\max\{\nicefrac{\sqrt{c}}{n}, \ \nicefrac{1}{n_{min}(\inst) - c}\}),\\
    \text{\upshape{\textsf{manip}}}_{\text{ext}}(\mathcal{I},\gold_1,c) &\in O(\nicefrac{1}{n_{min}}),\\
    \text{\upshape{\textsf{manip}}}_{\text{comp}}(\mathcal{I},\gold_1,c) &\in O(c\cdot \text{\upshape{\textsf{manip}}}_{\text{int}}(\mathcal{I},\gold_1,c)) + \max_{(f,v) \in FV} (u_{f,v} - \ell_{f,v}),\\
    \text{\upshape{\textsf{manip-fairness}}}(\inst,\gold_1,c) &\in \Omega(\min\left\{\nicefrac{1}{n\sqrt{c}}, \ \nicefrac{(n_{min}-c)}{n^2}\right\}).
   \end{align*}
    \end{theorem}
\begin{proof}
    Fix any $\mathcal{I}$ and $C$ satisfying the requirements of the statement, and fix any $\boldsymbol{\tilde{w}} \in \boldsymbol{\mathcal{W}}_{\inst,C}$. Let $\pivec^* = \pivec^{\gold}(\inst)$ denote the optimal probabilities given by $\gold_1$ in $\inst$, and likewise let $\tilde{\pivec}^* = \pivec^{\gold}(\manipinst)$. The key steps of the proof are in bold.\\[-0.75em] 

    \noindent \textbf{Goldilocks minimizes multiplicative deviation from $k/n$.} We represent the minimum achievable multiplicative deviation from $k/n$ in $\inst$ as  $\delta(\pivec) := \min_{\pivec \in \Pi(\inst)}\max\left\{\frac{k/n}{\min(\pivec)},\frac{\max(\pivec)}{k/n}\right\}.$ Then,\footnote{One may wonder why we do not just directly minimize $\delta(\inst)$ to eliminate the 2-factor in \Cref{lem:gold-prob-bounds}. We study \textit{Goldilocks}$_\gamma$ because for practical reasons, we want an algorithm that continues improving both the minimum and maximum as far as possible, rather than stopping when the maximum multiplicative deviation from $k/n$ is reached.}
    \begin{lemma} \label{lem:gold-prob-bounds} For all $\inst$, $\pivec^{\textsc{Goldilocks}_1}(\inst) \in [k/n \cdot (2\delta(\inst))^{-1}, \ \ k/n \cdot 2\delta(\inst)]^n.$
\end{lemma}
\begin{proof}
    Fix any $\inst, \pivec \in \Pi(\inst)$ and let $\pivec^* = \pivec^{\textsc{Goldilocks}_1}(\inst)$. Then, by the optimality of $\pivec^*$, 
    \begin{center}
        $Goldilocks_1(\pivec^*) \leq \textit{Goldilocks}_1(\pivec) = \frac{\max(\pivec)}{k/n} +  \frac{k/n}{\min(\pivec)} \leq 2\delta(\pivec).$
    \end{center}
    We apply this bound to bound each term of \textit{Goldilocks}$_1(\pivec^*)$ separately to conclude: 
    \begin{align*}
        \max(\pivec^*)/(k/n) \leq 2\delta(\pivec) &\iff \max(\pivec^*) \leq k/n \cdot 2\delta(\pivec)\\
        (k/n)/\min(\pivec^*) \leq 2\delta(\pivec) &\iff \min(\pivec^*) \geq k/n \cdot (2\delta(\pivec))^{-1}. \qedhere
    \end{align*}
\end{proof}
    
    \noindent \textbf{$\delta(\inst)$ and $\delta(\manipinst)$ are bounded.} First considering $\inst$, we know that there exists $\pivec \in \Pi(\inst)$ such that $\max(\pivec) \in O(1/(n_{min}))$ and $\min(\pivec) \in \Omega(1/n)$ (\Cref{lem:feasibility-1/n}). It follows that 
    \begin{equation} \label{eq:orig-solution} \delta(\inst) \leq \max\left\{\frac{k/n}{\Omega(1/n)}, \ \frac{O(1/n_{min})}{k/n}\right\} \in O\left(\max\left\{1,\frac{n}{n_{min}}\right\}\right) = O\left(\frac{n}{n_{min}}\right).
    \end{equation}
    Next considering $\manipinst$, we know that for all $z \in (0,1/n]$, there exists $\pivec \in \Pi(\manipinst)$ such that $\max(\pivec) \in  O(\max\{cz,1/(n_{min}-c)\})$ and $\min(\pivec) \in  \Omega(z)$ (\Cref{thm:UB-tradeoff}). We elect to apply this bound with $z = 1/(n\sqrt{c})$. Conceptually, we choose this $z$ because this is a trade-off that \textsc{Goldilocks}$_1$ is particularly well-suited to recover: \textsc{Goldilocks}$_1$ penalizes deviations above and below $k/n$ equally, and this choice of $z$ correspondingly positions the requisite multiplicative probability gap of $c$ symmetrically over $k/n$ (i.e., $z = 1/(n\sqrt{c})$ and $cz = \sqrt{c}/n$). It follows that
    \begin{equation} \label{eq:manipbounds}
        \delta(\manipinst) \leq \max\left\{\frac{k/n}{\Omega(1/(n\sqrt{c}))}, \ \frac{O(\max\{\sqrt{c}/n,\ 1/(n_{min}-c)\})}{k/n}\right\} \in O\left(\max \left\{\sqrt{c}, \frac{n}{n_{min}-c}\right\}\right).
    \end{equation}
    \textbf{Maximum and minimum probabilities given by \textsc{Goldilocks}$_1$ are bounded.} Combining \Cref{lem:gold-prob-bounds} and \Cref{eq:orig-solution,eq:manipbounds}, we conclude that
    \begin{center}
        $\pivec^* \in \left[\Omega\left(\frac{n_{min}}{n^2}\right), \ O\left(\frac{1}{n_{min}}\right)\right]^n, \quad \tilde{\pivec}^* \in \left[\Omega\left(\min\left\{\frac{1}{n\sqrt{c}}, \ \frac{n_{min}-c}{n^2}\right\}\right), \ O\left(\max\left\{\frac{\sqrt{c}}{n}, \ \frac{1}{n_{min}-c}\right\}\right)\right]^n.$ 
    \end{center}
    
    \noindent \textbf{Concluding the proof.} The bounded probabilities above imply an upper bound on the largest possible gain in probability by any agent:\[\text{\upshape{\textsf{manip}}}_{int}(\inst,\textsc{Goldilocks}_1,c) \leq \max_{i \in [n]}(\tilde{\pi}_i^* - \pi_i^*) \in O\left(\max\left\{\sqrt{c}/n, \ 1/(n_{min}-c)\right\}\right).\]
    
    They also imply a bound on the largest possible \textit{loss} in probability for any agent: 
    \[\text{\upshape{\textsf{manip}}}_{\text{ext}}(\mathcal{I},\gold_1,c) \leq \max_{i \in [n]} (\pi^*_i - \tilde{\pi}^*_i)  \in  O(1/n_{min}).\]
    Analyzing $\mcomp$ requires a little more care, because if the quotas are not perfectly tight (e.g., if some group $f,v$ is permitted to receive some \textit{range} of seats $u_{f,v} - \ell_{f,v} > 0$), manipulators can gain seats for this group not just by misreporting value $v$ for $f$, but also by ensuring that all panels in the support of the distribution give this group $u_{f,v}$ seats instead of $\ell_{f,v}$. 
    To make this formal, let $\tilde{f}:[n] \to V_f$ map each agent $i$ to their \textit{reported} value for feature $f$. In the argument below, for each $(f,v) \in FV$, we divide the quantity we want to bound into two, where the first represents the probability garnered among agents who honestly report value $v$ for feature $f$, and the second is the probability garnered among agents who do \textit{not} report value $v$ for feature $f$ in the post-manipulation pool, but truly possess that feature. Then, $\mcomp(\mathcal{I},\gold_1,c)$ is equal to
    \begin{align*}
\max_{(f,v) \in FV} \sum_{i : f(i) = v} (\tilde{\pi}^{*}_{i} - \pi_i^*)
        &= \max_{(f,v) \in FV} \sum_{i : f(i) = v \land \tilde{f}(i) = v} \hspace*{-0.5cm}(\tilde{\pi}^{*}_{i} - \pi_i^*) \ \ +\sum_{i :  f(i) = v \land \tilde{f}(i) \neq v} \hspace*{-0.5cm} (\tilde{\pi}^{*}_{i} - \pi_i^*)\\
        & \leq \max_{(f,v) \in FV} u_{f,v} - \ell_{f,v} + |\{i | f(i)=v \land \tilde{f}(i) \neq v\}| \cdot \max_{i \in [n]}(\tilde{\pi}^{*}_{i} - \pi_i^*)\\
        & \leq \max_{(f,v) \in FV} u_{f,v} - \ell_{f,v} + O(c \cdot \text{\upshape{\textsf{manip}}}_{int}(\inst,\textsc{Goldilocks}_1,c)).
    \end{align*} 
   
   \[\text{ Finally,} \  \ \text{\upshape{\textsf{manip-fairness}}}(\mathcal{I},\gold_1,c) \geq \min_{i \in [n]} \tilde{\pi}_i^* \in \Omega\left(\min\left\{1/(n\sqrt{c}), \ (n_{min}-c)/n^2\right\}\right). \qedhere\] 
\end{proof}

\subsection{On the Optimality of \textsc{Goldilocks}}
When considering the optimality of \textsc{Goldilocks}, we will primarily consider \textsf{manip}$_{int}$ and \textsf{manip-fairness}. We comment on the tightness and practicality of our other upper bounds in \Cref{app:tightness}. We first extend our impossibility on controlling high and low probabilities (\Cref{thm:LB-tradeoff}) to prove a corresponding impossibility on the ability to simultaneously guarantee
fairness and manipulation robustness. This theorem is proven with almost an identical construction to the proof of \Cref{thm:LB-tradeoff}, with a slightly more complicated coalition structure. We defer the proof to \Cref{app:lb}.
\begin{theorem}[\textbf{Lower Bound}] \label{thm:lb}
    Fix any algorithm \textsc{A}, any even $k \in \{6,\dots,n_{min}-5\}$, $n_{min} \in \{k+5, \dots, \lfloor n/k\rfloor\}$ and $c \in \{5,\dots,n_{min}-k\}$. There exists $\inst$ with $k, n_{min}(\inst) = n_{min}$ satisfying Assumption \ref{ass:inclusion} such that for all $z \in (0,1/n]$,
    \begin{center}
        $\textsf{manip-fairness}(\inst,\algo,c) \in \Omega(z) \implies 
\mint(\inst,\algo,c) \in \Omega(\max\{cz,1/(n_{min}-c)\})$.
    \end{center}
\end{theorem}
Observe that \textsc{Goldilocks}$_1$ \textit{almost} achieves the optimal trade-off at $z = 1/(n\sqrt{c})$: our bound on \textsf{manip}$_{int}$ in \Cref{thm:manip-gold} matches that in \Cref{thm:lb}, and our bound on \textsf{manip-fairness} in \Cref{thm:manip-gold} matches \Cref{thm:lb} \textit{under the condition that $n_{min}, n_{min} - c \in \Omega(n)$}. In fact, this special case where \textsc{Goldilocks}$_1$ is optimal is not entirely unnatural: it holds \textit{in expectation} in the real-world panel selection process, and thus should hold by concentration for sufficiently large $n$.\footnote{Were $n$ to be increased in practice, it would be done by increasing the number of letters sent out in stage 1. Because these letters are sent out uniformly at random, all vector groups present in the population would in expectation compose a constant fraction of the pool. This guarantees that $n_{min} \in \Omega(n)$; it is a tiny leap to strengthen \Cref{eq:restriction1} such that $c \leq (1-\epsilon)n_{min} - k$ for any constant $\epsilon > 0$, in which case $n_{min} - c \in \Omega(n)$ as well.} 
\begin{observation}
    For all $\inst,c$ such that $n_{min}(\inst), n_{min}(\inst)-c \in \Omega(n)$, \textsc{Goldilocks}$_1$ is optimal $z = 1/(n\sqrt{c})$ (i.e., guarantees \textsf{manip}$_{int}$ and \textsf{manip-fairness} matching \Cref{thm:lb} at this $z$ in $\inst$).
\end{observation}

The fact that \textsc{Goldilocks}$_1$ is near-optimal for generic $c$ (and optimal when $n_{min}$ is large) may seem surprising, because the algorithm has no knowledge of $c$ with which to tune $\gamma$. This achievement is a consequence of the fact that when Problem 2 is the main barrier to good solutions (e.g., when $n_{min}$ is large), it is always possible to place the $c$-dependent fundamental trade-off between maximum and minimum probabilities symmetrically over $k/n$ (\Cref{thm:UB-tradeoff}). \textsc{Goldilocks}$_1$ will do so, thus achieving an optimal trade-off. If Problem 1 is instead the main barrier to good solutions (e.g., when $n_{min}$ is small), then the minimal achievable deviations above and below $k/n$ are \textit{asymmetric}, with the above deviation being larger. Here, \textsc{Goldilocks}$_1$ achieves the optimal maximum probability but is not guaranteed to recover the best possible minimum probability due to the asymmetry. This is the sense in which \textsc{Goldilocks}$_1$ is tailored to Problem 2, as discussed in \Cref{sec:intuition-problems}.

Nonetheless, one may wonder if there exists a setting of $\gamma$ such that \textsc{Goldilocks}$_\gamma$ can perfectly recover our bound in \Cref{thm:lb} for generic $z,n_{min},c$. If not, this would suggest some weakness in the \textit{functional form} of \textit{Goldilocks}$_\gamma$, and perhaps some other objective could perform better. Fortunately, we find that there does exist such a $\gamma$. We defer the proof to \Cref{app:general-opt-gold}.
\begin{proposition}\label{prop:general-opt-gold}
      Fix any $\inst$ satisfying \Cref{ass:inclusion} with $n_{min}(\inst) > k$, $c \leq n_{min}(\inst) - k$, and $z \in (0, 1/n]$.  
    Let $\gamma^* = z  \cdot \max\{1/(n_{min}-c),cz\} \cdot (n/k)^2$. Then,
    \begin{align*}
        \textsf{manip-fairness}(\inst,\gold_{\gamma^*},c) &\in \Omega(z)  \quad \text{and}\\
\mint(\inst,\gold_{\gamma^*},c) &\in O(\max\{cz,1/(n_{min}-c)\}).
    \end{align*}
\end{proposition}
 
Of course, the catch with this setting of $\gamma = \gamma^*$ is that it depends on $c$, which is unknown in practice. While one may be tempted to pursue more sophisticated, instance-dependent settings of $\gamma$, it will turn out that $\gamma = 1$ is already practically good enough: in \Cref{sec:empirics}, we will find that \textsc{Goldilocks}$_1$ has near-instance-optimal performance across several real-world datasets. It will also significantly empirically outperform all previously-studied algorithms.

\subsection{Comparison to Existing Algorithms}
Before empirically comparing \textsc{Goldilocks}$_1$ to previously-studied algorithms \textsc{Minimax}, \textsc{Leximin}, and \textsc{Nash}, we establish these algorithms' theoretical separation within our model (\Cref{thm:lb-maximin-nash,thm:lb-minimax}).\footnote{One may wonder why we do not also compare to the linear analog of \textit{Goldilocks}$_\gamma$, $\max(\pivec) - \gamma \min(\pivec)$. We prove a lower bound on this objective in  \Cref{app:linear}. At a high level, this objective is limited because it does not penalize low probabilities steeply enough relative to high ones; \textit{Goldilocks}$_\gamma$ has a steeper gradient, allowing it to reach better trade-offs.} We defer the proofs of these results to \Cref{app:lb-maximin-nash,app:minimax}, as they are technically involved but conceptually only mildly generalize results proven in  \cite{flanigan2024manipulation}.

First, it is easy to see that \textsc{Minimax} makes a highly undesirable trade-off, giving no guarantees on fairness \textit{even when $n_{min}$ is large and there is no manipulation:}
\begin{proposition}\label{thm:lb-minimax}
There exists $\inst$ with $n_{min}(\inst) \in \Omega(n)$ satisfying Assumption \ref{ass:inclusion} such that \[\text{\upshape{\textsf{fairness}}}(\inst,\textsc{Minimax}) = 0.\]
\end{proposition}

In contrast, \textsc{Leximin} and \textsc{Nash} are not as obviously poor, but strike a different trade-off than \textsc{Goldilocks}$_1$: both algorithms prioritize essentially only low probabilities, roughly akin to selecting $z \in \Theta(1/n)$ in \Cref{thm:UB-tradeoff} instead of $z \in \Theta(1/(\sqrt{c}n))$, as does \textsc{Goldilocks}$_1$:
\begin{proposition}\label{thm:lb-maximin-nash}
Fix any even $k \in \{6,\dots,n_{min}-5\}$, $n_{min} \in \{k+5, \dots, \lfloor n/k\rfloor\}$ and $c \in \{6,\dots,n_{min}-k\}$. There exists $\inst$ with $k,n_{min}(\inst) = n_{min}$ satisfying Assumption \ref{ass:inclusion} such that
    \[ \text{\upshape{\textsf{manip}}}_{int}(\inst,\textsc{Leximin},c), \ \  \text{\upshape{\textsf{manip}}}_{int}(\inst,\textsc{Nash},c) \ \in \ \Omega(\max\{c/n,1/(n_{min}-c)\}).\]
\end{proposition}
To motivate our comparison of this bound with our upper bound on $\mint(\inst,\gold_1,c)$ in \Cref{thm:manip-gold}, we note that one of the simplest ways to improve manipulation robustness in practice is to increase the pool size $n$. It is therefore desirable for an algorithm's manipulation robustness to decline quickly in $n$. This is precisely where \textsc{Leximin} / \textsc{Nash} and \textsc{Goldilocks}$_1$ differ: unless $n_{min}$ is prohibitively small (in which case both algorithms are subject to Problem 1), the manipulation robustness of \textsc{Leximin} and \textsc{Nash} declines at a rate of $c/n$ while \textsc{Goldilocks}$_1$ achieves $\sqrt{c}/n$. This may seem like a small difference, but suppose manipulators make up a constant fraction of the pool (i.e., $c \in \Theta(n)$). Then, $c/n \in \Theta(1)$ and \textsc{Leximin} and \textsc{Nash} have \textit{arbitrarily poor} manipulation robustness, even in the best case where $n$ and $n_{min}$ are large. In contrast, for any $c$, \textsc{Goldilocks}$_1$'s manipulation robustness must decline in $n$ at a rate of $O(\sqrt{c}/n) \in O(1/\sqrt{n})$.

\section{\textsc{Goldilocks}$_1$: Manipulation Robustness, Fairness, and \textit{Transparency}}\label{sec:transparency}
Finally, we extend our bounds from \Cref{thm:manip-gold} to the \textit{transparent extension} of \textsc{Goldilocks}$_1$, defined as $\mathcal{R}_m \circ$\textsc{Goldilocks}$_1$ for some $\mathcal{R}_m$. To do so, we will use the following key lemma from previous work, which shows that there exists an $\mathcal{R}_m$ that can round \textit{any} panel distribution to an $m$-uniform lottery while changing all agents' selection probabilities to only a bounded degree:

\begin{lemma}[Thms 3.2 and 3.3, \cite{flanigan2021transparent}] \label{thm:transparency-oldpaper}
For all $\inst$ and $m \in \mathbb{Z}^+$, there exists an $\mathcal{R}_m$ such that for all $\pdist$ with corresponding $\pivec$, it holds for $\mathcal{R}_m(\pdist) = \overline{\pdist}$ with corresponding $\overline{\pivec}$ that
        \begin{center}
            $\|\pivec-\overline{\pivec}\|_\infty \in O\left(\min \left\{k,\sqrt{|\mathcal{W}_\inst|\log(|\mathcal{W}_\inst|)}\right\}/m\right).$
        \end{center}
\end{lemma}
When applying this bound to our setting, we run into a problem: by \Cref{thm:lb}, the minimum probability may be dropping at a rate of $1/n\sqrt{c}$, which can be as low as order $1/n\sqrt{n}$. In contrast, the above bound does not shrink in $n$, so as $n$ grows (which is beneficial for manipulation robustness), the upper bound above will exceed the minimum probability in the pre-rounded instance, resulting in a fairness guarantee of 0. We solidify this concern by proving a lower bound showing that rounding the $\gold$-optimal solution can indeed decrease the minimum probability by up to $\sqrt{k}/m$ (\Cref{prop:lbrounding}, \Cref{app:lbrounding}).
  
    To avoid this issue, we need $m$ to grow at a rate of $\Omega(n \sqrt{n})$. Fortunately, in practice it is much easier to scale up $m$ than $n$: scaling up $n$ by a factor of 10 requires sending out 10 times as many letters, while multiplying $m$ by 10 just requires adding another lottery bin (so panels are numbered 0000 - 9999 instead of 000 - 999). We thus assume that $m \geq n \sqrt{n}$. 
    
    We must deal with one more wrinkle: after manipulation, the number of unique feature vectors in the pool may grow from $|\mathcal{W}_\inst|$ to at most $|\mathcal{W}_\inst + c|$. Combining this observation, \Cref{thm:transparency-oldpaper}, and \Cref{thm:manip-gold}, we conclude that there exists an $\mathcal{R}_m$ such that a transparent algorithm $\mathcal{R}_m \circ \textsc{Goldilocks}_1$ achieves the following simultaneous fairness 
and manipulation robustness guarantees. Here, we use the shorthand $\triangle(\inst) := \min\{k,\sqrt{|\mathcal{W}_\inst+c| \, \log(|\mathcal{W}_\inst +c|)}\}/m$ (as in \Cref{thm:transparency-oldpaper}).
    \begin{theorem}[\textbf{Upper Bound}] Fix any $m \geq n\sqrt{n}$ and $\inst$ satisfying \Cref{ass:inclusion} with $n_{min}(\inst) > k$. Then, there exists an $\mathcal{R}_m$ such that for all $c \leq n_{min}(\inst) - k$,
       \begin{align*}
\text{\upshape{\textsf{manip}}}_{int}&(\mathcal{I},\mathcal{R}_m \circ \gold_1,c) \in O\left(\text{\upshape{\textsf{manip}}}_{int}(\mathcal{I},\gold_1,c) + \triangle(\inst)\right),\\[-0.25em]
           \text{\upshape{\textsf{manip}}}_{ext}&(\mathcal{I},\mathcal{R}_m \circ \gold_1,c) \in O\left(\text{\upshape{\textsf{manip}}}_{ext}(\mathcal{I}, \gold_1,c) + \triangle(\inst)\right),\\
           \text{\upshape{\textsf{manip}}}_{comp}&(\mathcal{I},\mathcal{R}_m \circ \gold_1,c) \in O\left(\text{\upshape{\textsf{manip}}}_{comp}(\mathcal{I}, \gold_1,c)  + c\, \triangle(\inst)\right),\\ 
           \text{\upshape{\textsf{manip-fairness}}}&(\mathcal{I},\mathcal{R}_m \circ \gold_1,c) \in \Omega\left(\text{\upshape{\textsf{manip-fairness}}}(\mathcal{I},\gold_1,c) - \triangle(\inst)\right).
       \end{align*}
    \end{theorem}

 \section{Empirical Evaluation} \label{sec:empirics}

    \textbf{Instances.} We analyze 9 instances of real-world panel selection data identified only by number (their sources are anonymized). The relevant properties of these instances are in \Cref{app:instances}.\\[-0.9em]

   \noindent \textbf{Algorithms.} We evaluate four algorithms implementing \citet{flanigan2021fair}'s maximally fair algorithmic framework: \textsc{Leximin}, \textsc{Nash}, \textsc{Minimax}, and \textsc{Goldilocks}$_1$.\footnote{One might wonder if an \textit{instance-specific} setting of $\gamma$ might perform better than $\gamma=1$. We evaluate two natural such definitions of $\gamma$, but find that they make little difference to the performance of \textsc{Goldilocks}$_\gamma$ (see \Cref{sec:alternate-gammas}, \Cref{tab:alternate-gammas}).} For our most computationally-intensive experiments, we replace \textsc{Leximin} with \textsc{Maximin}, as it runs much faster on large instances but behaves similarly with respect to the properties we aim to test.\footnote{Because \textsc{Leximin}/ \textsc{Maximin} and \textsc{Minimax} are optimizing such low dimensional features of $\pivec$, one may wonder if tie-breaking may improve these algorithms (e.g., fix the minimum probability achieved by the \textsc{Maximin}-optimal solution, and then minimize the maximum probability). As shown in \Cref{tab:alternate-gammas}, we find that this helps somewhat in some instances, but \textsc{Goldilocks}$_1$ still substantially dominates, suggesting that \textsc{Goldilocks}$_1$ it is finding a non-trivially good trade-off. Thus, for consistency with past work and real-world implementations, our main results do not implement tie-breaking.} We also analyze the selection algorithm \textsc{Legacy}, which is a greedy heuristic that was used widely in practice, and serves here as a benchmark representing greedy algorithms that remain in use (see \Cref{app:legacy} for details). 
   In some analyses, we consider only a key subset of these algorithms: $\minimax, \leximin$, and $\gold_1$. \Cref{app:algo-implementation} describes our implementation of \citet{flanigan2021fair}'s framework.

  \subsection{Maximum and Minimum Probabilities}
  We first compare algorithms' ability to simultaneously control the maximum and minimum probability. In \Cref{tab:maxes-mins}, for each algorithm \textsc{A} and instance $\inst$, we report how closely the maximum and minimum probability given by \textsc{A} approximate the optimal minimum probability (given by \textsc{Maximin}) and optimal maximum probability (given by \textsc{Minimax}). Formally, table entries are $\left(\min(\pivec^{\textsc{A}}(\inst))\, / \, \min(\pivec^{\textsc{Maximin}}(\inst)), \ \max(\pivec^{\textsc{A}}(\inst)) \, / \, \max(\pivec^{\textsc{Minimax}}(\inst))\right)$. 
 \begin{table}[h!] 
        \centering
        \vspace{-1em}
        \begin{tabular}{c||cc|ccc|c}
         \multicolumn{7}{c}{\textbf{Equality Notions}}\\
         $\inst$ &     \textsc{Legacy} &     \textsc{Minimax} &    \textsc{Maximin} &     \textsc{Leximin} &     \textsc{Nash} &  \textsc{Goldilocks}$_1$ \\
         \hline
         1 &  (0.0, 1.14) &  (0.0, 1.0) &  (1.0, 2.0) & (1.0, 2.0) &  (0.62, 2.0) &         (0.72, 1.1) \\
        2 & (0.03, 1.01) &  (0.0, 1.0) &  (1.0, 1.33) & (1.0, 1.33) & (0.67, 1.33) &           (0.9, 1.0) \\
        3 &   (0.0, 1.0) &  (0.0, 1.0) &  (1.0, 1.0) &  (1.0, 1.0) &  (0.61, 1.0) &          (1.0, 1.0) \\
        4 &  (0.01, 1.0) &  (0.0, 1.0) &  (1.0, 1.0) & (1.0, 1.0) &  (0.61, 1.0) &          (1.0, 1.0) \\
        5 &  (0.0, 1.02) &  (0.0, 1.0) &  (1.0, 1.17) & (1.0, 1.17) & (0.57, 1.17) &          (0.95, 1.0) \\
        6 & (0.66, 1.11) & (0.25, 1.0) &  (1.0, 1.5) &  (1.0, 1.11) &  (0.9, 1.08) &          (1.0, 1.09) \\
        7 &  (0.0, 2.18) &  (0.0, 1.0) &  (1.0, 3.5) &  (1.0, 3.5) &  (0.46, 3.5) &          (0.7, 1.38) \\
        8 &   (0.0, 1.0) &  (0.0, 1.0) &  (1.0, 1.0) &  (0.98, 1.0) &  (0.78, 1.0) &          (1.0, 1.0) \\
        9 &   (0.0, 1.0) &  (0.0, 1.0) &  (1.0, 1.0) &  (1.0, 1.0) &  (0.45, 1.0) &          (1.0, 1.0) \\ 
        \end{tabular}
   \caption{Approximations to the optimal minimum, maximum probabilities across algorithms, instances. In instance 8, \textsc{Leximin} is slightly worse than \textsc{Maximin} due to numeric convergence errors in the solver.}
        \label{tab:maxes-mins}
        \vspace{-1em}
    \end{table}

    What we see is already encouraging: in 7 out of 9 instances, $\gold_1$ achieves within 10\% of the optimal maximum \textit{and} minimum probabilities. This is striking, because it was not even clear \textit{a priori} that this would be possible for \textit{any} algorithm. In contrast, we see that \textsc{Legacy} and \textsc{Minimax} perform poorly on low probabilities, \textsc{Maximin}/\textsc{Leximin} perform poorly on high probabilities, and \textsc{Nash} performs somewhat poorly on both. However, these results do not paint a complete picture: in instances 3, 4, 8, and 9, \textit{all} algorithms achieve the optimal maximum probability \textit{simply because the quotas require an agent to receive probability 1}. Thus, to fully compare the performance of these algorithms, we must examine their performance on less constrained\emdash but still realistic\emdash instances. We therefore study these algorithms' maximum and minimum probabilities as we successively drop features in decreasing order of their selection bias, as in Figure 1c of \cite{flanigan2024manipulation}. Details on feature dropping and results for omitted instances are in \Cref{app:feature-drop}.

    \begin{figure}[t!]
    \centering
    \includegraphics[width=0.75\textwidth]{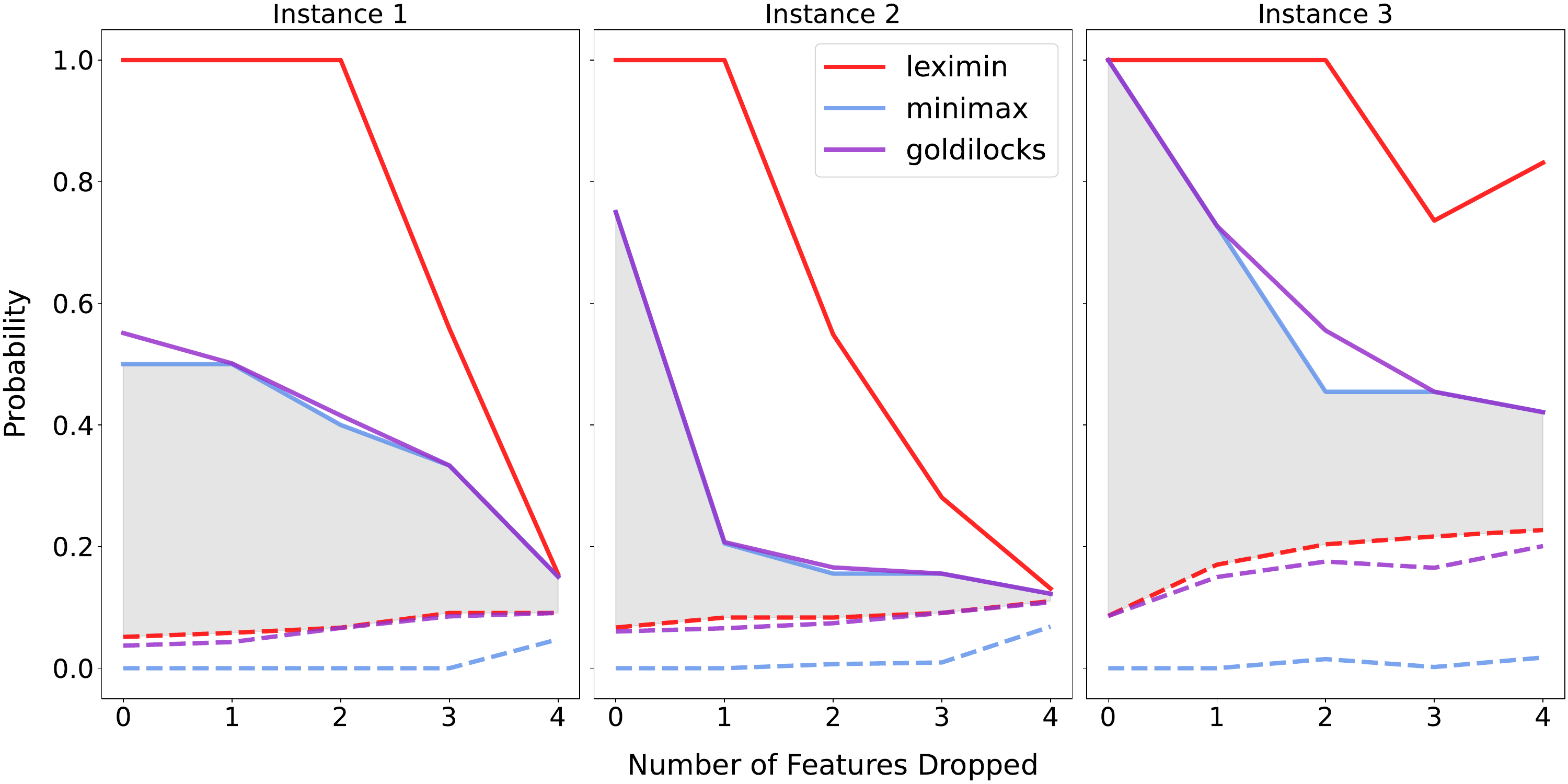}
    \caption{The solid, dashed lines respectively represent maximum, minimum probabilities per algorithm. The shaded region lies between the optimal maximum probability and optimal minimum probability, establishing the region where no algorithm's extremal probabilities can exist. }
    \vspace{-0.75em}
    \label{fig:feature-drop}
    \end{figure}
     \Cref{fig:feature-drop} shows something striking: $\gold_1$ hugs the gray region almost perfectly above and below, thus maintaining near optimality as features are dropped. This is in contrast to \textsc{Leximin} and \textsc{Minimax}, which respectively continue to perform poorly on high and low probabilities, even as the instance is loosened and better probabilities are possible. Together, these results show that controlling high and low probabilities simultaneously is generally possible to a great extent, and all previously explored algorithms were leaving a lot on the table with respect to this goal.

    \subsection{Fairness, Manipulation Robustness, and Transparency}    
    \textbf{Fairness.} While the above results already show the performance of all algorithms on \textit{Maximin} fairness, there are other normatively justified notions of fairness. In \Cref{fig:gini}, we additionally evaluate fairness according to the \textit{Gini Coefficient}, defined as $\text{\textit{Gini}}(\pivec):=\frac{\sum_{i,j \in [n]}|\pi_i - \pi_j|}{2\sum_{i,j \in [n]}\pi_i \pi_j}.$ Note that a smaller Gini Coefficient reflects greater fairness, as \textit{Gini} measures \textit{inequality.}
    
    \Cref{fig:gini} shows similar algorithmic behavior across instances: \textsc{Legacy} and \textsc{Minimax}\emdash which we expect to be very unfair\emdash tend to have high inequality per \textit{Gini}. In contrast, $\gold_1$ and \textsc{Leximin} perform far better. Unsurprisingly, \textsc{Leximin} is slightly better, as it prioritizes fairness alone.
    \begin{figure}[h!]
    \centering
    \vspace{-1em}
\includegraphics[width=0.75\textwidth]{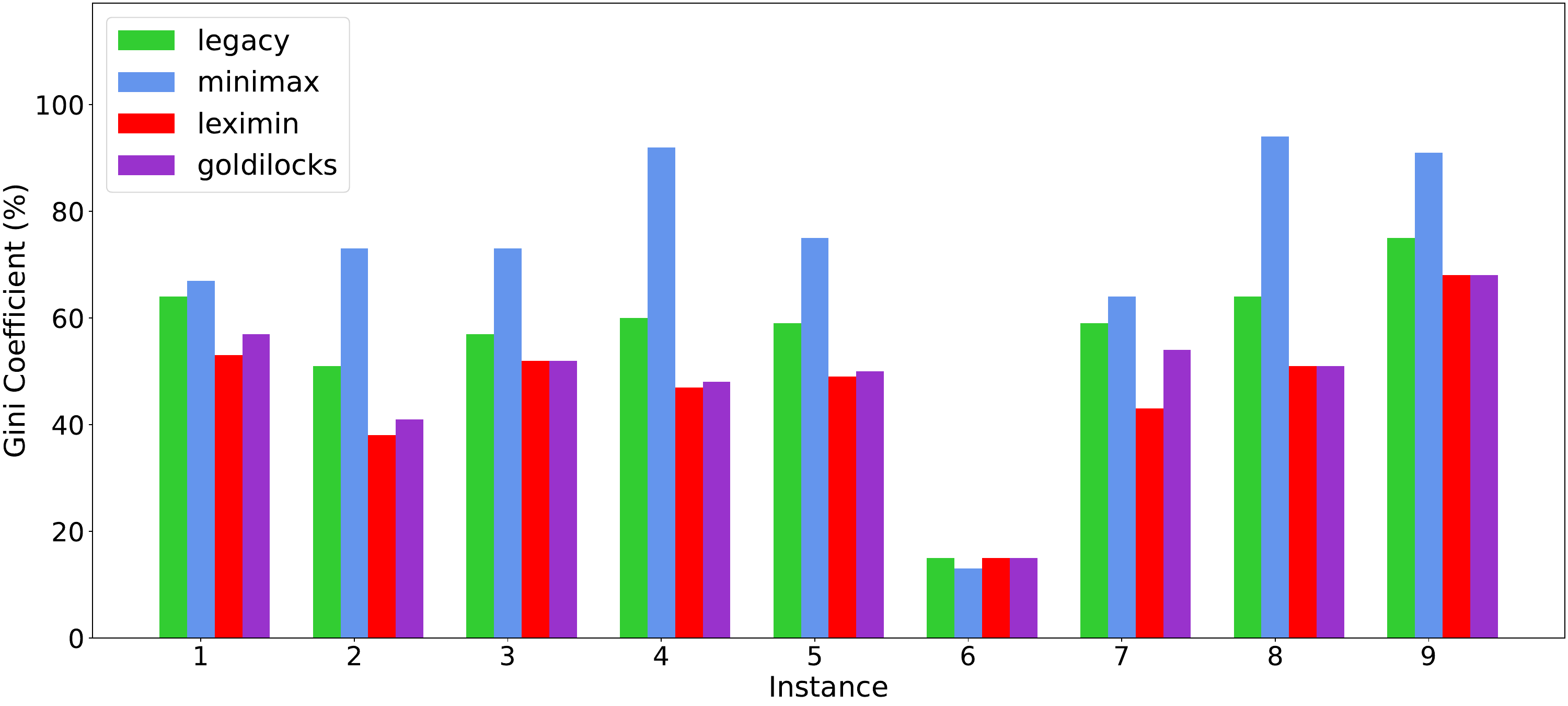}
    \vspace{-0.5em}
    \caption{Gini coefficient across algorithms and instances. Lower Gini Coefficient means greater fairness.}
    \label{fig:gini}
    \vspace{-0.75em}
    \end{figure}

    \noindent \textbf{Manipulation Robustness.} In accordance with previous work, we evaluate manipulation robustness by measuring the maximum probability gainable by any single, \textit{weakened} manipulator. This manipulator uses the strategy \textit{Most Underrepresented (MU)}, meaning they report the value of each feature that is most disproportionately underrepresented in the pool (as studied in \citet{flanigan2024manipulation}). We evaluate how this probability changes as $n$ grows, which we simulate by simply duplicating the pool. Details on these experiments are found in \Cref{app:manip-robust}.

    In \Cref{fig:manip}, we see that in instances 1 and 2, $\gold_1$ is far less manipulable than \maximin; in instance 3, we know the quotas require some agents to receive probability 1. However, we see that as the pool is duplicated, $\gold_1$ makes use of this and the manipulation drops; in contrast, across instances, \textsc{Maximin} remains just as manipulable. From Figures \ref{fig:manip} and \ref{fig:gini}, we conclude that $\gold_1$ achieves meaningful gains in manipulation robustness over \textsc{Leximin}\emdash the practical state-of-the-art\emdash without any meaningful cost to fairness, as desired.
    \begin{figure}[t!]
    \centering    \includegraphics[width=0.8\textwidth]{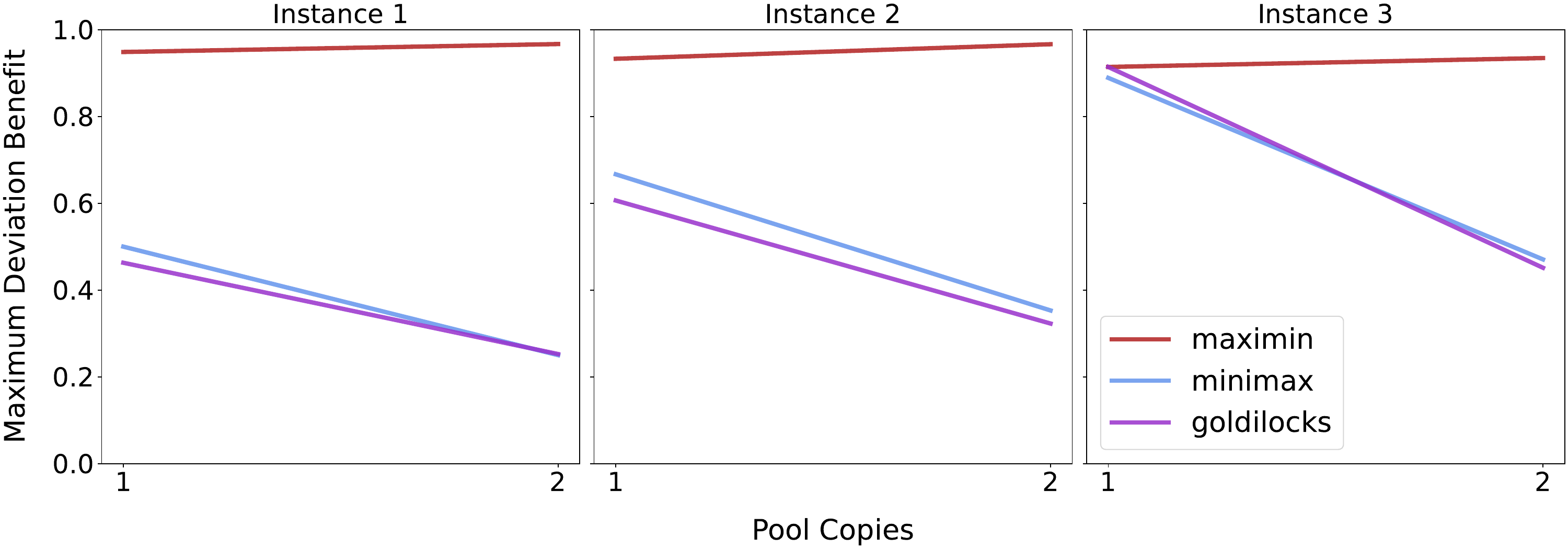}
    \vspace{-0.5em}
    \caption{The maximum amount of probability any single \textit{MU} manipulator can gain, for 1 and 2 pool copies.}
    \vspace{-0.5em}
    \label{fig:manip}
    \end{figure}

    \vspace{0.25em}
    \noindent \textbf{Transparency.} Finally, we evaluate the extent to which we can round $\gold_1$-optimal panel distributions to $m$-uniform lotteries without losing too much on high or low probabilities. In this analysis, we use $m=1000$. Although this is lower than $n \sqrt{n}$ as our theory dictates, we will find that this practicable number of panels is sufficient for good performance.
    
    We use \textit{Pipage} rounding \cite{gandhi2006dependent}, a simple randomized dependent rounding procedure. Although this algorithm does not come with any formal guarantees on how much its rounded distribution will change agents' selection probabilities, \textit{Pipage} is fast; already implemented in practice for the purposes of transparent sortition \cite{stratificationapp}; and has the added advantage that, over the randomness of the rounding \textit{and} sampling, it \textit{perfectly} preserves the selection probabilities, thereby exactly maintaining our guarantees in \Cref{thm:manip-gold} end-to-end. Details on our rounding algorithm and experimental methods are in \Cref{app:transparency}.

    \Cref{fig:transparency} shows good news: \textit{Pipage} is reliably leaving $\gold_1$'s optimal selection probabilities essentially unchanged. This is great news, because it means that we can have the best of both worlds: we can achieve a high-quality uniform lottery \textit{while} preserving our fairness and manipulation robustness guarantees from \Cref{thm:manip-gold} exactly, end-to-end.
    
     \begin{figure}[h!]
     \vspace{-0.75em}    \includegraphics[width=0.75\textwidth]{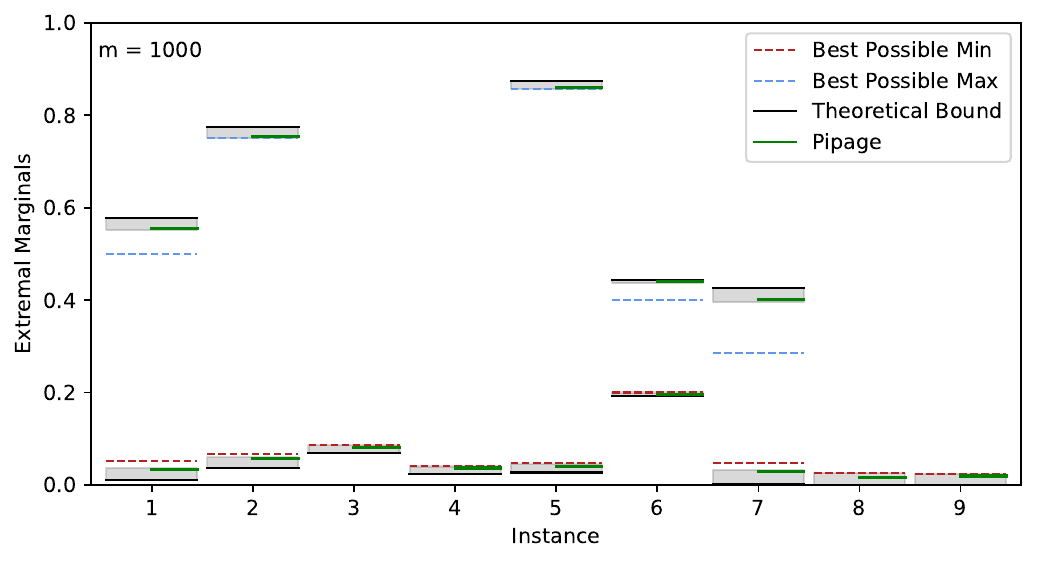}
    \vspace{-1em}
         \caption{Deviations from $\gold_1$-optimal selection probability assignments by \textit{Pipage}. The values for \textit{Pipage} correspond to averages of minimum, maximum probability per run over 1000 runs. Error bars are plotted to indicate standard deviation, but they are so small that they are not visible. Gray boxes extend vertically from the minimum (resp. maximum) probability given by $\gold_1$ to the theoretical bound given by \Cref{thm:transparency-oldpaper}. Optimal minimum, maximum probabilities per instance are shown for reference. }
         \label{fig:transparency}
     \end{figure}

 \section{Discussion} \label{sec:discussion}
Stepping back from our worst-case model, we reflect on what manipulation is likely to look like in practice. Based on anecdotes from sortition practitioners, people who misreport their features usually appear to do so with the intent to either increase their own probability ($\mint$) or misappropriate seats from other groups ($\mcomp$). People tend to employ similar strategies, misreporting belonging to identity groups that are more marginalized. This usually amounts to joining groups that are underrepresented in the pool (as in the strategy $MU$ in \Cref{sec:empirics}). Such misreports occur especially for identities that are difficult to verify or challenge. Finally, it appears that people tend to manipulate as individuals rather than as part of collusion networks (though if sortition becomes more politically mainstream, this may change).

These anecdotes motivate the pursuit of stronger bounds on $\textsf{manip}_{int}$ and $\textsf{manip}_{comp}$. In the likely event that better worst-case guarantees are impossible, one could pursue \textit{instance-wise} bounds. This direction would be of particular interest, as it might reveal new conceptual barriers to good solutions beyond Problems 1 and 2. Another approach would be to prove bounds in terms of instance parameters that interpolate between the worst case and more realistic cases. For example, since it is unlikely that all $c$ manipulators would act in concert, one could pursue bounds assuming that at while there may be $c$ manipulators, they collude in groups of size  $s \leq c$. Such $s$-dependent bounds would reveal any gaps between the power of manipulators acting in concert versus independently. As another example, since manipulators often misreport uncheckable features, one could pursue bounds parameterized by the instance \textit{and} which features are uncheckable. This approach might reveal, on an instance-wise basis, which features are ``safer'' to protect with quotas.

\begin{acks}
    We thank Chara Podimata for helpful technical conversations, Ariel Procaccia for computational resources, and the organizations who provided data (the Sortition Foundation, the Center for Blue Democracy, Healthy Democracy, MASS LBP, Of by For, and New Democracy). For funding, we thank the ETH Excellence Scholarship Opportunity Programme (CB), the Fannie and John Hertz Foundation (BF), and the NSF GRFP (BF).
\end{acks}

\bibliographystyle{ACM-Reference-Format}
\bibliography{bibliography}


\begin{thebibliography}{15}


\ifx \showCODEN    \undefined \def \showCODEN     #1{\unskip}     \fi
\ifx \showDOI      \undefined \def \showDOI       #1{#1}\fi
\ifx \showISBNx    \undefined \def \showISBNx     #1{\unskip}     \fi
\ifx \showISBNxiii \undefined \def \showISBNxiii  #1{\unskip}     \fi
\ifx \showISSN     \undefined \def \showISSN      #1{\unskip}     \fi
\ifx \showLCCN     \undefined \def \showLCCN      #1{\unskip}     \fi
\ifx \shownote     \undefined \def \shownote      #1{#1}          \fi
\ifx \showarticletitle \undefined \def \showarticletitle #1{#1}   \fi
\ifx \showURL      \undefined \def \showURL       {\relax}        \fi
\providecommand\bibfield[2]{#2}
\providecommand\bibinfo[2]{#2}
\providecommand\natexlab[1]{#1}
\providecommand\showeprint[2][]{arXiv:#2}

\bibitem[Beck and Fiala(1981)]%
        {beck1981integer}
\bibfield{author}{\bibinfo{person}{J{\'o}zsef Beck} {and} \bibinfo{person}{Tibor Fiala}.} \bibinfo{year}{1981}\natexlab{}.
\newblock \showarticletitle{“Integer-making” theorems}.
\newblock \bibinfo{journal}{\emph{Discrete Applied Mathematics}} \bibinfo{volume}{3}, \bibinfo{number}{1} (\bibinfo{year}{1981}), \bibinfo{pages}{1--8}.
\newblock


\bibitem[Benad{\`e} et~al\mbox{.}(2019)]%
        {benade2019no}
\bibfield{author}{\bibinfo{person}{Gerdus Benad{\`e}}, \bibinfo{person}{Paul G{\"o}lz}, {and} \bibinfo{person}{Ariel~D Procaccia}.} \bibinfo{year}{2019}\natexlab{}.
\newblock \showarticletitle{No stratification without representation}. In \bibinfo{booktitle}{\emph{Proceedings of the 2019 ACM Conference on Economics and Computation}}. \bibinfo{pages}{281--314}.
\newblock


\bibitem[Bradley et~al\mbox{.}(1977)]%
        {bradley1977applied}
\bibfield{author}{\bibinfo{person}{Stephen~P Bradley}, \bibinfo{person}{Arnoldo~C Hax}, {and} \bibinfo{person}{Thomas~L Magnanti}.} \bibinfo{year}{1977}\natexlab{}.
\newblock \showarticletitle{Applied mathematical programming}.
\newblock \bibinfo{journal}{\emph{(No Title)}} (\bibinfo{year}{1977}).
\newblock


\bibitem[Bürgerrat(2023)]%
        {france2}
\bibfield{author}{\bibinfo{person}{Bürgerrat}.} \bibinfo{year}{2023}\natexlab{}.
\newblock \bibinfo{title}{French citizens' assembly supports assisted dying}.
\newblock \bibinfo{howpublished}{Available at \url{https://www.buergerrat.de/en/news/french-citizens-assembly-supports-assisted-dying/} (2024/02/11)}.
\newblock


\bibitem[Ebadian et~al\mbox{.}(2022)]%
        {EKMP+22}
\bibfield{author}{\bibinfo{person}{Soroush Ebadian}, \bibinfo{person}{Gregory Kehne}, \bibinfo{person}{Evi Micha}, \bibinfo{person}{Ariel~D Procaccia}, {and} \bibinfo{person}{Nisarg Shah}.} \bibinfo{year}{2022}\natexlab{}.
\newblock \showarticletitle{Is Sortition Both Representative and Fair?}
\newblock \bibinfo{journal}{\emph{Advances in Neural Information Processing Systems}}  \bibinfo{volume}{35} (\bibinfo{year}{2022}).
\newblock


\bibitem[Ebadian and Micha(2023)]%
        {ebadianboosting}
\bibfield{author}{\bibinfo{person}{Soroush Ebadian} {and} \bibinfo{person}{Evi Micha}.} \bibinfo{year}{2023}\natexlab{}.
\newblock \bibinfo{title}{Boosting Sortition via Proportional Representation}.
\newblock \bibinfo{howpublished}{Manuscript}.
\newblock


\bibitem[Flanigan et~al\mbox{.}(2021a)]%
        {flanigan2021fair}
\bibfield{author}{\bibinfo{person}{Bailey Flanigan}, \bibinfo{person}{Paul G{\"o}lz}, \bibinfo{person}{Anupam Gupta}, \bibinfo{person}{Brett Hennig}, {and} \bibinfo{person}{Ariel~D Procaccia}.} \bibinfo{year}{2021}\natexlab{a}.
\newblock \showarticletitle{Fair algorithms for selecting citizens’ assemblies}.
\newblock \bibinfo{journal}{\emph{Nature}} \bibinfo{volume}{596}, \bibinfo{number}{7873} (\bibinfo{year}{2021}), \bibinfo{pages}{548--552}.
\newblock


\bibitem[Flanigan et~al\mbox{.}(2021b)]%
        {flanigan2021transparent}
\bibfield{author}{\bibinfo{person}{Bailey Flanigan}, \bibinfo{person}{Gregory Kehne}, {and} \bibinfo{person}{Ariel~D Procaccia}.} \bibinfo{year}{2021}\natexlab{b}.
\newblock \showarticletitle{Fair sortition made transparent}.
\newblock \bibinfo{journal}{\emph{Advances in Neural Information Processing Systems}}  \bibinfo{volume}{34} (\bibinfo{year}{2021}), \bibinfo{pages}{25720--25731}.
\newblock


\bibitem[Flanigan et~al\mbox{.}(2024)]%
        {flanigan2024manipulation}
\bibfield{author}{\bibinfo{person}{Bailey Flanigan}, \bibinfo{person}{Jennifer Liang}, \bibinfo{person}{Ariel~D Procaccia}, {and} \bibinfo{person}{Sven Wang}.} \bibinfo{year}{2024}\natexlab{}.
\newblock \showarticletitle{Manipulation-Robust Selection of Citizens’ Assemblies}. In \bibinfo{booktitle}{\emph{Proceedings of the AAAI Conference on Artificial Intelligence}}.
\newblock


\bibitem[Gandhi et~al\mbox{.}(2006)]%
        {gandhi2006dependent}
\bibfield{author}{\bibinfo{person}{Rajiv Gandhi}, \bibinfo{person}{Samir Khuller}, \bibinfo{person}{Srinivasan Parthasarathy}, {and} \bibinfo{person}{Aravind Srinivasan}.} \bibinfo{year}{2006}\natexlab{}.
\newblock \showarticletitle{Dependent rounding and its applications to approximation algorithms}.
\newblock \bibinfo{journal}{\emph{Journal of the ACM (JACM)}} \bibinfo{volume}{53}, \bibinfo{number}{3} (\bibinfo{year}{2006}), \bibinfo{pages}{324--360}.
\newblock


\bibitem[G{\k{a}}siorowska(2023)]%
        {gkasiorowska2023sortition}
\bibfield{author}{\bibinfo{person}{Adela G{\k{a}}siorowska}.} \bibinfo{year}{2023}\natexlab{}.
\newblock \showarticletitle{Sortition and its Principles: Evaluation of the Selection Processes of Citizens’ Assemblies}.
\newblock \bibinfo{journal}{\emph{Journal of Deliberative Democracy}} \bibinfo{volume}{19}, \bibinfo{number}{1} (\bibinfo{year}{2023}).
\newblock


\bibitem[Giraudet et~al\mbox{.}(2022)]%
        {giraudet2022co}
\bibfield{author}{\bibinfo{person}{Louis-Ga{\"e}tan Giraudet}, \bibinfo{person}{B{\'e}n{\'e}dicte Apouey}, \bibinfo{person}{Hazem Arab}, \bibinfo{person}{Simon Baeckelandt}, \bibinfo{person}{Philippe Begout}, \bibinfo{person}{Nicolas Berghmans}, \bibinfo{person}{Nathalie Blanc}, \bibinfo{person}{Jean-Yves Boulin}, \bibinfo{person}{Eric Buge}, \bibinfo{person}{Dimitri Courant}, {et~al\mbox{.}}} \bibinfo{year}{2022}\natexlab{}.
\newblock \showarticletitle{“Co-construction” in deliberative democracy: lessons from the French Citizens’ Convention for Climate}.
\newblock \bibinfo{journal}{\emph{Humanities and Social Sciences Communications}} \bibinfo{volume}{9}, \bibinfo{number}{1} (\bibinfo{year}{2022}), \bibinfo{pages}{1--16}.
\newblock


\bibitem[gov.scot(2021)]%
        {scotland}
\bibfield{author}{\bibinfo{person}{gov.scot}.} \bibinfo{year}{2021}\natexlab{}.
\newblock \bibinfo{title}{Scotland's Climate Assembly - recommendations for action: SG response}.
\newblock \bibinfo{howpublished}{Available at \url{https://www.gov.scot/publications/scottish-government-response-scotlands-climate-assembly-recommendations-action/} (2024/02/11)}.
\newblock


\bibitem[OIPD(2024)]%
        {ostbelgian}
\bibfield{author}{\bibinfo{person}{OIPD}.} \bibinfo{year}{2024}\natexlab{}.
\newblock \bibinfo{title}{The Ostbelgien Model: a long-term Citizens' Council combined with short-term Citizens' Assemblies}.
\newblock \bibinfo{howpublished}{Available at \url{https://oidp.net/en/practice.php?id=1237} (2024/02/11)}.
\newblock


\bibitem[Sortition-Foundation(2024)]%
        {stratificationapp}
\bibfield{author}{\bibinfo{person}{Sortition-Foundation}.} \bibinfo{year}{2024}\natexlab{}.
\newblock \bibinfo{title}{Stratification App}.
\newblock \bibinfo{howpublished}{\url{https://github.com/sortitionfoundation/stratification-app}}.
\newblock


\end{thebibliography}
     \appendix

\newpage
\section{Supplemental Materials for Section \ref{sec:model}}

\subsection{Proofs of Convexity, Conditional Equitability, and Anonymity}
Below, we show that all of our stated equality objectives in Section \ref{sec:model} are convex and satisfy conditional equitability and anonymity.

\begin{proposition}\label{prop:convexity}
    $Maximin, Minimax, Nash,$ and $Goldilocks_{\gamma}$ are all convex.
\end{proposition}
\begin{proof}
    The convexity of $Maximin$ and $Minimax$ follows immediately from their definition: $\min$ is concave, so $-\min$ is convex, and $\max$ is a convex function. Geometric mean is known to be concave, and as we define the $Nash$ objective to be the negative geometric mean, it is convex. Finally for $Goldilocks_{\gamma}$, we can rewrite the second term as $\frac{\gamma \max(1/\boldsymbol{\pi})}{n/k}$. $1/\boldsymbol{\boldsymbol{\pi}}$ is convex as all entries of $\boldsymbol{\pi}$ are nonnegative, and $\max$ is convex and increasing. Hence the composition of these two functions is convex. Therefore, $Goldilocks_{\gamma}$ is the sum of two convex functions, and is itself convex.
\end{proof}

\begin{proposition}\label{prop:ce}
    $Maximin, Minimax, Nash$ and $Goldilocks_{\gamma}$ are all conditionally equitable.
\end{proposition}
\begin{proof}
    We will simply lower bound each objective function for any $\boldsymbol{\pi} \in \Pi(\inst)$ and then show that $k/n\mathbf{1}^n$ achieves this bound. This will imply that $k/n\mathbf{1}^n \in \Pi^{\mathcal{E}}(\inst)$. Fix any solution $\pi \in \Pi(\inst)$. We know that $\max(\pi) \geq k/n$ and $\min(\pi)\leq k/n$--- otherwise $\sum_{i \in [n]} \pi_i \neq k$. Hence, for any feasible solution:
    \begin{center}
       $Maximin(\boldsymbol{\pi}) \geq -k/n \qquad Minimax(\boldsymbol{\pi}) \geq k/n \qquad Goldilocks_{\gamma}(\boldsymbol{\pi}) \geq n/k \cdot k/n + \frac{\gamma}{n/k \cdot k/n} = 1 + \gamma$
    \end{center}
   Each of these lower bounds are realized by the solution $k/n\mathbf{1}$. For $Nash$ we use the AM-GM inequality as follows: $Nash(\boldsymbol{\pi}) = -\left(\Pi_{i \in [n]} \pi_i\right)^{1/n} \geq \frac{-1}{n} \sum_{i \in [n]} \pi_i = \frac{-k}{n}$.
    Again, this lower bound is realized by the solution $k/n\mathbf{1}$.
\end{proof}

We transfer the following claim about anonymity from \citet{flanigan2021transparent} as it is relevant to the structure of our final proposition proof.

\begin{claim}[\cite{flanigan2021transparent} Claim B.6]\label{claim:anon_feas}
    For any instance $\inst$ and any realizable $\boldsymbol{\pi}$, let $\boldsymbol{\pi'}$ be the ``anonymized'' marginals obtained by setting $\pi'_i$ to the average $\boldsymbol{\pi_j}$ across all $j$ such that $w_j = w_i$. Then $\boldsymbol{\pi'}$ is realizable as well.
\end{claim}

     
\begin{proposition}[Adapted from \cite{flanigan2021transparent} Claim B.6]\label{prop:anon}
    $Maximin, Minimax,$ $ Nash, \text{ and } Goldilocks_{\gamma}$ are all anonymous.
\end{proposition}
\begin{proof}
    Fix some instance $\inst$, and $\mathcal{E} \in \{Maximin, Minimax, Nash, Linear_{\gamma}, \text{ and } Goldilocks_{\gamma}\}$. By Assumption \ref{ass:inclusion}, we have that $\inst$ is feasible (all agents existing on a valid panel implies the existence of valid panels and thus instance feasibility) -- hence $\Pi^{\mathcal{E}}(\inst)$ is nonempty. Now we will use a similar proof as to that of \Cref{claim:anon_feas}, but will pay attention to the impact of incrementally anonymizing the panel distribution on the equality objective.
    
    Assume for sake of contradiction that there is no anonymous $\boldsymbol{\pi}$ such that $\boldsymbol{\pi} \in \Pi^{\mathcal{E}}(\inst)$. Let $\boldsymbol{\pi}$ be the most anonymized optimal vector of marginals, and $\boldsymbol{d}$ be the corresponding panel distribution inducing it. Formally:\[
    \boldsymbol{\pi} = \arg \min_{\boldsymbol{\pi} \in \Pi^{\mathcal{E}}(\inst)} \max_{w \in \mathcal{W}_N} \left(\max_{i \in [n] \colon w_i = w} \pi_i - \min_{i \in [n] \colon w_i = w} \pi_i\right)
    \]
    There must be a finite number of pairs of marginals that are maximizing this gap. We argue that we can equalize these pairs one-by-one without affecting other marginals, while never increasing $\mathcal{E}$. Let $i, j \in [n]$ be such that $w_i = w_j$ and they have the maximum gap between any marginals of the same feature-vector in $\boldsymbol{\pi}$. Without loss of generality, assume $\pi_i > \pi_j$. We construct a new panel distribution $\pdist'$ as follows: $\pdist'$ is identical to $\pdist$ except that it swaps $i$ for $j$ and vice versa on all panels. Then define $\pdist'' = (\pdist + \pdist')/2$. Let $\boldsymbol{\pi''}$ be the marginals resulting from $\boldsymbol{d''}$. We have that $\pi''_i = \pi''_j$ and all other marginals remain the same. Now we consider how our equality objective, $\mathcal{E}$ might be impacted.

    Notice that $\min(\boldsymbol{\pi}) \leq \pi_j < \pi''_j=\pi''_i < \pi_i \leq \max(\boldsymbol{\pi})$. Therefore, as all other marginals remain unchanged, we have that $\min(\boldsymbol{\pi}) \leq \min(\boldsymbol{\pi''})$ and $\max(\boldsymbol{\pi}) \geq \max(\boldsymbol{\pi''})$. Therefore:\begin{align*}
        Maximin(\boldsymbol{\pi''}) &= - \min(\boldsymbol{\pi''}) \leq -\min(\boldsymbol{\pi}) \leq Maximin(\boldsymbol{\pi})\\
        Minimax(\boldsymbol{\pi''}) &= \max(\boldsymbol{\pi''}) \leq \max(\boldsymbol{\pi}) \leq Minimax(\boldsymbol{\pi})\\
        Goldilocks_{\gamma}(\boldsymbol{\pi''}) &= n/k \max(\boldsymbol{\pi''})- \gamma \cdot \frac{1}{n/k\min(\boldsymbol{\pi''})} \leq n/k \max(\boldsymbol{\pi})- \gamma \cdot \frac{1}{n/k\min(\boldsymbol{\pi})} \leq Goldilocks_{\gamma}(\boldsymbol{\pi})
    \end{align*}
    Finally, we just consider the case of $Nash$:\begin{align*}
        Nash(\boldsymbol{\pi''}) &= -\left(\Pi_{a \in [n]} \pi''_a \right)^{1/n} = -\left(\Pi_{a \in [n]} \pi_a \right)^{1/n} \cdot \left(\frac{\pi''_i\pi''_j}{\pi_i\pi_j}\right)^{1/n}
    \end{align*}
    We have that $\pi''_i\pi''_j = \left(\frac{\pi_i + \pi_j}{2}\right)^2 = \frac{\pi_i^2 + 2\pi_i\pi_j + \pi_j^2}{4}$. We know that $(\pi_i - \pi_j)^2 \geq 0$ which implies that $\pi_i^2 + \pi_j^2 \geq 2\pi_i\pi_j$. So this gives us that $\pi''_i\pi''_j \geq \frac{4\pi_i\pi_j}{4} = \pi_i\pi_j$. Returning to our analysis of $Nash$, we see that we are multiplying the negative geometric mean of $\boldsymbol{\pi}$ by a value greater than or equal to 1. So we have that $Nash(\pivec'') \leq -\left(\Pi_{a \in [n]} \pi_a \right)^{1/n} \leq Nash(\pivec)$.
\end{proof}
Then we have shown that for all equality objectives we are considering, $\mathcal{E}(\pivec'') \leq \mathcal{E}(\pivec)$. Therefore, after repeating this adjustment for all of the finitely many pairs enforcing this maximum gap, we will have arrived at a \textit{more} anonymized vector of marginals that has objective value at most $\boldsymbol{\pi}$. This is a contradiction to $\boldsymbol{\pi}$ being the most anonymized vector of marginals in $\Pi^{\mathcal{E}}(\inst)$. Therefore, we have arrived at a contradiction and can conclude that there exists an anonymous $\pivec \in \Pi^{\mathcal{E}}(\inst)$.

\subsection{Proof of Necessity of \Cref{ass:inclusion} } \label{app:assumptions-necessary}
\begin{proposition} \label{prop:assumptions-necessary}
    If $\inst$ is such that $EX(\inst) \neq \emptyset$ (satisfying Assumption \ref{ass:inclusion}), then for all algorithms $\algo$, \textsf{fairness}$(\algo,\inst) = 0$ and \textsf{manip-fairness}$(\algo,\inst,c) = 0$ for all $c \in [0,n-1]$. 
\end{proposition} \label{prop:structural-exclusion}
\begin{proof}
    Fix $\inst = ([n], \boldsymbol{w}, k, \boldsymbol{\ell, u})$ such that $EX(\inst) \neq \emptyset$. Then there exists some $i \in [n]$ such that $i \not \in K$ for any $K \in \mathcal{K}$. Therefore, for any panel distribution $\mathbf{d} \in \Delta(\mathcal{K})$, we have that $\pivec(\mathbf{d})_i = 0$, because $i$ does not appear on any panel in the support of $\mathbf{d}$. As this is true for every possible panel distribution, we have that $\textsf{fairness}(\inst,\textsc{A}) = 0$ for any algorithm $\textsc{A}$. Note that \textsf{manip-fairness}$(\algo,\inst,c) \leq \textsf{fairness}(\inst,\textsc{A})$ for any $c \in [0, n-1]$ because there is always the possible scenario that \textbf{none} of the coalition members misreport their vectors, and then $\manipinst = \inst$. Therefore, we have that  \textsf{manip-fairness}$(\algo,\inst,c) \leq 0$ and is non-negative by definition, so \textsf{manip-fairness}$(\algo,\inst,c) = 0$.
\end{proof}

\begin{proposition} \label{prop:necessity-assumptions1}
  Given a truthful instance $\inst$ that satisfies Assumption \ref{ass:inclusion}, \ref{eq:restriction1} is nearly necessary to ensure that no non-coalition members are structurally excluded in $\manipinst$. That is, $|C|$ must be at most $n_{min} - k + 1$.
\end{proposition}
\begin{proof}
 Consider the instance $\inst$ in which there is one binary feature and quotas require that there are $k-1$ 1s for this feature on the panel. Let the pool be evenly divided between those with value 1 and those with 0 ($n_0 = n_1 = n/2$). Note that the pool satisfies Assumption \ref{ass:inclusion}: there is one valid panel composition in which there are $k-1$ members with value 1 and $1$ member with value 0. Therefore, there are no structurally excluded agents in the truthful pool. All members of the pool initially have feature-vector 1, but a coalition $C$ of size $> n_{min} - k + 1$ misreports their feature-vector as 0 in $\manipinst$. As $n_{min}(\inst) = n/2$, we have that the coalition has size $> n/2-k + 1$, and there are strictly fewer than $k-1$ remaining members with vector 1. Therefore, there is no longer any valid panel remaining, so $EX(\manipinst) \cap ([n] \setminus C) \neq \emptyset$. 

\end{proof}

\newpage
\section{Supplemental Materials for Section \ref{sec:existence}}
\subsection{Analysis for \Cref{ex:trade-off}}
\begin{proposition} \label{app:trade-off}
    In the instance in \Cref{ex:trade-off}, $\textsf{p}_{01} = \textsf{p}_{10} (n/2-1)$.
\end{proposition}
\begin{proof}
    Consider any valid panel $K \in \mathcal{K}$ for the instance $\inst$. Note that $K$ must have the \textit{same number} of agents of type 01 and type 10 in order to maintain equality of the total number of 0/1 values for each feature. For a given vector $w \in \mathcal{W}$, we can write $\vecprobs_w(\pivec) = \frac{1}{n_w}\sum_{i \in [n] \colon w_i = w} \pivec_i = \frac{1}{n_w}\sum_{i \in [n] \colon w_i = w} \sum_{K \in \mathcal{K} \colon i \in K} d_K$, as this is just looking at the vector-indexed probability as the average of marginals of agents with that probability vector. Rewriting, we get that $\vecprobs_w(\pivec) = \frac{1}{n_w} \sum_{K \in \mathcal{K}} |\{i \in K : w_i = w\}| \cdot d_K$. Plugging in for Example \ref{ex:trade-off} gives that $\vecprobs_{01} = \frac{1}{1} \sum_{K \in \mathcal{K}} |\{i \in K : w_i = 01\}| \cdot d_K = (n/2 - 1) \cdot \frac{1}{n/2 - 1} \sum_{K \in \mathcal{K}}|\{i \in K : w_i = 10\}| \cdot d_K = (n/2-1)\vecprobs_{10}$, where we use that $|\{i \in K : w_i = 01\}| = |\{i \in K : w_i = 10\}|$ for all $K \in \mathcal{K}$.
\end{proof}

\subsection{Proof of \Cref{thm:LB-tradeoff}} \label{app:LB-tradeoff}
 
\begin{proof} 
    Fix $k, n_{min}, c$ in the ranges specified in the statement. Our truthful instance $\inst = ([n],\boldsymbol{w}, k,\boldsymbol{\ell,u})$ has three binary features $f_1, f_2, f_3$ and $\mathcal{W}_{\inst} = \{000, 110, 111\}$. Let the pool composition be
    \begin{center}
        $n_{000} = n_{110} = (n - n_{min})/2$, \qquad $n_{111} = n_{min}$,
    \end{center} Let the quotas require perfect balance on the first two features: $\boldsymbol{\ell}_{f_1, 0} = \boldsymbol{\ell}_{f_1, 1} = \boldsymbol{u}_{f_1, 0} = \boldsymbol{u}_{f_1, 1} = k/2$ and the same holds for $f_2$. On the third feature, let $\boldsymbol{\ell}_{f_3, 0} = \boldsymbol{u}_{f_3, 0} = k-2$ while $\boldsymbol{\ell}_{f_3, 1} = \boldsymbol{u}_{f_3, 1} = 2$. 
    
    Fix any $C$ be of size $c$ and defined to contain $c$ agents who all have the same truthful vector, $111$. To construct $\boldsymbol{\tilde{w}}$, let there be two unique coalition members $i, j \in C$ that report different vectors from the rest of the coalition: agent $i$ reports $\tilde{w}_i = 111$ (truthfully), and agent $j$ misreports $\tilde{w}_j = 010$. All remaining agents $q \in C \setminus \{i,j\}$ misreport $\tilde{w}_q = 100$. Let $\tilde{n}_w:=|\{i \in [n] | \tilde{w}_i = w\}$ represent the number of agents in the manipulated pool with each vector $w\in \mathcal{W}$. Then, our manipulated pool composition is as follows:
    \begin{center}$\tilde{n}_{000} = \tilde{n}_{110} = (n - n_{min})/2, \qquad \tilde{n}_{111} = n_{min}-c+1, \qquad \tilde{n}_{100} = c - 2, \qquad \tilde{n}_{010} = 1.$\end{center}

    Let $\tilde{K}$ denote the set of valid panels in instance $\manipinst$. Observe that $\tilde{K}$ contains exactly two valid panel types (where both can be realized so long as $k \geq 6$): 
    \begin{itemize}
        \item \textbf{Type 1} contains 2 agents with vector $111$, $k/2-1$ agents with vector $000$, and $k/2-1$ agents with vector $110$. 
        \item \textbf{Type 2} contains 2 agents with vector $111$, $1$ agent with vector $010$, $1$ agent with vector $100$, $k/2-2$ agents with vectors $000$, and $k/2-2$ agents with vector $110$.
    \end{itemize}
    Now, fix some $\mathbf{d} \in \Delta(\tilde{\mathcal{K}})$ with associated probability allocation $\pivec \in \Pi(\manipinst)$ and vector-indexed probabilities $\boldsymbol{\textsf{p}}$. First, because all valid panels contain exactly 2 agents with vector $111$, it must be that $\textsf{p}_{111} = 2/(n_{min} - c + 1)$. Now deriving the other vector-indexed probabilities, let $d_1,d_2$ be the probabilities placed by $\pdist$ on panels of types 1 and 2, respectively. By definition of the panel types,
\begin{align} 
    \vecprobs_{000} = \vecprobs_{110} = d_1 \frac{k/2-1}{(n-n_{min})/2} + d_2 \frac{k/2-2}{(n-n_{min})/2}, \qquad \vecprobs_{100} = d_2 \frac{1}{c-2}, \qquad \vecprobs_{010} = d_2. 
\end{align}
Noting that $d_1 = (1-d_2)$ and simplifying, we get that
\begin{align} 
    \vecprobs_{000} = \vecprobs_{110} = \frac{k/2 - 1 - d_2}{(n-n_{min})/2}, \quad \vecprobs_{100} = d_2 \frac{1}{c-2}, \quad \vecprobs_{010} = d_2, \quad  \vecprobs_{111} = \frac{2}{n_{min} - c + 1}.  \label{selprobs_inst2}
\end{align}
It follows that $\textsf{p}_{010} = (c-2)\textsf{p}_{110}$. To conclude the proof, note that if $\min(\tilde{\pivec}) \geq z$ and consequently 
$\textsf{p}_{010} \geq z$, then $\max(\tilde{\pivec}) \geq (c-2)\textsf{p}_{010} \geq (c-2)z$. Given also the value of \textsf{p}$_{111}$, we conclude that $\min(\tilde{\pivec}) \geq z \implies \max(\tilde{\pivec}) \geq \max\{2/(n_{min} - c + 1),(c-2)z\})$, as needed.
\end{proof}

\subsection{Notation and preliminaries for proofs of \Cref{lem:feasibility-1/n} and \Cref{thm:UB-tradeoff}}
In this proof, we will work exclusively with feature-vector indexed objects, which treat individuals with the same feature vector as interchangeable (this is without loss of generality because, by \Cref{prop:anon}, all objectives we consider are anonymous). To begin, we will define these objects, which collapse all individuals of the same feature vector.\\[-0.5em] 

\noindent \textit{\textbf{Pool and panel compositions}}: For panel $K$, we let its panel \textit{composition} $\panelcomp(K) \in [0,1]^{|\mathcal{W}|}$ describe the frequencies of each feature vector on a panel, with $w$-th entry 
\[\panelcomp_w(K) = \frac{|\{i : i\in K \land w_i = w\}|}{|K|} \ \  \text{and} \ \  \panelcomp(K):= (\panelcomp_w(K) | w \in \mathcal{W}).\] 
We say that $\panelcomp$ \textit{contains} vector $w$ iff $\panelcomp_w > 0$.\\[-0.5em] 

We define a \textit{\textbf{pool composition}} $\poolcomp(N) \in [0,1]^{|\mathcal{W}|}$ analogously, so the pool composition of $N$ is given by \[\poolcomp(N):= (\poolcomp_w(N) | w \in \mathcal{W}) \ \  \text{where} \ \ \poolcomp_w(N) = \frac{|\{i : i\in N \land w_i = w\}|}{|N|}.\] When $N$ or $K$ is clear from context, or when referring to an arbitrary pool or panel composition, we will simply use $\panelcomp$ or $\poolcomp$ respectively. \\[-0.5em]

\noindent Let the \textbf{set of \textit{valid} panel compositions} be
\[\mathfrak{K}(\mathcal{K}):=\{\panelcomp(K) | K \in \mathcal{K}\}.\]
When $\mathcal{K}$ is clear, we will shorten this to $\mathfrak{K}$.\\[-0.5em] 

\noindent Then, a \textbf{\textit{panel composition distribution}} is then any distribution over the set of valid panel compositions; that is, $\textsf{d} \in \Delta(\mathfrak{K})$.\\[-0.5em]  


\noindent \textbf{\textit{Vector-indexed total probabilities}:} Finally, for a given panel composition distribution $\textsf{d}$ we define the \textit{total probabilities} given to each vector $\textsf{t}(\textsf{d}) \in [0,k]^{|\mathcal{W}|}$ as \[\textsf{t}_w(\textsf{d}) := \sum_{\panelcomp \in \mathfrak{K}} \textsf{d}_{\panelcomp} \cdot k \cdot \panelcomp_{w}\qquad \text{and} \qquad \textsf{t}(\textsf{d}) = (\textsf{t}_w(\textsf{d}) | w \in \mathcal{W}).\] 
Notice that we can just as easily define these totals for $\vecprobs$ as $\textsf{t}_w(\vecprobs) = N_w \cdot \vecprobs_w$ \emdash abusing notation, we will allow this.

Before proceeding,
we prove the following two lemmas, which show how to reconstruct a panel distribution from a panel composition distribution while preserving the vector-indexed total probabilities and vice versa.
\begin{lemma} \label{lem:comp-to-dist}
   Fix a panel composition distribution $\textsf{d}$. We will now show how to construct a corresponding panel distribution $\pdist$ such that 
  $\pivec(\pdist)$ is anonymous with 
   \[\pi_i(\mathbf{d}) = \frac{\textsf{t}_w(\textsf{d})}{N_w} \quad \text{for all }i : w_i = w, \text{ all }w \in \mathcal{W}.\]
\end{lemma}
\begin{proof}        
   Fix $\textsf{d}$. We will construct $\pdist$ via the following algorithm. 
   
   Initialize $\pdist \leftarrow \mathbf{0}^{|\mathcal{K}|}$. 

    For all panel compositions $\panelcomp \in \mathfrak{K}$ such that $\textsf{d}_\panelcomp > 0$, do the following: 
    \begin{quote}
        Let $W_{\panelcomp}:=\{w : \panelcomp_w > 0\}$ be the set of all feature vectors contained by $\panelcomp$. Then, let $L$ be the least common multiple of $N_w | w \in W_{\panelcomp}$, i.e., the number of people in the pool with each such vector $w$. Now create $L$ panels $K_1^{(\panelcomp)} \dots K_L^{(\panelcomp)}$, where all these panels contain $k \cdot \panelcomp_w$ seats reserved for people of vector $w$, for each $w \in W_\panelcomp$. Populate the seats reserved for vector $w$ on each panel with individuals with vector $w$ round-robin style until all panels of individuals are constructed. Because $L$ is a multiple of $N_w$ for all $w$, each $i$ with vector $w$ will be placed on the same number of panels, and will be placed on a total of $L \cdot k \cdot \panelcomp_w / N_w$ panels. Also, note that because $\panelcomp$ was a valid panel composition, $K_1^{(\panelcomp)} \dots K_L^{(\panelcomp)}$ must be valid panels.

        Now, for each panel $K_j \in \{K_1^{(\panelcomp)} \dots K_L^{(\panelcomp)}\}$,
        $d_{K_j^{(\panelcomp)}} \leftarrow d_{K_j^{(\panelcomp)}} + \textsf{d}_\panelcomp/L$.
    \end{quote}
    Now, it just remains to prove that for all $i$ with $w_i = w$, we have that $\pi_i(\pdist) = t_w(\textsf{d})/N_w$, for all $w \in \mathcal{W}$. Fix such a $w$ and corresponding $i$ with $w_i$. Then, based on the algorithm above,
    \[\pi_i(\pdist) = \sum_{\panelcomp : \textsf{d}_{\panelcomp} > 0} L \cdot k \cdot \panelcomp_w / N_w \cdot \mathsf{d}_\panelcomp / L = \sum_{\panelcomp \in \mathfrak{K}} k \cdot \panelcomp_w \cdot \mathsf{d}_\panelcomp / N_w = \frac{\textsf{t}_w(\textsf{d})}{N_w}. \qedhere \]
\end{proof}

\begin{lemma}\label{lem:dist-to-comp}
Given a panel distribution $\pdist$, we will show how to construct a corresponding panel composition distribution $\textsf{d}$ such that 
   \[ \textsf{t}_w(\textsf{d}) = \sum_{i \in [n] \colon w_i = w} \pi_i(\mathbf{d}) \quad \text{for all } w \in \mathcal{W}.\]
\end{lemma}
\begin{proof}
    Fix our panel distribution $\pdist$. We will essentially just abstract it into a panel composition distribution.
    Initialize $\textsf{d} \leftarrow \mathbf{0}^{|\mathfrak{K}|}$. 

    For all panels $K \in \mathcal{K}$ such that $\pdist_K > 0$, update $\textsf{d}$ as follows: $\textsf{d}_{\panelcomp(K)} \leftarrow \textsf{d}_{\panelcomp(K)} + \pdist_K$. This is clearly a valid distribution because all entries are non-negative and sum to 1 because we simply distribute the probability mass of $\pdist$ across panel compositions. 

    Fix some $w \in \mathcal{W}$. Based on the algorithm above, we have that:\begin{align*}
        t_w(\mathsf{d}) &= \sum_{\textsf{K} \in \mathfrak{K}} \textsf{d}_{\textsf{K}} \cdot k \cdot \panelcomp_w = \sum_{\textsf{K} \in \mathfrak{K}} \sum_{K \in \mathcal{K} \colon \textsf{K}(K) = \textsf{K}} \pdist_K \cdot k \cdot \textsf{K}_w = \sum_{K \in \mathcal{K}} \pdist_K \cdot |\{i \colon i \in K \land w_i = w\}|\\
        &= \sum_{K \in \mathcal{K}} \sum_{i \colon i \in K \land w_i = w} \pdist_K = \sum_{i \in [n] \colon w_i = w} \sum_{K \in \mathcal{K} \colon i \in K} \pdist_K = \sum_{i \in [n] \colon w_i = w} \pi_i(\pdist)
    \end{align*}
\end{proof}

\subsection{Proof of \Cref{lem:feasibility-1/n}} \label{app:feasibility-1/n}
\begin{proof}
    Fix an instance $\inst$ satisfying \Cref{ass:inclusion}. We will construct $\pdist$ by constructing a panel \textit{composition distribution} $\textsf{d}$, and then transforming it into a panel distribution via \Cref{lem:comp-to-dist}. 
    By \Cref{ass:inclusion}, for each $w \in \mathcal{W}_{\inst}$ there must exist some panel composition $\panelcomp \in \mathfrak{K}$ such that $\panelcomp_w > 0$. Let $\panelcomp^{(w)}$ be the associated panel for each vector $w \in \mathcal{W}_{\inst}$ (these panel compositions need not be unique). Then, define $h(\panelcomp):= \sum_{w : \panelcomp^{(w)} = \panelcomp} n_w$ as the number of agents in the pool whose ``associated'' panel composition is $\panelcomp$. Now, define our panel composition distribution such that
    \[\textsf{d}_{\panelcomp} = \begin{cases}
        h(\panelcomp)/n \qquad &\text{for all }\panelcomp \in \bigcup_{w \in \mathcal{W}_{\inst}} \panelcomp^{(w)}\\
        0 &\text{else}.
    \end{cases}\]
    
    First, note that this is a valid distribution: its probabilities are trivially non-negative, and
    \[\sum_{\panelcomp \in \bigcup_{w \in \mathcal{W}_{\inst}} \panelcomp^{(w)}} h(\panelcomp)/n = \sum_{\panelcomp \in \bigcup_{w \in \mathcal{W}_{\inst}} \panelcomp^{(w)}} \sum_{w : \panelcomp^{(w)} = \panelcomp} n_w/n = \sum_{w \in \mathcal{W}_{\inst}} n_w/n = 1.\]
    By the above, for all feature vectors in the pool $w \in \mathcal{W}_{\inst}$, the total probability given to that vector group is bounded as 
    \[n_w/n \leq h(\panelcomp^{(w)}) \leq t_w(\textsf{d}) \leq k.\]

    Finally, we apply \Cref{lem:comp-to-dist} to transform $\textsf{d}$ into our panel distribution $\pdist$ such that $\pi_i(\pdist) = t_w(\textsf{d})/n_w$; it follows that for this $\pdist$,
    \[1/n \leq \pi_i(\pdist) \leq k/n_{min}(\inst) \in O(1/n_{min}(\inst)) \qquad \text{for all }i \in [n],\]
    as needed.
\end{proof}



\subsection{Proof of \Cref{thm:UB-tradeoff}} \label{app:impact-ub}
   \begin{proof} 
        Fix $\inst,C,\boldsymbol{\tilde{w}}$ as specified in the statement. By the assumption that $c$ respects \ref{eq:restriction1}, we know that $|C|=c \leq \max\{n_{min}(\inst) -k,0\}$. If $c = 0$ based on this restriction, then the claim is trivially true due to \Cref{lem:feasibility-1/n}. Thus, we henceforth assume that $\inst$ is such that $c > 0$. 
        
         Let $\mathcal{K}$ be the set of valid panels in $\mathcal{I}$ and $\tilde{\mathcal{K}}$ be the set of valid panels in $\manipinst$; likewise, let $\mathfrak{K}$ be the set of valid panel \textit{compositions} in $\inst$, and let $\tilde{\mathfrak{K}}$ be the set of valid panel \textit{compositions} in $\manipinst$. Finally, fix any $z \in (0,1/n]$. \\[-0.5em]

        \noindent \textbf{Approach.} Fix the panel distribution $\pdist \in \Delta(\mathcal{K})$ that implies selection probability assignment $\pivec \in \Pi(\inst)$ giving all agents selection probability in $[\Omega(1/n),O(1/n_{min}(\inst))]$, which we know to exist by \Cref{lem:feasibility-1/n}. Our approach will be to construct a panel distribution $\tilde{\pdist} \in \Delta(\mathcal{\tilde{K}})$ from our original panel distribution $\pdist$ with the desired properties. We will do this construction in \textit{panel composition space.} We begin with Claim 1, which characterizes the space of valid panel compositions in instance $\inst$ versus $\tilde{\inst}$.\\[-0.5em]

        \noindent \textbf{\textit{Claim 1:}} $\mathfrak{K} \subseteq \tilde{\mathfrak{K}}$  (the set of valid compositions only grows after $C$ misreports).\\[-0.5em]
        
        \noindent \textit{Proof of Claim 1.} Recall that $\mathcal{W}_{\inst}$ describes the unique feature vectors in the pool in $\inst$. By \ref{eq:restriction1}, we have that $c \leq n_{min}(\inst) -k \leq n_w(\inst) - k$ for all $w \in \mathcal{W}_{\inst}$. This implies that $n_w(\manipinst) \geq n_w(\inst) - c \geq k$ for each $w \in \mathcal{W}_{\inst}$.
        In other words, there are at least $k$ agents $i$ such that $\tilde{w}_i = w$ for all vectors present in the original instance, 
        meaning that
        \[\panelcomp \in \mathfrak{K} \implies \panelcomp \in \tilde{\mathfrak{K}}.\]  \textit{(End proof of Claim 1)}.\\[-0.5em]

        While the set of feasible panel compositions could not have \textit{shrunk} due to $C$ misreporting, it could certainly have grown, as members of $C$ may have reported vectors not present in $N$. We now partition $C$ into three mutually exclusive and exhaustive subsets:
        \begin{itemize}
            \item $C_1 \subseteq C$ contains all $i$ whose feature vector $\tilde{w}_i \notin \mathcal{W}_{\inst}$ \textit{and} is \textit{not} contained on any panel composition in $\tilde{\mathfrak{K}}$. 
            \item $C_2 \subseteq C$ contains all $i$ whose feature vector $\tilde{w}_i \notin \mathcal{W}_{\inst}$ (and therefore is not contained on any panel composition $\panelcomp \in \mathfrak{K}$), but \textit{is} contained on some panel composition $\panelcomp \in \tilde{\mathfrak{K}} \setminus \mathfrak{K}$
            \item $[n] \setminus C_1 \setminus C_2$, which contains only agents with vectors that were originally present in the original instance (i.e., appeared in $\boldsymbol{w}$), since $\mathcal{W}_{\manipinst} = \mathcal{W}_{\inst} \cup \mathcal{W}_{C_1} \cup W_{C_2}$ and all sets on the right-hand side are mutually exclusive.
        \end{itemize}        
        \noindent \textbf{Handling $\boldsymbol{C_1}$.} By our handling of structural exclusion in \Cref{sec:exclusion}, we know that when given instance $\manipinst$, the selection algorithm will ignore agents in $C_1$, meaning that our effective pool size in $\manipinst$ is $n - |C_1|$. Observe that it is still the case that $n_w(\manipinst) \geq k$ for all $w \in \mathcal{W}_{\inst}$ as before, because the number of agents $i$ with any $\tilde{w}_i \in \mathcal{W}_{\inst}$ is unaffected by dropping agents from the manipulated pool with $\tilde{w}_i \notin \mathcal{W}_{\inst}$.\\[-0.5em]

        \noindent \textbf{Construction of valid panels / panel compositions containing new vectors.} For each $i \in C_2$, identify a panel $K^{(i)} \in \tilde{\mathcal{K}}$ such that $i \in K^{(i)}$ (these panels needs not be unique). 
        Let $Z$ represent the maximum total number seats reserved for any single vector across all these panels (counting duplicates with their multiplicity):
        \[Z:=\max_{w \in W_{\manipinst}} \ \sum_{i \in C_2} k \cdot \panelcomp_w(K^{(i)}).\] 
        Note that $Z \geq 1$, because each agent $i \in C_2$ is given at least one seat on one panel $K^{(i)}$. Also note that $Z \leq k|C_2|$, as we sum over $|C_2|$ panels that can allot at most $k$ seats to any vector.

        Now, let $g \colon \bigcup_{i \in C_2} \{\panelcomp\left(K^{(i)}\right)\} \rightarrow \mathbb{N}$ map any given panel composition we have identified above to the number of agents in $C_2$ whose chosen panel has that composition. Formally, it is defined as $g(\panelcomp) = |\{i \in C_2 \colon \panelcomp(K^{(i)}) = \panelcomp\}|$. Note that grouping common panel compositions across the $K^{(i)}$ forms a partition of $|C_2|$ (in which agents with common panel composition are grouped together). Because $g(\panelcomp)$ represents the number of agents in the coalition with associated panel composition $\panelcomp$, it follows that
        \[\sum_{\panelcomp \in \bigcup_{i \in C_2} \{\panelcomp(K^{(i)})\}} g(\panelcomp) = |C_2| \qquad \text{and} \qquad g(\panelcomp) \leq |C_2| \ \ \text{for all } \panelcomp.\]


  We have that $Z \geq g(\panelcomp)$ for all $\panelcomp$ because if there are $g(\panelcomp)$ many copies of the same panel composition in the representative panels then there are at least $g(\panelcomp)$ seats reserved for any given vector on this panel composition.\\[-0.5em]
 
 \noindent \textbf{Construction of $\tilde{\textsf{d}}$.} Let $\textsf{d}$ be the panel composition distribution corresponding to $\pdist$, transformed via \Cref{lem:dist-to-comp}. We will now construct a new panel composition distribution $\tilde{\textsf{d}}$ from $\textsf{d}$. In this construction, we will add our newly feasible panel compositions constructed above to the support and redistribute some probability mass over them. 
Define panel composition distribution $\tilde{\textsf{d}}$ as follows:
        \[\tilde{\textsf{d}}_{\panelcomp}:=
        \begin{cases}
            \textsf{d}_{\panelcomp} \cdot \left(1-|C_2|z\right)  & \text{if } \panelcomp \in \mathfrak{K}\\
            z\,g(\panelcomp) &\text{if } \panelcomp \in \bigcup_{i \in C_2} \{\panelcomp(K^{(i)})\}\\
           0 &\text{else}
        \end{cases} \qquad \text{for all }\panelcomp \in \tilde{\mathfrak{K}}.\]

        \noindent \textbf{\textit{Claim 2:}} $\tilde{\textsf{d}}$ is a well-defined distribution.\\[-0.5em]

        \noindent \textit{Proof of Claim 2.} First, note that for every $\panelcomp \in \tilde{\mathfrak{K}}$, $\tilde{\textsf{d}}_\panelcomp$ is set to a single value. This is because the cases are by definition mutually exclusive: if $\panelcomp \in \mathfrak{K}$, it cannot be among the panels compositions $\panelcomp(K^{(i)}) \colon i \in C_2$ by definition. Now, we argue that all entries are in $[0,1]$. First, using that $g(\panelcomp) \leq |C_2| \leq n_{min}$ and $z \leq 1/n$, 
        \begin{align*}
            0 \leq z\, g(\panelcomp) \leq z|C_2| \leq z n_{min} \leq n_{min}/n \leq 1.
        \end{align*}


        Again using that $|C_2| \leq c \leq n_{min}$ and $z \leq 1/n$, 
        \begin{align*}
            1 \geq \textsf{d}_{\panelcomp} \geq \textsf{d}_{\panelcomp} (1- z|C_2|) \geq \textsf{d}_{\panelcomp} (1-zc) \geq \textsf{d}_{\panelcomp} (1 - n_{min}/n) \geq 0.
        \end{align*}
             Finally, all probabilities in this distribution sum to 1: 
             \begin{align*}
                \sum_{\panelcomp \in \tilde{\mathfrak{K}}} \tilde{\textsf{d}}_{\panelcomp} &= \sum_{\panelcomp \in \bigcup_{i \in C_2} \panelcomp(K^{(i)})} z\, g(\panelcomp)  + \sum_{\panelcomp \in \tilde{\mathfrak{K}}\setminus \bigcup_{i \in C_2} \panelcomp(K^{(i)})} \textsf{d}_{\panelcomp} \cdot \left(1-z|C_2|\right)\\
                &= z \sum_{\panelcomp \in \bigcup_{i \in C_2} \panelcomp(K^{(i)})} g(\panelcomp) +  \left(1-z|C_2|\right) \sum_{\panelcomp \in \tilde{\mathfrak{K}}\setminus \bigcup_{i \in C_2}\panelcomp(K^{(i)})} \textsf{d}_{\panelcomp}\\
                &= z|C_2| + \left(1-z|C_2|\right) \cdot 1\\
                &= 1
            \end{align*}
        \textit{(End proof of Claim 2).}\\[-0.5em]

        \noindent \textbf{Bounding $\textsf{t}_w(\tilde{\textsf{d}})$ for all $w \in \mathcal{W}_{C_2}$.} We begin by looking at $w \in \mathcal{W}_{C_2}$. We first make some observations about the $g(\panelcomp)$, and in particular their relationship to sets of agents: 
        \begin{enumerate}
            \item $n_w(\manipinst) \leq \sum_{\panelcomp \in \bigcup_{i \in C_2} \{\panelcomp(K^{(i)})\} \land \panelcomp_w > 0} g(\panelcomp)$ for all $w \in \mathcal{W}_{C_2}$. To see this, note that we can partition agents in $C_2$ according to the panel composition of $K^{(i)}$, the panel we identified to include them. $g(\panelcomp)$ is then exactly the number of agents who chose $\panelcomp$. Adding up over all panel compositions including vector $w$ will necessarily add 1 per person with vector $w$, since for each such person there is at least one panel composition containing them whose composition group they belong to. 
            \item By definition of $g(\panelcomp)$,  $g(\panelcomp) z = \sum_{i \in C_2: \panelcomp(K^{(i)}) = \panelcomp} z$.
        \end{enumerate}
        Now, we lower bound $t_w(\tilde{\textsf{d}})$ for all $w \in \mathcal{W}_{C_2}$. The first inequality comes from applying observation (1) above, noting that when $\panelcomp_w > 0$, $k \panelcomp_w \geq 1$. The final step uses that $Z \leq k|C_2| \leq kc$.
        \begin{align*}
     t_w(\tilde{\textsf{d}}) &=  \sum_{\panelcomp \in \bigcup_{i \in C_2}\{\panelcomp(K^{(i)})\}} z\,g(\panelcomp) \cdot k \cdot \panelcomp_w \geq z\, n_w(\manipinst)  \qquad \text{for all }w \in \mathcal{W}_{C_2}.
        \end{align*}

        Now, to upper bound $\textsf{t}_w(\tilde{\textsf{d}})$ for all $w \in \mathcal{W}_{C_2}$, we apply observation (2) above.
        \begin{align*}
            t_w(\tilde{\textsf{d}}) =  \sum_{\panelcomp \in \bigcup_{i \in C_2}\{\panelcomp(K^{(i)})\}} z\,g(\panelcomp) \cdot k \cdot \panelcomp_w &= z\sum_{\panelcomp \in \bigcup_{i \in C_2}\{\panelcomp(K^{(i)})\}}  \sum_{i \in C_2: \panelcomp(K^{(i)}) = \panelcomp} k \cdot \panelcomp_w
            \intertext{We can condense the sums: the first is over all panel compositions, and the second is over all $i \in C_2$ whose $K^{(i)}$ fits that composition; therefore, this is just}
            &= z\sum_{i \in C_2} k \cdot \panelcomp_w(K^{(i)})
            \intertext{This sum is by definition at most $Z$:}
            &\leq z Z
            \intertext{And finally using that $Z \leq kc$,}
            &\leq zkc.
        \end{align*}
    
    \noindent We conclude that for all $w \in \mathcal{W}_{C_2}$,
    \begin{equation}\label{eq:totals-boundsC2}
    \textsf{t}_w(\tilde{\textsf{d}}) \in \left[zn_w(\manipinst), \ zkc \right].
    \end{equation}

\noindent \textbf{Bounding $\textsf{t}_w(\tilde{\textsf{d}})$ for all $w \in \mathcal{W}_{\inst}$.} Now, we deduce the following bounds on $\textsf{t}_w(\tilde{\textsf{d}})$ for all $w \in \mathcal{W}_{\inst}$. The lower bound corresponds to the case where $w$ occurs on no panel compositions in $\bigcup_{\tilde{w} \in \mathcal{W}_{C_2}} \{\panelcomp^{(\tilde{w})}\}$.
Beforehand, recall our assumption that $|\mathcal{W}_{\inst}| \geq 2$, which implies directly that $n_{min}(\inst) \leq n/2$. Then, using that $z \in (0,1/n]$ and $c \leq n_{min} \leq n/2$, 
\begin{equation} \label{eq:totalprobeq}
      \textsf{t}_w(\tilde{\textsf{d}}) \geq (1-|C_2|z) \, \textsf{t}_w(\textsf{d}) \geq (1-n_{min}/n)\, \textsf{t}_w(\textsf{d}) \geq \textsf{t}_w(\textsf{d})/2 \geq n_w(\inst)/(2n) \geq  n_w(\manipinst)/(4n) .
\end{equation}
Up to the final two steps, we deduce that the total probability given to group $w$ does not decrease in order from $\textsf{d}$ to  $\tilde{\textsf{d}}$. In the penultimate step, we apply a lower bound on $\textsf{t}_w(\textsf{d})$ that comes from the fact that by \Cref{lem:feasibility-1/n}, all \textit{agents} received probability at least $1/n$ from $\textsf{d}$, which means that $\textsf{t}_w(\textsf{d}) \geq n_w(\inst)/n$. In the last step, we use that $n_w(\manipinst) \leq n_w(\inst) + c \leq 2n_w(\inst)$.
     %
The upper bound corresponds to the case where this vector occurs on all panel compositions to the maximum possible extent as captured in $Z$. Using that $t_w(\textsf{d})\leq k$, $g(\textsf{K}) \leq Z$, and $Z \leq ck$, 
\begin{align*}
 \textsf{t}_w(\tilde{\textsf{d}})\leq t_w(\textsf{d}) + zg(\textsf{K})n_w(\manipinst) \leq  \textsf{t}_w(\textsf{d}) + z Z  \, n_w(\manipinst) \leq k + z ck \, n_w(\manipinst).
    \end{align*}

We conclude that
\begin{equation} \label{eq:boundsN}
    \textsf{t}_w(\tilde{\textsf{d}}) \in \left[n_w(\manipinst)/(4n),k+kzcn_w(\manipinst)\right] \qquad \text{for all } w \in \mathcal{W}_{\inst}.
\end{equation}

\noindent \textbf{Transforming $\tilde{\textsf{d}}$ into our final panel distribution $\tilde{\pdist}$.} Now, we apply \Cref{lem:comp-to-dist} to transform our panel composition distribution $\tilde{\textsf{d}}$ into a corresponding panel distribution $\tilde{\pdist}$. This lemma shows that for the $\tilde{\pivec}$ implied by $\tilde{\pdist}$, it holds that
        \[\tilde{\pi}_i = \frac{\textsf{t}_w(\tilde{\textsf{d}})}{n_w(\manipinst)} \qquad \text{for all }i \in [n] : \tilde{w}_i = w, \text{all }w \in \mathcal{W}.\]

    \noindent First, combining \Cref{eq:totals-boundsC2} and that $n_w(\manipinst) \geq 1$ for all $w \in C_2$, we get that
    \begin{align*}
        \tilde{\pi}_i &\in \left[z, kzc\right] \subseteq [\Omega(z),O(cz)] \qquad &&\text{for all $i : \tilde{w}_i \in \mathcal{W}_{C_2}$.}
        \intertext{Likewise, combining \Cref{eq:boundsN} and that $n_w(\manipinst) \geq n_w(\inst) - c \geq n_{min}-c$ for all $w \in \mathcal{W}_{\inst}$,}
        \tilde{\pi}_i &\in \left[\frac{1}{4n}, \ \frac{k}{n_{min}-c} + kzc \right] \subseteq \left[\Omega(1/n), \max\{cz,1/(n_{min}-c)\}\right]\qquad &&\text{for all }i : \tilde{w}_i \in \mathcal{W}_{\inst}.
    \end{align*}
    Taking the union of both these ranges to bound the probabilities of all agents, we conclude the result:
    \[\tilde{\pivec} \in \left[\Omega\left(z\right),O\left(\max\{cz,1/(n_{min}-c)\}\right)\right]^n. \qedhere\]
    \end{proof}

\newpage

\section{Supplemental Materials for Section \ref{sec:gold-analysis}}

\subsection{Comments on \Cref{thm:manip-gold} upper bounds on \textsf{manip}$_{ext}$ and \textsf{manip}$_{comp}$} \label{app:tightness}

First, our upper bound on \textsf{manip}$_{ext}$ appears to be tight. To see this, suppose our truthful instance is the instance in \Cref{ex:small-group}. Let there be some coalition $C$ of size $c \in \kappa n_{min}$ for some constant $\kappa > 0$ (which is entirely possible under \ref{eq:restriction1}). Suppose every coalition member has truthful vector $0$ but misreports vector $1$, thereby joining the smallest group. They have then multiplied the group size by $(1+\kappa)$, thereby decreasing their members' probabilities from $1/n_{min}$ to $1/(n_{min} \cdot (1+\kappa))$ --- an additive decrease of the order $1/n_{min}$.

Moving onto \textsf{manip}$_{comp}$, one may notice something concerning about our upper bound: if $c \in \Omega(n^{2/3})$, this bound is \textit{constant} and could be as large as $k$, even regardless of $n_{min}$. Conceivably, then, a coalition could misappropriate the entire panel. We now unpack when this may or may not be of practical concern. 

First, it in fact may be that in practice, the number of agents willing to misreport grows as $\Omega(n^{2/3})$. It would actually make sense if the number of manipulators grew \textit{linearly} in $n$, as this should be true in expectation: if there is a constant fraction of the population willing to manipulate, they should compose that same constant fraction of the pool in expectation. While there may be many manipulators, it seems unlikely that such a large number of agents would \textit{collude} --- a key feature of our lower bounds when Problem 2 is the main issue. An approach to pursuing bounds that account for this overly pessimistic aspect of our model is discussed in \Cref{sec:discussion}.

Importantly, however, when $n_{min}$ is small and thus Problem 1 is the main issue, our upper bound's dependence on $1/n_{min}$ may actually be somewhat tight, even when $c$ is small and agents do not collude. To see this, consider again the instance in Example 1; if $c$ agents misreport as discussed above (which they could easily do independently, as this is just an example of the manipulation heuristic \textit{MU} from \Cref{sec:empirics}), all those agents will then receive selection probability $1/(n_{min}+c)$, and garner $c/(n_{min} + c)$ total expected seats originally allocated for agents with the vector 1. This negative finding is consistent with a theme of the paper: when $n_{min}$ is small, there is little any algorithm can do about manipulation in the worst case.


\subsection{Proof of \Cref{thm:lb}} \label{app:lb}
\begin{proof}
    Fix $k, n_{min}, c$ in the ranges specified in the statement. Consider the same truthful instance $\inst$ as in the proof of \Cref{thm:LB-tradeoff} (\Cref{app:LB-tradeoff}). We will construct the coalition almost identically --- indeed, it will result in \textit{almost} an identical manipulated instance $\inst^{\to_C \boldsymbol{\tilde{w}}}$ except for a net shift of two agents --- but agents will move in slightly different ways in order to maximize the manipulators' gains in probability.

    Let $C$ be defined to contain $c-2$ agents who have the truthful vector $111$, an agent $i \in C$ with vector $000$, and an agent $i' \in C$ with vector $110$. To construct $\boldsymbol{\tilde{w}}$, let $\tilde{w}_i = 111$, $\tilde{w}_{i'} = 010$. Now, for two agents $j,j' \in C$ such that $w_{j} = w_{j'} = 111$, let $\tilde{w}_j = 000$ and $\tilde{w}_{j'} = 110$. Let all remaining agents $q \in C \setminus \{i,i',j,j'\}$ misreport $\tilde{w}_q = 100$. The manipulated pool is almost exactly as it was in the proof of \Cref{thm:LB-tradeoff}, except that there are two more agents with 111 and two fewer agents with 100, where 
    \begin{center}$\tilde{n}_{000} = \tilde{n}_{110} = (n - n_{min})/2, \qquad \tilde{n}_{111} = n_{min}-c+3, \qquad \tilde{n}_{100} = c - 4, \qquad \tilde{n}_{010} = 1.$\end{center} 
    
     Then, starting from where that proof left off but accounting for our slight shifts in population, fix any $\tilde{\pivec} \in \Pi(\manipinst)$; assuming $\tilde{\pivec}$ is anonymous, by \Cref{selprobs_inst2}, we know that it must imply vector-indexed selection probabilities as follows (defining $d_1,d_2$ the same way as in the proof of \Cref{app:LB-tradeoff}):
    \begin{align*} 
    \vecprobs_{000} = \vecprobs_{110} = \frac{k/2 - 1 - d_2}{(n-n_{min})/2}, \quad \vecprobs_{100} = d_2 \frac{1}{c-4}, \quad \vecprobs_{010} = d_2, \quad  \vecprobs_{111} = \frac{2}{n_{min} - c + 3}.  
\end{align*}
To lower-bound how much is gained by the manipulation constructed here, we must upper bound the probability received by manipulators received in the original instance $\inst$. We will study two manipulators in particular, $i$ and $i'$, because they have misreported the most lucrative vectors. By the above probabilities, they receive the following probabilities post-manipulation:
\[\tilde{\pi}_i = \textsf{p}_{111} = 2/(n_{min}-c+3) \qquad \text{and} \qquad \tilde{\pi}_{i'} = \textsf{p}_{010} =  d_2.\]
Next, we upper bound $i,i'$'s probabilities in $\inst$. Fix any $\pivec \in \Pi(\inst)$. Observe that both $i,i'$  belong to the vector groups $000$ or $110$, both of which are of size $(n - n_{min})/2$. These groups can be given at most $k/2-1$ seats on any valid panel, meaning that
\[\pi_i \leq 2(k/2-1)/(n - n_{min}) \qquad \text{and} \qquad \pi_{i'} \leq 2(k/2-1)/(n - n_{min}).\]

Now, to conclude the proof, fix any $z \in (0,1/n]$, and fix any algorithm $\algo$ such that \textsf{manip-fairness}$(\inst,\textsc{A},c) \in \Omega(z)$. Note that we have argued so far about generic $\pivec \in \Pi(\inst)$ and $\tilde{\pivec} \in \Pi(\inst)$ so when we now set $\pivec = \pivec^{\algo}(\inst)$ and $\tilde{\pivec} = \pivec^{\algo}(\manipinst)$, all claims above apply to these particular probability allocations. 

We deduce that $\textsf{manip-fairness}(\inst,\textsc{A},c) \in \Omega(z) \implies \min(\tilde{\pivec}) \in \Omega(z) \implies \textsf{p}_{100} = d_2/(c-4)\in \Omega(z) \iff d_2 \in \Omega(cz).$
If $d_2 \in \Omega(cz)$, then $i$ has gained probability 
\begin{equation} \label{eq:probs-i}
    \tilde{\pi}_i- \pi_i = 2/(n_{min}-c+3) - (k-2)/(n - n_{min}) \in \Omega\left(1/(n_{min}-c)\right).
\end{equation}
where the difference is non-negative given that $k \geq 2$ and $n_{min} \leq n/k$ as is by construction.
Similarly $i'$ has gained 
\begin{equation} \label{eq:probs-i'}
    \tilde{\pi}_{i'} - \pi_{i'} = \max\{0,d_2 - (k-2)/(n - n_{min})\} \in \Omega(cz).
\end{equation}
We simplify this term to $\Omega(cz)$ because if $d_2 \in o(k/n)$ or $d_2 - (k-2)/(n - n_{min})\leq 0$, then the other term in the ultimate bound will dominate it, per the reasoning above. 

Finally, by \Cref{eq:probs-i,eq:probs-i'}, we conclude the claim:
\[\mint(\inst,\textsc{A},c) \geq \max_{j \in [n]} (\tilde{\pi}_j - \pi_j) \geq \max\{\tilde{\pi}_{i'} - \pi_{i'}, \ \tilde{\pi}_i- \pi_i\} \in \Omega(\max\{cz,1/(n_{min}-c)\}. \qedhere\]
\end{proof}

\subsection{Proof of \Cref{prop:general-opt-gold}} \label{app:general-opt-gold}
\begin{proof}
    Let $\gamma^* =  z  \cdot \max\{1/(n_{min}-c),cz\} \cdot (n/k)^2$ as in the statement. 
    Just as in the proof of \Cref{thm:manip-gold}, this claim follows directly from proving the following claim: \textit{Fix any $\inst$, any coalition $C \subseteq [n]$ with $|C| = c \leq n_{min}(\inst) - k$, any \ $\boldsymbol{\tilde{w}} \in \boldsymbol{\mathcal{W}}_{\inst,C}$, and any $z \in (0,1/n]$. Then,}
    \[\pivec^{\textsc{Goldilocks}_{\gamma^*}(\inst^{\to_C \boldsymbol{\tilde{w}}})} \in [\Omega(z),O(\max\{cz,1/(n_{min}-c\})]\]
    We will in turn prove this by proving a much more general fact about \textit{Goldilocks}$_\gamma$: it can recover 
    \textit{any achievable simultaneous setting of the maximum and minimum selection probabilities}. Formally, fixing any $\pivec \in \Pi(\inst)$, we will show that 
    \begin{equation}\label{eq:claim}
    \pivec^{\textsc{Goldilocks}_{\gamma^*}}(\inst) \in \left[\min(\pivec)/2, 2\max(\pivec) \right]^n. 
    \end{equation} The original claim then follows directly from \Cref{thm:UB-tradeoff}, which positively bounds the achievable maximum and minimum probability in any instance $\inst^{\to_C \tilde{\boldsymbol{w}}}$ subject to the conditions in the statement. 
    
    Now to show \Cref{eq:claim}, by definition of \textit{Goldilocks}$_\gamma$ and $\gamma^*$ and the optimality of $\pivec^*$,  
    \begin{center}
        $ Goldilocks_{\gamma^*}(\pivec^*) \leq Golidlocks_{\gamma^*}(\pivec) \leq \frac{\max(\pivec)}{k/n} + \gamma^* \,\frac{k/n}{\min(\pivec)} = 2\,\frac{\max(\pivec)}{k/n}.$
    \end{center}
We apply this bound to each term of $Goldilocks_{\gamma^*}(\pivec^*)$ separately to conclude the result.
\begin{align*}
    \max(\pivec^*)/(k/n) \leq 2\max(\pivec)/(k/n) &\iff \max(\pivec^*) \leq 2\max(\pivec) \\
    \min(\pivec)\cdot \max(\pivec) \cdot (n/k)^2 \cdot (k/n)/\min(\pivec^*) \leq 2\max(\pivec)/(k/n) &\iff \min(\pivec^*) \geq \min(\pivec)/2. \qedhere
\end{align*}
\end{proof}

\subsection{Proof of \Cref{thm:lb-maximin-nash}} \label{app:lb-maximin-nash}
\begin{proof}
    Fixing $k,n_{min}$ and $c$ as in the statement, we construct our truthful instance $\inst$, $C$, $\boldsymbol{\tilde{w}}$, and manipulated instance $\manipinst$, exactly as in the proof of \Cref{thm:lb} (\Cref{app:lb}). 
%
    %
    Then, starting from where that proof left off, fix any $\tilde{\pivec} \in \Pi(\manipinst)$; assuming $\tilde{\pivec}$ is anonymous, by \Cref{selprobs_inst2}, we know that it must imply vector-indexed selection probabilities as follows (defining $d_1,d_2$ the same way as in the proof of \Cref{thm:lb}):
    \begin{align*} 
    \vecprobs_{000} = \vecprobs_{110} = \frac{k/2 - 1 - d_2}{(n-n_{min})/2}, \quad \vecprobs_{100} = d_2 \frac{1}{c-4}, \quad \vecprobs_{010} = d_2, \quad  \vecprobs_{111} = \frac{2}{n_{min} - c + 3}.  
\end{align*}
Note that because this applies to all $\tilde{\pivec} \in \Pi(\manipinst)$, these constraints must apply to $\pivec^{\textsc{Leximin}(\manipinst)}$ and $\pivec^{\textsc{Nash}(\manipinst)}$. 
The question is then just how each algorithm will tune $d_1,d_2$. We derive this now, showing that all these algorithms will drive $d_2$ quickly upwards as $c$ grows in order to increase the minimum probability.\\[-0.5em]

\noindent \textsc{Leximin}.
\textsc{Leximin} maximizes the minimum selection probability (and then the next lowest, and so forth) subject to the constraints in \Cref{selprobs_inst2} as well as the constraint that $d_2 \in [0,1]$. In its first round, \textsc{Leximin} will optimize the following objective, corresponding to maximizing the minimum:
\[\min\left\{\vecprobs_{000},\vecprobs_{110},\vecprobs_{100},\vecprobs_{010},\vecprobs_{111}\right\} = \min\left\{\frac{k/2-1-d_2}{(n-n_{min})/2}, \ d_2 \frac{1}{c-4}, \ d_2, 
\ \frac{2}{n_{min} - c + 3}\right\}.\]
We can effectively ignore $\textsf{p}_{111}$, because it does not depend on $d_1,d_2$; thus, if in the first round this probability is the minimum probability (i.e., all other probabilities can be made higher), \text{Leximin} will just advance to the next round and maximize the next lowest minimum probability. Therefore, we can without loss of generality rewrite the objective dropping this term as
\[\min\left\{\vecprobs_{000},\vecprobs_{110},\vecprobs_{100},\vecprobs_{010}\right\} = \min\left\{\frac{k/2-1-d_2}{(n-n_{min})/2}, \ d_2 \frac{1}{c-4}, \ d_2\right\}.\]
$\vecprobs_{010} \geq \vecprobs_{100}$ for all $d_2 \in [0,1]$, so $\vecprobs_{010}$ cannot be the minimum probability. Thus, the minimum term in this objective must be between the first two. By observation, it holds that at the first two terms of this minimum are equal when $d_2 = \frac{(k-2)(c-4)}{n-n_{min}+2(c-4)}$, in which case $\textsf{p}_{000} = \textsf{p}_{110} =\textsf{p}_{100} = \frac{k-2}{n-n_{min}+2(c-4)}.$ Observe that either increasing or decreasing $d_2$ from this value will make the objective value worse: $\textsf{p}_{000} = \textsf{p}_{110}$ are decreasing in $d_2$, so if we \textit{increase} $d_2$ from here, $\textsf{p}_{000}$ will decrease and the minimum will decrease, making the objective value worse; $\textsf{p}_{100}$ is increasing in $d_2$, so if we \textit{decrease} $d_2$ from here, $\textsf{p}_{100}$ will decrease, again making the objective value worse. We conclude that the maximin-optimal solution sets $d_2 = \min\left\{\frac{(k-2)(c-4)}{n-n_{min}+2(c-4)},1\right\}$, meaning that $\pivec^{\textsc{Leximin}(\manipinst)}$ implies vector-indexed selection probabilities such that
\[\vecprobs_{010} = \min\left\{\frac{(k-2)(c-4)}{n-n_{min}+2(c-4)},1\right\}.\]

To lower-bound how much is gained by the manipulation constructed here, we must upper bound the probability received by manipulators received in the original instance $\inst$. We will be interested in two manipulators in particular, $i$ and $i'$, because they have misreported the most lucrative vectors, thus receiving lower-bounded probabilities post-manipulation:
\[\pi^{\leximin(\inst^{\to_C \boldsymbol{\tilde{w}}})}_i = \textsf{p}_{111} = \frac{2}{n_{min}-c+3} \ \ \text{and} \ \ \pi^{\leximin(\inst^{\to_C \boldsymbol{\tilde{w}}})}_{i'} = \textsf{p}_{010} =  \min\left\{\frac{(k-2)(c-4)}{n-n_{min}+2(c-4)},1\right\}.\]
To upper bound these agents' probabilities in $\inst$, we note that $w_i = 000$ and and $w_{i'} = 110$, meaning that both belong to groups of size $(n-n_{min})/2$. These groups can be given at most $k/2-1$ seats on any valid panel, meaning that
\[\pi^{\leximin}(\inst)_i \leq 2(k/2-1)/(n-n_{min}), \qquad \text{and} \qquad \pi^{\leximin}(\inst)_{i'} \leq 2(k/2-1)/(n-n_{min}).\] 
Therefore, $i'$ has gained 
\[\pi^{\leximin(\inst^{\to_C \boldsymbol{\tilde{w}}})}_{i'} - \pi^{\leximin}(\inst)_{i'} = \min\left\{\frac{(k-2)(c-4)}{n-n_{min}+2(c-4)},1\right\} -  \frac{k-2}{n-n_{min}} \in \Omega(\min\{c/n,1\}).\]
By a few rearrangements of inequalities, the difference above is nonnegative and of the order specified for $n_{min} \leq n/2, k \geq 6$, and $c \geq 6$, as is true by construction. Note that $c \in O(n)$ trivially, which means that $\Omega(\min\{c/n,1\})$ can be rewritten as simply $\Omega(c/n).$

Similarly, $i$ has gained probability 
\[\pi^{\leximin(\inst^{\to_C \boldsymbol{\tilde{w}}})}_i - \pi^{\leximin}(\inst)_i \geq  \frac{2}{n_{min}-c+3} - \frac{k-2}{n-n_{min}} \in \Omega\left(\frac{1}{n_{min}-c}\right).\]
We omit the technicality that this difference could be negative; in the corner case that it is, this term will simply be dominated by the other term in the ultimate lower bound, which is non-negative by construction as shown above.

This concludes the proof, implying that 
    $\mint(\inst,\textsc{Leximin},c) \in \Omega(\max\left\{1/(n_{min}-c),c/n\right\}.$



\vspace{1em}
\noindent \textsc{Nash}. We analyze this algorithm in much the same way. We will equivalently optimize the logarithm of $\nash$, defined as $\sum_{i \in [n]}\log(\pi_i).$ The same constraints hold, but we are instead optimizing the \textit{product} of probabilities. Thus, \textsc{Nash} optimizes the following objective, noting the symmetry between $\vecprobs_{000}$ and $\vecprobs_{110}$:
\[ (n-n_{min})\log(\vecprobs_{000})+ (c-4)\log(\vecprobs_{100}) + \log(\vecprobs_{010}) + (n_{min}-c+3)\log(\vecprobs_{111})\]
which, by \Cref{selprobs_inst2} is equal to 
\begin{equation} \label{nash-obj}
    (n-n_{min})\log\left(\frac{k/2 - 1 - d_2}{(n-n_{min})/2}\right) + (c-4) \log\left(d_2 \frac{1}{c-4}\right) + \log (d_2) + \log\left(\frac{2}{n_{min}-c+3}\right).
\end{equation}
Differentiating \Cref{nash-obj} with respect to $d_2$, we get 
\begin{align*}
    \frac{\partial \, \text{\Cref{nash-obj}}}{\partial \, d_2} = \frac{c-3}{d_2} + \frac{2 (n - n_{min})}{2 + 2 d_2 - k} = -\frac{2d_2(n-n_{min}+c-3) - (c-3)(k-2)}{d_2(k-2-2d_2)}
\end{align*}
This derivative is 0 at $d_2 = \frac{(c-3)(k-2)}{2(n-n_{min}+c-3)}$. By the concavity of the original objective function, this value of $d_2$ must be its unique global optimizer. Now, letting $\textsf{p}$ be the vector-indexed probabilities induced by $\pivec^{\textsc{Nash}(\manipinst)}$, 
\[\vecprobs_{010} = \min\left\{\frac{(c-3)(k-2)}{2(n-n_{min}+c-3)},1\right\} \in \Omega(\min\{c/n,1\}).\]
By the same reasoning as used to conclude our analysis of \textsc{Leximin}, we conclude that
\[\mint(\inst,\textsc{Nash},c) \in \Omega(\max\left\{1/(n_{min}-c),c/n\right\}. \qedhere\]
\end{proof}

\subsection{Proof of \Cref{thm:lb-minimax}} \label{app:minimax}
\begin{proof}
We defer this proof to its strict generalization in \Cref{app:linear}, where the relevant case of that proof is $\gamma = 0$.
\end{proof}

\subsection{Lower bound for $\boldsymbol{\max(\pivec) - \gamma \cdot \min(\pivec)}$} \label{app:linear}
We henceforth call this objective \textsc{Linear}$_\gamma \coloneqq \max(\pivec) - \gamma \cdot \min(\pivec)$. Because this proof is so complex, we analyze the specific (and more interesting) sub-case where Problem 2 is the main issue, i.e., $n_{min}(\inst), n_{min}(\inst) - c \in \Theta(n)$. These lower bounds will therefore not be tight in their dependency on $n_{min}$. Because Problem 2 is where algorithms are differentiated, this proof still allows for useful comparison of algorithms. For comparability, we will ensure the coalitions we use respect \ref{eq:restriction1}.

This result shows a theoretical separation between \textsc{Linear}$_\gamma$ and \textsc{Goldilocks}$_\gamma$. To see this, note that while \textsc{Goldilocks}$_\gamma$ can strike any trade-off for appropriately chosen $\gamma$ (\Cref{prop:general-opt-gold}), there are certain, very important trade-offs that \textsc{Linear}$_\gamma$ cannot strike. For example, per \Cref{thm:linear}, there is \textit{no setting} of $\gamma$ that can strike the trade-off with $z = 1/(n\sqrt{c})$, as is achieved by \textsc{Goldilocks}$_1$ (a desirable trade-off, as it places the multiplicative probability gap due to Problem 2 symmetrically over $k/n$).
\begin{proposition}[\textbf{Lower Bound}] \label{thm:linear}
        For all $\gamma \in [0,1)$, there exists $\inst$ satisfying \Cref{ass:inclusion} and $c \leq n_{min}(\inst) - k$ such that
        \begin{center}
            \upshape{\textsf{fairness}}$(\inst,\linear_\gamma) = 0$.
        \end{center}
        For all $\gamma \in [1,n/(3k)-1)$, there exists an instance $\inst'$ satisfying \Cref{ass:inclusion} in which
        \begin{center}
            \upshape{\textsf{manip-fairness}}$(\inst',\linear_\gamma,n/6-k) \in O(1/n^2)$.
        \end{center}
        For all $\gamma \in [n/(3k)-1,\infty)$, there exists an instance $\inst'$ satisfying \Cref{ass:inclusion} such that 
        \begin{center}
            \upshape{\textsf{manip}}$(\inst',\linear_\gamma,n/6-k) \in \Omega(1)$.
        \end{center}
\end{proposition}
\begin{proof}
    
\textbf{Truthful instance.} All truthful instances we consider in this proof will have two binary features: $F = (f_1, f_2)$ with $V_{f_1} = V_{f_2} = \{0, 1\}$, so $\mathcal{W}=\{00,11,01,10\}$. Our truthful pool $[n]$ will have the following composition: $n_{00} =n_{11} = n/3$, $n_{01} = n_{10} =n/6$. Our truthful instance will have quotas that depend on the case:
\begin{itemize}
\itemsep0em
    \item $0\leq \gamma < 1$: $\inst$ will have quotas $\ell_{f_1, 0} = u_{f_1, 0} = 2k/3$ and $\ell_{f_2, 0} = u_{f_2, 0} = k/3$. 
    \item $\gamma \geq 1$: $\inst'$ will have quotas $\ell_{f_1, 0} = u_{f_1, 0} = k/2$ and $\ell_{f_2, 0} = u_{f_2, 0} = k/2$. 
\end{itemize}

\textbf{Coalitions.} What coalitions deviate from our truthful instance also depends on the case.
\begin{itemize}
    \item $0\leq \gamma < 1$: There is no coalition; in this case, we directly analyze $\inst$, because we are analyzing the outcome of \textsf{fairness}, which is measured in the absence of any manipulation.
    \item $\gamma \geq 1$: In $\inst'$, we let $C \subseteq [n]$ be of size $n/6 - k$, where for all $i \in C,  w_i = 01$. Let all $n/6-k$ agents $i \in C$ misreport $\tilde{w}_i=10$.
\end{itemize}

\textbf{Manipulated Instances.} The pools resulting from these coalitional manipulations are
     \begin{itemize}
    \item $0\leq \gamma < 1$: Not applicable
    \item $\gamma \geq 1$: In the resulting instance, $\tilde{n}_{00}(\inst'^{\to_C\boldsymbol{\tilde{w}}})=\tilde{n}_{11}(\inst'^{\to_C\boldsymbol{\tilde{w}}})= n/3$, $\tilde{n}_{10}(\inst'^{\to_C\boldsymbol{\tilde{w}}}) = n/3 -k$, and $\tilde{n}_{01}(\inst'^{\to_C\boldsymbol{\tilde{w}}}) = k$. 
\end{itemize}

Now, we handle each case separately. For convenience, we will first analyze the Case 1 instance in the relaxation of our setting studied in \citet{flanigan2024manipulation}, where they study the same panel selection task but permit agents to be \textit{divisible}. We call this setting the \textit{continuous setting}. Formally speaking, in instance $\mathcal{I} = (N,k,\boldsymbol{\ell}, \boldsymbol{u})$ such that $\boldsymbol{\ell}, \boldsymbol{u})$, the set of feasible selection probability assignments over which $\mathcal{E}(\pivec)$ could be optimized was
\[P(\inst) = \left\{\pivec : \pivec \in [0,1]^n \ \land \  \sum_{i \in [n]} \pi_i = k \ \land \ \sum_{i \in [n]: f(i) = v} \pi_i = ku_{f,v} \ \forall f,v\in FV\right\}.\]
We analogously define $P^\mathcal{E}(\inst):=\text{arg}\inf_{\pi \in P(\inst)} \mathcal{E}(\inst)$ as the set of $\mathcal{E}$-optimal selection probability assignments over the feasible space $P(\inst)$. Now we proceed with giving constructions in the continuous setting. At the end, we will prove a general method for translating lower bounds in the continuous setting to our setting.

\subsubsection{\textbf{Case 1:} $0\leq \gamma < 1$.} 

\begin{claim}
    $\inst$ satisfies \Cref{ass:inclusion}.
\end{claim}
\begin{proof}
    Consider all panels containing $k/2$ agents with vector $01$, and $k/6$ agents of each remaining vector. Panels of this composition satisfy the quotas, and all agents can be contained on such a panel.
\end{proof}

\begin{claim}
    For all $\gamma \in [0,1)$, $\min(\pivec^{\linear_\gamma}) = 0.$
\end{claim}
\begin{proof}
    Note that we only need one quota constraint per feature because the constraint for one value implies the constraint for the other. Our constraints are then: \begin{align}
    \vecprobs_{01} \cdot \frac{n}{6} + \vecprobs_{10} \cdot \frac{n}{6} + \vecprobs_{00} \cdot \frac{n}{3} + \vecprobs_{11} \cdot \frac{n}{3} &= k \label{case0:c0}\\
         \vecprobs_{01} \cdot \frac{n}{6} + \vecprobs_{00} \cdot \frac{n}{3} &= \frac{2k}{3} \label{case0:c1}\\
                 \vecprobs_{01} \cdot \frac{n}{6} + \vecprobs_{11} \cdot \frac{n}{3} &= \frac{2k}{3} \label{case0:c2}
    \end{align}
             \textbf{Showing that $\vecprobs_{10} =0$ in the optimizer.} By constraints \ref{case0:c1} and \ref{case0:c2} we see that $\vecprobs_{00} = \vecprobs_{11} = k/n - \vecprobs_{10}/2$. Plugging this back into constraint \ref{case0:c0}, we can solve for $\vecprobs_{01}$ and get that
             \[\vecprobs_{01} = \frac{2k}{n} + \vecprobs_{10} \qquad \text{and} \qquad \vecprobs_{00} = \vecprobs_{11} = \frac{k}{n} - \frac{\vecprobs_{10}}{2}.\]
             \textit{Reducing the box constraints to constraints on $\vecprobs_{10}$}: if $\vecprobs_{10}$ is close enough to 0 (or 0), clearly all probabilities are in $[0,1]$.
             
             Now, observe that $\vecprobs_{01}$ must be larger than both $\vecprobs_{10}$ and $\vecprobs_{00}$, so it is the maximum marginal. Then, our objective is the following:
             \[\max\left\{\frac{2k}{n} + \vecprobs_{10} - \gamma \vecprobs_{10}, \ \frac{2k}{n} + \vecprobs_{10} - \gamma \left(\frac{k}{n} - \frac{\vecprobs_{10}}{2}\right)\right\}.\]
            where the two terms in the maximum account for either $\vecprobs_{10}$ or $\vecprobs_{00} = \vecprobs_{11}$ being the minimum marginal probability. By the fact that $0 \leq \gamma < 1$, both of these terms are increasing in $\vecprobs_{10}$. Thus, this is minimized when $\vecprobs_{10} =0$.\\

            We claim that this instance can be translated to our panel distribution setting. Fix the same $\inst$ and require without loss of generality that that $n$ is divisible by 6 and $k$ is divisible by $3$. We take the same definition of $N, \ell,$ and $u$. First we observe that any $\pi \in \Pi(\inst)$ is also in $P(\inst)$. Intuitively this is because  $P(\inst)$ is defined by a relaxation of the constraints defining $\Pi(\inst)$. This was shown formally in Appendix A.2 of \citet{flanigan2024manipulation}. From above, we know that $\boldsymbol{\pi^*} \in P^{\mathcal{E}}(I)$ is of the following form: $\textsf{p}_{10}(\boldsymbol{\pi^*}) = 0, \textsf{p}_{01}(\boldsymbol{\pi^*}) = \frac{2k}{n}, \textsf{p}_{00}(\boldsymbol{\pi^*}) = \textsf{p}_{11}(\boldsymbol{\pi^*}) = \frac{k}{n}$. We first construct a panel distribution, $\pdist$, to assign the same total probability to each feature vector group as $\boldsymbol{\pi^*}$. Let $K$ be a panel populated with $k/3$ agents of type $01$, $k/3$ agents of type $11$ and, $k/3$ agents of type $00$, and let $d_K = 1$. By \Cref{claim:anon_feas}, we know that we can construct a new \textit{anonymous} panel distribution $\pdist'$ with the same total probability assigned to each feature vector. Let $\boldsymbol{\pi'} = \boldsymbol{\pi}(\pdist')$. Therefore, we have that $\textsf{p}_w(\boldsymbol{\pi'}) = \frac{\sum_{i \colon w_i = w} \pi_i(\pdist)}{N_{w_i}}$ for all $i \in [n]$. Solving this gives us:\begin{align*}
                &\textsf{p}_{10}(\boldsymbol{\pi'}) = \frac{\sum_{i \colon w_i = 10} \pi_i(\pdist)}{n_{w_i}} = 0 \quad &&\textsf{p}_{01}(\boldsymbol{\pi'}) = \frac{\sum_{i \colon w_i = 01} \pi_i(\pdist)}{n_{w_i}} = \frac{k/3}{n/6} = \frac{2k}{n}\\
                &\textsf{p}_{00}(\boldsymbol{\pi'}) = \frac{\sum_{i \colon w_i = 00} \pi_i(\pdist)}{n_{w_i}} = \frac{k/3}{n/3} = \frac{k}{n}
                \quad &&\textsf{p}_{11}(\boldsymbol{\pi'}) = \frac{\sum_{i \colon w_i = 11} \pi_i(\pdist)}{n_{w_i}} = \frac{k/3}{n/3} = \frac{k}{n}\\
            \end{align*}
            Therefore, we can observe that $\textsf{p}(\boldsymbol{\pi'}) = \textsf{p}(\boldsymbol{\pi^*})$. Hence, because both $\boldsymbol{\pi'}$ and $\boldsymbol{\pi^*}$ are anonymous and on the same instance, we get that $\boldsymbol{\pi'} = \boldsymbol{\pi^*}$. As $\Pi(\inst) \subseteq P(\inst)$, and $P^{\mathcal{E}}(\inst) = \{\boldsymbol{\pi^*}\}$, we know that $\Pi^{\mathcal{E}}(\inst) = \{\boldsymbol{\pi^*}\}$ as well --- there cannot be any other optimal marginals, otherwise they would be in $P^{\mathcal{E}}(\inst)$ as well. Therefore, we also know that in the panel setting, optimizing $\linear_{\gamma}$ will set $\textsf{p}_{10} = 0$. 
\end{proof}
\end{proof}

\subsubsection{\textbf{Cases 2 and 3:} $1\leq \gamma < n/3-1$ and $\gamma \geq n/3-1$.} 
\begin{claim}
    $\inst'$ satisfies \Cref{ass:inclusion}.
\end{claim}
\begin{proof}
    Observe that we can place all vectors on a valid panel: consider all panels containing $k/3$ agents with $00$, $k/3$ agents with $11$, $k/6$ agents with $01$, and $k/6$ agents with $10$. Panels of this composition satisfy the quotas, and all agents can be contained on such a panel. 
\end{proof}

\begin{claim}
    For all $\gamma \in [1,n/3-1)$, 
    \[\textsf{p}_{10}\left(\pivec^{\linear_{\gamma}}(\inst'^{\to_C\boldsymbol{\tilde{w}}})\right ) = \frac{9k}{2(n^2-9)}\]
    and for all $\gamma \geq n/3-1$,
    \[\textsf{p}_{01}\left(\pivec^{\linear_{\gamma}}(\inst'^{\to_C\boldsymbol{\tilde{w}}})\right ) = 1.\]
\end{claim}
\begin{proof}
    Take the instance $\inst'^{\to_C\boldsymbol{\tilde{w}}}$, and let $\tilde{\mathcal{K}}$ be its associated set of valid panels. Observe that in $\tilde{\mathcal{K}}$, there are two valid panel types:
    \begin{itemize}
        \item \textit{Type 1:} Contains $k/2$ agents with vector 00 and $k/2$ agents with vector 11.
        \item \textit{Type 2:} Contains $k/2-1$ agents with vector 00, $k/2-1$ agents with vector 11, 1 agent with vector 01, and 1 agent with vector 10. 
    \end{itemize}
    
    Fix any $\pdist \in \Delta(\tilde{\mathcal{K}})$, and let $d_1,d_2$ represent the total probability $\pdist$ places on panels of Types 1 and 2, respectively. Then, by simply dividing the expected panel seats given to agents with each vector $w$ divided by the total number of pool members with vector $w$, the resulting selection probabilities (assumed to be anonymous) are:
\begin{align} 
    \vecprobs_{00} = \vecprobs_{11} = d_1 \frac{k/2}{n/3} + d_2 \frac{k/2-1}{n/3}, \qquad \vecprobs_{10} = d_2 \frac{1}{n/3-k}, \qquad \vecprobs_{01} = d_2/k. 
\end{align}
Using that $d_1 + d_2 = 1$ and simplifying, we get that
    \begin{align} 
    \vecprobs_{00} = \vecprobs_{11} =   \frac{k/2-d_2}{n/3}, \qquad \vecprobs_{10} = d_2 \frac{1}{n/3-k}, \qquad \vecprobs_{01} = d_2/k.
\end{align}
Now, we make some observations: 
\begin{itemize}
    \item $\textsf{p}_{01} \geq \textsf{p}_{10}$ for all $d_2 \in [0,1]$, so long as $k \leq n/6$ (which we now set it to be)
    \item $\textsf{p}_{01} \geq \textsf{p}_{00} \iff d_2 \geq \frac{3 k^2}{6 + 2 n}$
    \item $\textsf{p}_{00} \geq \textsf{p}_{10} \iff d_2 \leq \frac{k(n-3k)}{4 n-6k}$
\end{itemize} 
$\linear_\gamma$ is in terms of the maximum and minimum. This gives us three cases for the values of the minimum and maximum probability:

\textbf{Case 1:} $\textsf{p}_{00} \geq \textsf{p}_{01} \geq \textsf{p}_{10}$, which occurs when $d_2 \leq \frac{3 k^2}{6 + 2 n}$. Here,
\begin{align*}
    \linear_\gamma = \textsf{p}_{00} - \textsf{p}_{10} = \frac{k/2-d_2}{n/3} - \gamma \frac{d_2}{n/3-k}.
\end{align*}
By the derivative, observation tells us that for all $\gamma \geq 1$, this objective is decreasing in $d_2$, meaning it is optimized when $d_2$ is maximized over the relevant domain. Then, the optimal solution over this domain is $d_2 = \frac{3 k^2}{6 + 2 n}$, at which point $\textsf{p}_{00} = \textsf{p}_{01}$, which means that this case at the optimizer over this domain is interchangeable with Case 3.

\textbf{Case 2:} $\textsf{p}_{01} \geq \textsf{p}_{10} \geq \textsf{p}_{00}$, which occurs when $d_2 \geq \frac{k(n-3k)}{4 n-6k}$. Here,
\begin{align*}
    \linear_\gamma = \textsf{p}_{01} - \textsf{p}_{00} = d_2/k - \gamma \frac{k/2-d_2}{n/3}
\end{align*}
By the derivative, observation tells us that for all $\gamma$, this objective is increasing in $d_2$, meaning it is optimized when $d_2$ is minimized over the relevant domain. Then, $d_2 = \frac{k(n-3k)}{4 n-6k}$. Again, at at this point $\textsf{p}_{00} = \textsf{p}_{10}$, which means that this case at the optimizer over this domain is interchangeable with Case 3.

\textbf{Case 3:} $\textsf{p}_{01} \geq \textsf{p}_{00} \geq \textsf{p}_{10}$, which occurs when $\frac{3 k}{6 + 2 n} \leq d_2 \leq \frac{k(n-3)}{4 n-6)}.$ Here,
\begin{align*}
    \linear_\gamma = \textsf{p}_{01} - \textsf{p}_{10} = d_2/k - \gamma \frac{d_2}{n/3-k}.
\end{align*}
Examining the derivative, when $\gamma < n/(3k)-1$, this objective function is increasing in $d_2$, meaning that it is minimized by minimizing $d_2$ over the relevant domain; therefore, at the $\linear_\gamma$ optimizer, $d_2 = \frac{3 k^2}{6 + 2 n}$. It follows that 
\[\textsf{p}_{10} = \frac{3 k^2}{(6 + 2 n)(n/3-k)} = \in O(1/n^2).\]

When $\gamma \geq n/(3k)-1$, this objective function is (weakly) decreasing in $d_2$, meaning that it is minimized by maximizing $d_2$ over the relevant domain; therefore, $d_2 = \min\left\{\frac{k(n-3k)}{4 n-6k},1\right\}.$ It follows that
\[\textsf{p}_{01} = \min\left\{\frac{k(n-3k)}{2k(2n-3k)},1\right\} \in \Omega(1). \qedhere\]
\end{proof}

\newpage
\section{Supplemental Materials for Section \ref{sec:transparency}}

\subsection{Proof of \Cref{prop:lbrounding}} \label{app:lbrounding}
 \begin{proposition} \label{prop:lbrounding}
        There exists an instance $\mathcal{I}$ such that for all conditionally equitable algorithms \textsc{E} and for all $\bar{\pivec} \in \overline{\Pi}_m(\inst)$, 
        $\min\left(\pivec^{\textsc{E}}(\inst)\right) - \min\left(\overline{\pivec}\right) \geq \sqrt{k}/m.$
    \end{proposition}
    \begin{proof} We defer the construction of the instance to \citet{flanigan2021transparent}, about which they show that for all $\bar{\pivec} \in \overline{\Pi}_m(\inst)$, $\textit{Maximin}(\pivec) - \textit{Maximin}(\overline{\pivec}) \geq \sqrt{k}/m$. We simply generalize their result to all conditionally equitable objectives with the following simple observation. In their construction, the original instance $\inst$ is such that $k/n\mathbf{1}^n \in \Pi(\inst)$, and in the original panel distribution they consider in fact implies $\pivec = k/n\mathbf{1}^n$. By the definition of conditional equitability, $\pivec$ must be  maximally equal with respect to any conditionally equitable objective $\mathcal{E}$. 
    \end{proof}

\newpage
\section{Supplemental Materials for Section \ref{sec:empirics}}\label{app:empirics}

\subsection{Investigation of alternative $\gamma$ values} \label{sec:alternate-gammas}
In some instances, Problem 1 may necessitate that some agents receive very high probability, regardless of the probability received by other agents. This can make the $\max(\pivec)$ term of \textsc{Goldilocks} unduly large, meaning we may want to increase $\gamma$ to compensate. The easiest way to understand this is in the extreme case: if the quotas require someone to receive probability 1, we are better off setting $\gamma$ to be extremely large and prioritizing only low probabilities, since we cannot gain anything on the high end.

We thus investigate two instance-wise $\gamma$ values, both which use information that would be available in practice to approximately respond to how quotas and self-selection bias in a given instance necessitate practically significant probability gaps. We can see this as trying to approximate $\gamma^*$ in \Cref{prop:general-opt-gold}.

    
\noindent \textbf{$\gamma_1$: minimax/maximin-balanced.} While we \textit{a priori} don't know the best solution in any given instance, we can try to approximately balance the terms using our knowledge of $\min(\pivec^{\maximin}(\inst))$, the maximal minimum probability, and $\max(\pivec^{\minimax}(\inst))$, the minimal maximum probability. The bounds given by our algorithm depend on $\max\{\gamma d, d'\}$, where $k/(dn)$ is the minimum probability and $d'k/n$ is the maximum probability in the feasible instance. To roughly balance these terms relative to one another, we can set
    \[k/(dn) = \min(\pivec^{\maximin}(\inst)) \iff d = \frac{k}{n} \cdot \frac{1}{\min(\pivec^{\maximin}(\inst))}\]
    and \[d'k/n = \max(\pivec^{\minimax}(\inst)) \iff d' = \frac{n}{k} \max(\pivec^{\minimax}(\inst)),\] 
    thereby optimistically proceeding as though there exists an instance where we can achieve the maximal minimum probability and the minimal maximum probability simultaneously. Given that our bounds depend on $\max\{\gamma d, d'\}$, we set $\gamma$ so that $\gamma d$ and $d'$ are balanced:
    \begin{align*}
        \gamma d = d' \iff \gamma = \frac{k}{n} \cdot \frac{1}{ \min(\pivec^{\maximin}(\inst))} = \frac{n}{k} \max(\pivec^{\minimax}(\inst))\\
        \implies \gamma = \frac{n^2}{k^2} \cdot \max(\pivec^{\minimax}(\inst))\cdot \min(\pivec^{\maximin}(\inst)).
    \end{align*}
    Some observations about this method: As we approach the ability to perfectly equalize, $\gamma \to 1$. As $\max(\pivec^{\minimax}(\inst)) \to 1$ but $\min(\pivec^{\maximin}(\inst))$ is around $k/n$, this gets large, approaching order $n/k$ and prompting us to prioritize low probabilities, as desired. Likewise, if $\max(\pivec^{\minimax}(\inst))$ is around $k/n$ but $\min(\pivec^{\maximin}(\inst)) \to 0$, this approaches 0, prompting us to prioritize only the higher probabilities, as desired.\\[-0.5em] 
    
\noindent \textbf{$\gamma_2$  selection bias-balanced.} The weakness of method 2 is that we have to optimize minimax and maximin before we can optimize \textsc{Goldilocks}. We can maybe get around this by getting a coarse-grained approximation to the above approach, which estimates how much gap must exist in the selection probabilities to
    satisfy individual constraints. Building on \citet{flanigan2024manipulation}'s measure $\Delta_{p,k,N}$, we set $\phi_{f,v}([n] ) := |\{i|f (i) = v\}|/n$ and let \[k/(nd) = \min_{(f,v) \in FV} \frac{(\ell_{f,v}+u_{f,v})/2}{\phi_{f,v}([n])} \cdot k/n \qquad \text{and} \qquad d' k/n=\max_{(f,v) \in FV} \frac{(\ell_{f,v}+u_{f,v})/2}{\phi_{f,v}([n])} \cdot k/n.\]
    Computing $\gamma$ to balance terms as we did in Method 2, we get that
    \[\gamma d = d' \iff \gamma =  \min_{f,v \in FV} \frac{(\ell_{f,v}+u_{f,v})/2}{\phi_{f,v}([n])} \cdot \max_{(f,v) \in FV} \frac{(\ell_{f,v}+u_{f,v})/2}{\phi_{f,v}([n])}.\]
    We now show that this has the same desirable behavior as Method 2: first, notice as the self-selection bias goes away, both these terms approach 1 and we get $\gamma = 1$. If the self-selection bias requires very high probabilities for some feature-vector, making the $\max$ term very large, this will make $\gamma$ larger, prompting us to prioritize low probabilities. If the self-selection bias requires very low probabilities for some feature-vector, this will make $\gamma$ term smaller, prompting us to prioritize high probabilities. If they depart equally from 1 (multiplicatively), then the terms will cancel and $\gamma = 1$.

\begin{table}[h!] 
        \centering
        \begin{tabular}{c|ccccccc}
         $\inst$ & \textsc{minimax} & \textsc{minimax-TB} & \textsc{leximin} & \textsc{maximin-TB} & \textsc{GL($1$)} & \textsc{GL($\gamma_1$)} & \textsc{GL($\gamma_2$)} \\
         \hline
  1 &  (0.0, 1.0) &  (0.0, 1.0) & (1.0, 2.0) & (1.0, 1.98) &     (0.72, 1.1) &           (0.76, 1.16) &            (0.78, 1.2) \\
        2 &  (0.0, 1.0) & (0.0, 1.0) & (1.0, 1.33) & (1.0, 1.21) &      (0.9, 1.0) &           (0.98, 1.07) &           (0.98, 1.07) \\
        3 &  (0.0, 1.0) &  (1.0, 1.0) & (1.0, 1.0) &  (1.0, 1.0) &    (1.0, 1.0) &             (1.0, 1.0) &             (1.0, 1.0) \\
        4 &  (0.0, 1.0) &  (0.97, 1.0) & (1.0, 1.0) & (1.0, 1.0) &      (1.0, 1.0) &             (1.0, 1.0) &             (1.0, 1.0) \\
        5 &  (0.0, 1.0) & (0.5, 1.0) & (1.0, 1.17) & (1.0, 1.16) &     (0.95, 1.0) &            (0.95, 1.0) &            (0.95, 1.0) \\
        6 & (0.25, 1.0) & (0.31, 1.0) & (1.0, 1.11) & (1.0, 1.09) &     (1.0, 1.09) &            (1.0, 1.09) &            (1.0, 1.09) \\
        7 &  (0.0, 1.0) & (0.01, 1.0) &  (1.0, 3.5) & (1.0, 2.48) &     (0.7, 1.38) &           (0.78, 1.53) &           (0.83, 1.64) \\
        8 &  (0.0, 1.0) & (1.0, 1.0) & (0.98, 1.0) & (1.0, 1.0)    & (1.0, 1.0) &             (1.0, 1.0) &             (1.0, 1.0) \\
        9 &  (0.0, 1.0) & (1.0, 1.0) &  (1.0, 1.0) &  (1.0, 1.0) &    (1.0, 1.0) &             (1.0, 1.0) &             (1.0, 1.0) \\
        \end{tabular}
   \caption{We compare the performance of the two instance-specific gamma values described above against \textsc{minimax}, \textsc{leximin}, and \textsc{goldilocks} with a gamma value of 1 (\textsc{Goldilocks} is abbreviated as \textsc{GL}). Additionally we include \textsc{Minimax-TB} and \textsc{Maximin-TB}, variants of minimax (and maximin) that tie-break solutions towards maximizing the minimum (minimizing the maximum).}
   \label{tab:alternate-gammas}
    \end{table}

\subsection{Instances} \label{app:instances}
\Cref{tab:instances} gives the values of $n$, $k$, and $|\mathcal{W}_N|$ associated with our 9 instances.
\begin{table}[h!] 
        \centering
        \begin{tabular}{cccccccccc}
        & \multicolumn{9}{c}{\textbf{Instances}}\\
        & 1 & 2 & 3 & 4 & 5 & 6 & 7 &8 & 9 \\
     $n$ & 239 & 312 & 161 & 250 & 404 & 70 & 321 & 1727 & 825\\
     $k$ & 30 & 35 & 44 & 20 & 40 & 24 & 30 & 110 & 75 \\
     $|\mathcal{W}_N|$ & 202 & 182 & 92 & 92 & 108 & 25 & 294 & 762 & 554
        \end{tabular}
   \caption{$k$, $n$, and $|\mathcal{W}_N|$ values across all 9 instances we analyze.} 
        \label{tab:instances}
    \end{table}

\subsection{Description of \textsc{Legacy}} \label{app:legacy}
The \textit{Legacy} algorithm is a greedy heuristic that populates the panel person by person, in each of its $k$ steps uniformly randomizing over all remaining pool members (not yet placed on the panel) who have value $v'$ for feature $f'$, where this feature-value is defined by the following ratio:

\[f',v':= \text{arg}\max_{f,v\in FV}\frac{\ell_{f,v} - \text{\# people already selected for the panel with $f,v$}}{\# \text{ people left in the pool with }f,v}.\]
Intuitively, this is computing how desperate we are for quota $f,v$: the top is how many more people we need to fill the quota, and the bottom is how many we have left. If this is large, then the quota is more desperate. The algorithm proceeds this way until either a valid panel is created, or it is impossible to satisfy the quotas with the remaining pool members, at which case it starts over. A more detailed description of how this algorithm handles corner cases can be found in Appendix 11 of \cite{flanigan2021fair}; these details are not pertinent to our results.

\subsection{Implementation of algorithmic framework} \label{app:algo-implementation}

We provide our code for implementing the framework for all $\mathcal{E}$ we optimize at \url{https://github.com/Cbaharav/fair_manipulation-robust_transparent_sortition}. We give approximate runtimes for optimizing the objectives we study below. These runtimes were obtained on a 13-inch MacBook Pro (2020) with an Apple M1 chip. For clarity, we select a representative run of a smaller instance (Instance 1) as well as a representative run of a large instance (Instance 8) with all of our equality objectives.
\begin{table}[h!] 
        \centering
        \begin{tabular}{c| c c c c c }
       Instance & \textsc{Minimax} & \textsc{Maximin} & \textsc{Leximin} & \textsc{Nash} & \textsc{Goldilocks} \\
       \hline
     1 & 10.05 & 10.68 & 32.35 & 28.59 & 122.94\\
     8 & 35.67 & 33.16 & 358.66 & 857.19 & 264.57\\
     \end{tabular}
   \caption{Times (seconds) of a representative run of all of the various objective-optimizing algorithms on Instances 1 and 8.} 
        \label{tab:instances}
    \end{table}
To calculate optimal distributions under various equality objectives, we used pre-existing implementations of $\maximin, \leximin, \nash$, and $\textsc{Legacy}$ from publicly available code \cite{flanigan2021fair, flanigan2021transparent}. We implemented $\minimax$ and $\gold$ using the algorithmic framework provided by \cite{flanigan2021fair}. Implementing \textsc{Minimax} is straightforward, as its implementation is almost exactly the same as \textsc{Maximin}. To implement \textsc{Goldilocks}, we formulate it as an SOCP, and use an auxiliary constraint (as permitted within the framework) to enforce that selection probabilities are anonymous within a tolerance of 0.01.\\ 

$\gold$ Primal program:
\begin{align*}
    -\max \quad & -(n/k) \cdot x - y/(n/k)\\
    \text{s.t.}\quad & \pi_i = \sum_{P \in \mathcal{K}: i \in P} q_P &&\text{for all }i \in [n]\\
    & \left\|\left(\begin{matrix}
        2\sqrt{\gamma}\\
        \pi_i - y
    \end{matrix}\right)\right\|_2\leq \pi_i + y  &&\text{for all }i \in [n]\\
    & \pi_i \leq x &&\text{for all }i \in [n]\\
    &|\pi_i - \pi_j| \leq 0.01 &&\text{for all } i,j \in [n]\colon w_i=w_j\\
    &\sum_{P \in \mathcal{K}} q_P = 1\\
    &q_P \geq 0 &&\text{for all }P \in \mathcal{K}
\end{align*}
At any optimal solution, we clearly have $x = \max(\mathbf{\pivec})$. We claim that the constraints on $y$ ensure that at any optimal solution, $y = \gamma/\min(\mathbf{\pivec})$. First note that if we instead wrote $y \geq \gamma/\pi_i$ for all $i \in [n]$ this would clearly follow, as the objective is minimizing $y$. Now as we can see below, this constraint is equivalent to our SOCP constraint. 
\begin{align*}
    y \geq \gamma / \pi_i &\iff 4\pi_i y \geq 4\gamma\\
    &\iff \pi_i^2 + 2\pi_i y + y^2 = (\pi_i + y)^2 \geq 4\gamma + \pi_i^2 - 2\pi_i y + y^2\\
    &\iff \pi_i + y \geq \sqrt{4\gamma + \pi_i^2 - 2\pi_i y + y^2}\\
    &\iff \pi_i + y \geq \left\|\left(\begin{matrix}
        2\sqrt{\gamma}\\
        \pi_i - y
    \end{matrix}\right)\right\|_2
\end{align*}

In order to discuss these constraints more succinctly, we rewrite our auxiliary constraints as:\begin{align*}
    A_j \coloneqq  \sqrt{4\gamma + (\pi_j - y)^2} -y - \pi_j &\leq 0 &&\forall j \in [n]\\
    B_j \coloneqq \pi_j - x &\leq 0 &&\forall j \in [n]
\end{align*}
and we denote their dual variables as $\mu_{A_j}$ and $\mu_{B_j}$ respectively. As in the framework, our stopping condition is defined as \begin{align*}
    \sum_{i \in P'} \eta_i &\leq \sum_{i \in P^*} \eta_i + \varepsilon_{GL}\\
    \intertext{ where }
    \eta_i &\gets \frac{\partial}{\partial \pi_i } (-(n/k) \cdot x - y/(n/k)) - \sum_{j=1}^n \mu_{A_j} \frac{\partial A_j}{\partial \pi_i} - \sum_{j=1}^n \mu_{B_j} \frac{\partial B_j}{\partial \pi_i}
\end{align*} where $P'$ is the maximizing panel not currently in the support of the panel distribution for the sum of $\eta_i \colon i \in P'$, and $P^*$ is the maximizing panel currently in the support of the panel distribution for the corresponding sum. The stopping condition as defined in the framework has $\varepsilon_{GL} = 0$, but for computational constraints we set $\varepsilon_{GL} = 1$. We experimented with thresholds down to $\varepsilon_{G:} = 0.1$ and found the degradation of the solution with $\varepsilon_{GL} = 1$ to be relatively insignificant. We now derive the explicit form of $\eta_i$.

\noindent The individual partials are:\begin{align*}
    \frac{\partial A_i}{\partial \pi_i} &= \frac{1}{2}(4\gamma + (\pi_i - y)^2)^{-1/2} \cdot 2(\pi_i - y) \cdot \left(1- \frac{\partial y}{\partial \pi_i}\right) - \frac{\partial y}{\partial \pi_i} - 1\\
    &= (4\gamma + (\pi_i - y)^2)^{-1/2} \cdot (\pi_i - y) \cdot \left(1 + \frac{\gamma}{\pi_i^2} \cdot \mathbb{I}(\pi_i = \gamma/y)\right) + \left(\frac{\gamma}{\pi_i^2} \cdot \mathbb{I}(\pi_i = \gamma/y)\right) - 1\\
    &= \left(\frac{\gamma}{\pi_i^2} \cdot \mathbb{I}(\pi_i = \gamma/y)\right) \cdot \left((4\gamma + (\pi_i - y)^2)^{-1/2} \cdot (\pi_i - y) + 1\right) + (4\gamma + (\pi_i - y)^2)^{-1/2} - 1\\
    \frac{\partial B_i}{\partial \pi_i} &= 1 - \frac{\partial x}{\partial \pi_i}\\
    &= 1 - \mathbb{I}(\pi_i = x)
\end{align*}

\noindent For $i \neq j$:\begin{align*}
    \frac{\partial A_j}{\partial \pi_i} &= \frac{1}{2}(4\gamma + (\pi_j - y)^2)^{-1/2} \cdot 2(\pi_j - y) \cdot \left(- \frac{\partial y}{\partial \pi_i}\right) - \frac{\partial y}{\partial \pi_i}\\
    &= (4\gamma + (\pi_j - y)^2)^{-1/2} \cdot (\pi_j - y) \cdot \left(\frac{\gamma}{\pi_i^2} \cdot \mathbb{I}(\pi_i = \gamma/y)\right) + \left(\frac{\gamma}{\pi_i^2} \cdot \mathbb{I}(\pi_i = \gamma/y)\right) \\
    &= \left(\frac{\gamma}{\pi_i^2} \cdot \mathbb{I}(\pi_i = \gamma/y)\right) \left((4\gamma + (\pi_j - y)^2)^{-1/2} \cdot (\pi_j - y) + 1\right)\\
    \frac{\partial B_j}{\pi_i} &= -\mathbb{I}(\pi_i = x)
\end{align*}

\noindent So altogether we have that: \begin{align*}
    \sum_{j=1}^n \mu_{A_j} \frac{\partial A_j}{\partial \pi_i} &= \left(\frac{\gamma}{\pi_i^2} \cdot \mathbb{I}(\pi_i = \gamma/y) \sum_{j\in[n], j \neq i} \mu_{A_j} \left((4\gamma + (\pi_j - y)^2)^{-1/2} \cdot (\pi_j - y) + 1\right)\right) \\
    &\quad + \mu_{A_i} \left((4\gamma + (\pi_i - y)^2)^{-1/2} - 1\right)\\
    &= \frac{-\gamma}{\pi_i^2} \cdot \mathbb{I}(\pi_i = \gamma/y) \cdot \left(w + \mu_{A_i} \left((4\gamma + (\pi_i - y)^2)^{-1/2} \cdot (\pi_i - y) + 1\right)\right)\\
    &\quad + \mu_{A_i} \left((4\gamma + (\pi_i - y)^2)^{-1/2} - 1\right)
\end{align*}
where $w := -\sum_{j=1}^n \mu_{A_j} \left((4\gamma + (\pi_j - y)^2)^{-1/2} \cdot (\pi_j - y) + 1\right)$. For the partials for the $B$ constraints:\begin{align*}
    \sum_{j=1}^n \mu_{B_j} \frac{\partial B_j}{\partial \pi_i} &= \mu_{B_i}(1 - \mathbb{I}(\pi_i = x))-\mathbb{I}(\pi_i = x)\sum_{j \in [n], j \neq i} \mu_{B_j}\\
    &= \mu_{B_i}(1 - \mathbb{I}(\pi_i = x))-\mathbb{I}(\pi_i = x) \cdot (q - \mu_{B_i})\\
    &= \mu_{B_i} - \mathbb{I}(\pi_i = x)q
\end{align*}
where $q := \sum_{j=1}^n \mu_{B_j}$.\\

\noindent For precision issues, we replace $\mathbb{I}(\pi_i = \gamma/y) \mapsto \pi_i = \min(\pivec)$ and $\mathbb{I}(\pi_i = x) \mapsto \pi_i = \max(\pivec)$ in the code implementation. Then putting it altogether we have: \begin{align*}
    \eta_i^* &\gets \frac{\partial}{\partial \pi_i } (-(n/k) \cdot x - y/(n/k)) - \sum_{j=1}^n \mu_{A_j} \frac{\partial A_j}{\partial \pi_i} - \sum_{j=1}^n \mu_{B_j} \frac{\partial B_j}{\partial \pi_i}\\
    &= -(n/k) \cdot \mathbb{I}(\pi_i = \max(\pivec)) + (k/n) \cdot (\gamma/\pi_i^2) \cdot \mathbb{I}(\pi_i = \min(\pivec)) \\
    &- \left(\sum_{j=1}^n \mu_{A_j} \frac{\partial A_j}{\partial \pi_i} + \sum_{j=1}^n \mu_{B_j} \frac{\partial B_j}{\partial \pi_i}\right)
\end{align*}

\subsection{Feature Dropping Methods \& Results for Additional Instances} \label{app:feature-drop}

In $\inst$, we define the selection bias of feature $f$ exactly as in \cite{flanigan2024manipulation}:
\[\Delta^f_{N,k,\boldsymbol{\ell},\boldsymbol{u}}:= \max_{v \in V_f} \frac{(\ell_{f,v}+u_{f,v})/2}{\eta_{f,v}(N)}-\min_{v \in V_f} \frac{(\ell_{f,v}+u_{f,v})/2}{\eta_{f,v}(N)}\]
where $\eta_{f,v}(N)$ represents the fraction of people in the pool $N$ with value $v$ for feature $f$.

Then, we order the features in decreasing order of $\Delta^f_{N,k,\boldsymbol{\ell},\boldsymbol{u}}$ as follows
\[\Delta^{f_1}_{N,k,\boldsymbol{\ell},\boldsymbol{u}} \geq \Delta^{f_2}_{N,k,\boldsymbol{\ell},\boldsymbol{u}} \geq \dots \geq \Delta^{f_{|F|}}_{N,k,\boldsymbol{\ell},\boldsymbol{u}}\] 
And in \Cref{fig:feature-drop}, we drop features $f_1$ (1 feature dropped), then $f_1$ and $f_2$ (2 features dropped), then $f_1,f_2,f_3$ (3 features dropped), and so forth. Dropping a feature, formally speaking, means that we are dropping their associated quota constraints; so after we have dropped $y$ features, we are imposing quotas 
\[\boldsymbol{\ell}':=(\ell_{f,v} |  v \in V_f, f \in F \setminus \{f_1,\dots,f_y\}), \quad \text{and} \quad \boldsymbol{u}':=(u_{f,v} |  v \in V_f, f \in F \setminus \{f_1,\dots,f_y\}).\]

\noindent \textbf{Results for additional instances.} In \Cref{fig:feature-drop-extras}, we provide the analog to \Cref{fig:feature-drop} for the remaining 6 instances omitted from the body.
\begin{figure}[h!]
    \centering
\includegraphics[width=\textwidth]{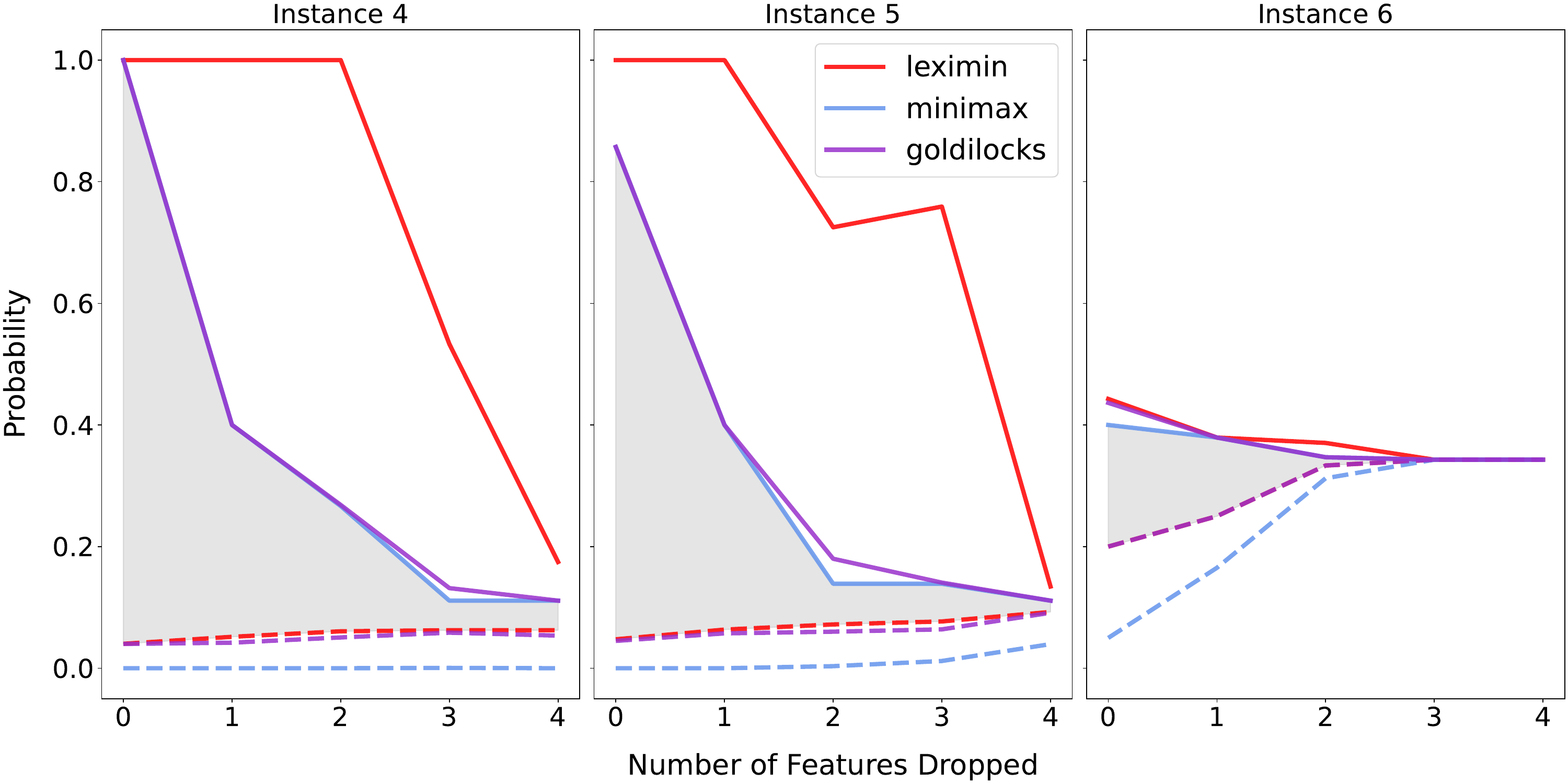}\\
\includegraphics[width=\textwidth]{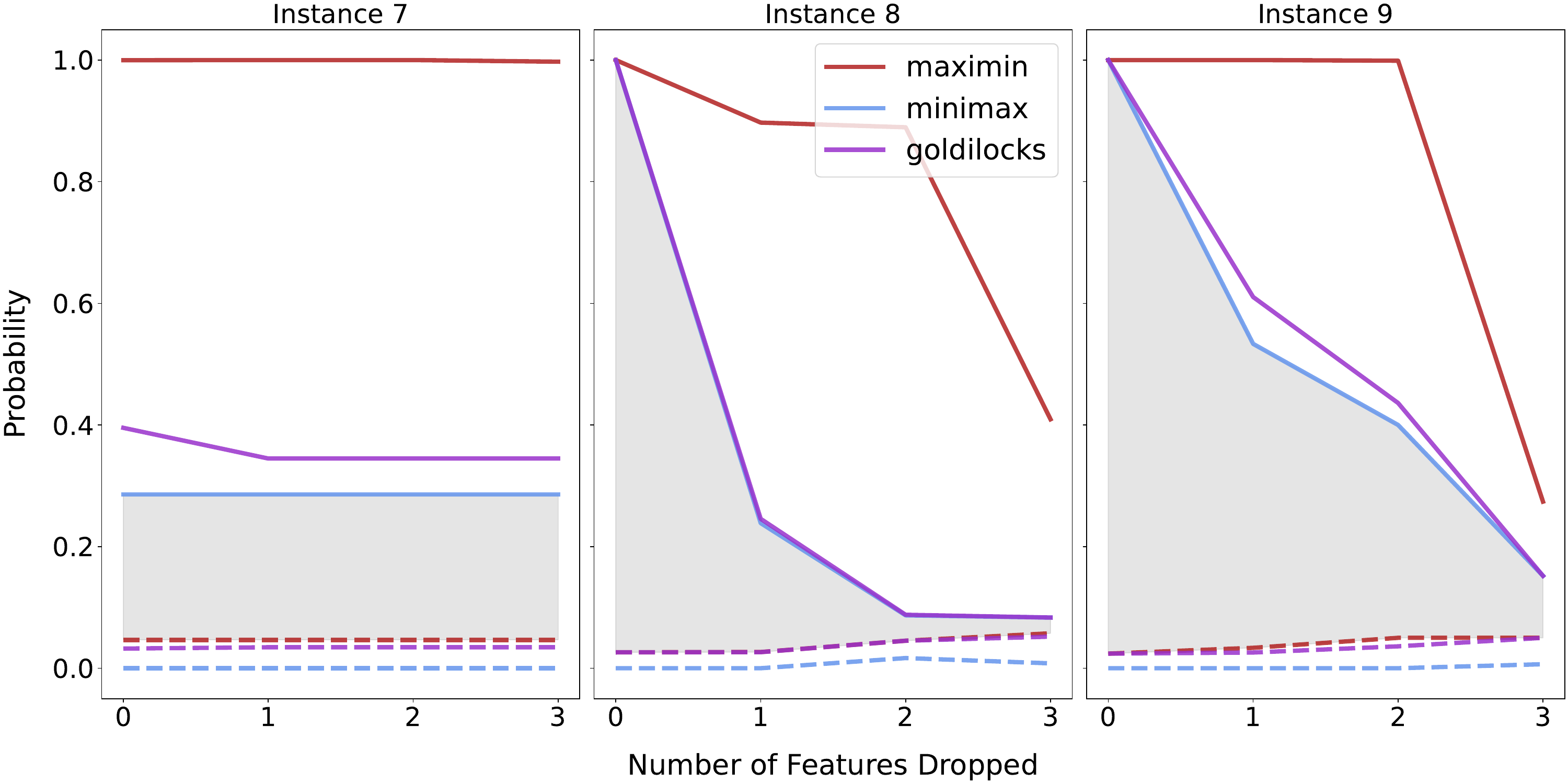}
    \caption{In instances 7-9, we use $\maximin$ instead of $\leximin$ to indicate the optimal minimum marginal probability because of computational costs due to the size of these instances. We additionally drop only 3 features instead of 4 because instance 9 only has 4 features.}
    \label{fig:feature-drop-extras}
\end{figure}

\subsection{Manipulation Robustness Experimental Methods} \label{app:manip-robust}

In our manipulability experiments, we used the high level structure implemented in \cite{flanigan2024manipulation}, but modified it to be in the \textit{panel} distribution setting as opposed to the continuous setting. We now formally describe several aspects of our experimental design.\\[-0.5em]

\noindent \textbf{Simulating the growth of the pool via \textit{pool copies}.} On the horizontal axis of our plots, we vary the number of \textit{pool copies}. In an instance $\inst = ([n], \boldsymbol{w}, k,\boldsymbol{\ell},\boldsymbol{u})$, 1 pool copy means the pool is $[n]$; 2 pool copies means that the pool is $[2n]$ and the vector of feature-vectors is duplicated (that is, we duplicate each agent in the original pool once, and leave all else about the instance the same).\\[-0.5em]

\noindent \textbf{The \textit{Most Underrepresented (MU)} strategy}
Fix an $\inst = ([n],\boldsymbol{w}, k,\boldsymbol{\ell},\boldsymbol{u})$. For every feature $f$, let the most underrepresented value be $v^*_f$, defined as
\[v^*_f:=\text{arg}\max_{v \in V_f} \frac{(\ell_{f,v}+u_{f,v})/2}{\phi_{f,v}([n])},\]
with $\phi_{f,v}([n])$ defined the same way as above. Then, when an agent $i \in [n]$ employs the \textit{MU} strategy, they misreport the vector
\[w^{MU}:=(v^*_f | f \in F).\]

\noindent \textbf{Computing the worst-case \textit{MU} manipulator.}
Fix an $\inst = ([n],\boldsymbol{w}, k,\boldsymbol{\ell},\boldsymbol{u})$ and a maximally equal algorithm \textsc{E}. Define $\boldsymbol{\tilde{w}}$ to be identical to $\boldsymbol{w}$ except that $\tilde{w}_i = w^{MU}$. Let $\tilde{\inst}_i = ([n], \boldsymbol{\tilde{w}},k,\boldsymbol{\ell},\boldsymbol{u})$ be the instance in which $i$ has employed the \textit{MU} manipulation strategy, and all other agents are truthful. Then, we run the following algorithm (pseudocode here) to compute the most any \textit{MU} manipulator in the instance can gain.
\begin{itemize}
    \item $\text{max-gain} \leftarrow 0$
    \item $\text{compute }\pivec^{\textsc{E}}(\inst)$
    \item \text{for all $i \in N$:}
    \begin{itemize}
        \item compute $\pivec^{\textsc{E}}(\tilde{\inst}_i )$
        \item if $\pi_i^{\textsc{E}}(\tilde{\inst}_i ) - \pi_i^{\textsc{E}}(\inst) > $max-gain, set max-gain to this larger difference.
    \end{itemize}
    \item return max-gain
\end{itemize}

\subsection{Transparency Experimental Methods} \label{app:transparency}

We model our transparency experiments after the experiments done by \citet{flanigan2021transparent}. The two rounding procedures that we utilize in this paper are \textit{ILP} and \textit{Pipage}.\\ 

\noindent\textbf{Theoretical Bounds.} In order to get theoretical upper bounds on the change in any individual's marginal probabilities as a result of rounding, we utilize the results from \citet{flanigan2021transparent}. Theorem 3.2 gives us an upper bound of $b_1 \coloneqq k/m$, while Theorem 3.3 gives a bound of:\begin{align*}
    b_2 \coloneqq \frac{\sqrt{\frac{1}{2}(1+ \frac{\ln 2}{\ln |\mathcal{W}_\inst|})} \cdot \sqrt{|\mathcal{W}_\inst| \ln(|\mathcal{W}_\inst|)} + 1}{m}
\end{align*}
You can find instance-specific values of $n, k$ and $|\mathcal{W}_\inst|$ in \Cref{app:algo-implementation}. For our experiments, we set the number of panels $m$ to $1000$.

Then, for a given instance $\inst$, we derived our theoretical bound on the minimum probability as $\min(\boldsymbol{\pi}^{\gold_1}(\inst)) - \min(b_1, b_2)$ and the theoretical bound on the maximum probability as $\max(\boldsymbol{\pi}^{\gold_1}(\inst)) + \min(b_1, b_2)$.\\




\noindent\textbf{Pipage.} We ran the pipage algorithm implemented by \citet{flanigan2021transparent} for 1000 independent repetitions for each of our instances. We stored the minimum and maximum marginals from each repetitions and computed the average minimum and average maximum marginal. Additionally, we computed the standard deviation of minimum and maximum marginals across these repetitions. We ultimately found that the spread of the data was very low --- standard deviation of minimum and maximum marginals across repetitions did not exceed 0.0015 across all of our instances, and was typically much lower.

\end{document}